\documentclass[11pt]{article}

\usepackage{footmisc}
\usepackage[T1]{fontenc}
\usepackage[utf8]{inputenc}
\usepackage{lmodern}
\usepackage{amssymb,amsfonts}
\usepackage{iftex}  
\ifPDFTeX                  
  \usepackage[activate={true,nocompatibility},final]{microtype}
\else                     
  \usepackage[protrusion=true,final]{microtype}
\fi

\usepackage[margin=0.9in]{geometry}

\usepackage{authblk}                    % \author[1]{...}\affil[1]{...} for multi-author

\usepackage{amsmath}
\usepackage{mathtools} 
\usepackage{bm}
\usepackage{nicefrac}
\allowdisplaybreaks         % allow long aligned displays to break across pages

\usepackage{amsthm}
\usepackage{thm-restate}

\usepackage{graphicx}
\graphicspath{{figures/}{./}}                       % find images in figures/
\usepackage{booktabs}                               % \toprule \midrule \bottomrule
\usepackage{multirow,makecell,array}
\usepackage{longtable}                              % multi-page tables (taxonomy appendix)
\usepackage{pdflscape}                              % landscape pages (taxonomy appendix)
\usepackage[font=small,labelfont=bf]{caption}       % restrained, bold-label captions
\usepackage{subcaption}
\usepackage{algorithm}
\usepackage{algpseudocode}                          % algorithmicx pseudocode

\usepackage{enumitem}
\setlist[itemize]{leftmargin=2.2em,itemsep=2pt,topsep=2pt}
\setlist[enumerate]{leftmargin=2.2em,itemsep=2pt,topsep=2pt}

\usepackage{xcolor}
\definecolor{LinkColor}{rgb}{0.10,0.40,0.75}        % internal references
\definecolor{CiteColor}{rgb}{0.70,0.25,0.20}        % citations
\definecolor{UrlColor} {rgb}{0.20,0.50,0.50}        % external URLs
\definecolor{TodoColor}{rgb}{0.80,0.30,0.10}        % draft notes

\usepackage{tikz}
\usetikzlibrary{positioning,calc,arrows.meta,decorations.pathreplacing,patterns,shapes.geometric}

\usepackage[round,sort&compress]{natbib}

\usepackage{url}
\usepackage{hyperref}
\hypersetup{
  colorlinks=true,
  linkcolor=LinkColor,
  citecolor=CiteColor,
  urlcolor=UrlColor,
  breaklinks=true,
  bookmarksnumbered=true,
}
\usepackage{bookmark}                               % robust PDF bookmarks
\usepackage[capitalise,nameinlink,noabbrev,sort&compress]{cleveref}  % \cref / \Cref

\numberwithin{equation}{section}

\newif\ifdraft \draftfalse
\ifdraft
  \usepackage{lineno}\linenumbers
  \newcommand{\todo}[1]{\textcolor{TodoColor}{\textbf{[TODO:}~#1\textbf{]}}}
\else
  \newcommand{\todo}[1]{}        
\fi

\usepackage[scale=0.94]{sourcesanspro}   % sans face for headings/captions only
\definecolor{floodblue}{HTML}{1F4FA8}
\definecolor{harvestgold}{HTML}{B97A0A}
\definecolor{floodtint}{HTML}{F3F7FC}
\definecolor{goldtint}{HTML}{FCF8EE}
\definecolor{inkgray}{HTML}{4D4D4D}

\usepackage{titlesec}
\titleformat*{\section}{\Large\sffamily\bfseries\color{floodblue}}
\titleformat*{\subsection}{\large\sffamily\bfseries\color{floodblue}}
\titleformat*{\subsubsection}{\normalsize\sffamily\bfseries\color{floodblue}}
\titleformat{\paragraph}[runin]{\normalsize\bfseries}{}{0pt}{}
\titlespacing*{\paragraph}{0pt}{1.4ex plus .4ex}{0.9em}

\usepackage{tocloft}

\usepackage{fancyhdr}
\fancyheadoffset{0pt}
\newcommand{\gradientrule}{\begin{tikzpicture}%
  \shade[left color=floodblue, right color=harvestgold]
    (0,0) rectangle (\textwidth, 0.06em);\end{tikzpicture}}
\usepackage[most]{tcolorbox}
\newcommand{\floodboxset}[1]{\tcolorboxenvironment{#1}{enhanced jigsaw, breakable,
  boxrule=0pt, frame hidden, colback=floodtint, arc=1.2mm,
  borderline west={1.5pt}{0pt}{floodblue!70},
  left=7pt, right=7pt, top=4pt, bottom=4pt,
  before skip=11pt, after skip=11pt}}
\newcommand{\goldboxset}[1]{\tcolorboxenvironment{#1}{enhanced jigsaw, breakable,
  boxrule=0pt, frame hidden, colback=goldtint, arc=1.2mm,
  borderline west={1.5pt}{0pt}{harvestgold!75},
  left=7pt, right=7pt, top=4pt, bottom=4pt,
  before skip=11pt, after skip=11pt}}
\AtBeginDocument{%
  \floodboxset{theorem}\floodboxset{lemma}\floodboxset{proposition}%
  \floodboxset{corollary}\floodboxset{fact}\floodboxset{conjecture}%
  \floodboxset{restatable}
  \goldboxset{definition}\goldboxset{remark}\goldboxset{assumption}\goldboxset{example}%
}
\newtcolorbox{statementbox}{enhanced jigsaw, breakable, boxrule=0pt, frame hidden,
  colback=floodtint, arc=1.2mm, borderline west={1.5pt}{0pt}{floodblue!70},
  left=7pt, right=7pt, top=4pt, bottom=4pt, before skip=11pt, after skip=11pt}

\newtcolorbox{poincarebox}{enhanced jigsaw, boxrule=0pt, frame hidden,
  interior style={left color=floodtint, right color=goldtint}, arc=2mm,
  left=9pt, right=9pt, top=5pt, bottom=5pt}

\newcommand{\N}{\mathbb{N}}

\newcommand{\mc}[1]{\mathcal{#1}} 

\newcommand{\defeq}{\coloneqq}      
\DeclarePairedDelimiter{\abs}{\lvert}{\rvert}

\DeclarePairedDelimiter{\set}{\{}{\}}

\DeclarePairedDelimiterX{\inner}[2]{\langle}{\rangle}{#1,#2}   % \inner{u}{v} = <u,v>

\newcommand{\gnd}{U}   
\newcommand{\Amb}{F}   
\newcommand{\Val}{H}  
\newcommand{\Pairs}{\mc{P}} 
\newcommand{\Hcal}{\mc{H}}  
\newcommand{\Hfam}{\Hcal_{\alpha}} 
\newcommand{\Seen}{S} 
\newcommand{\Out}{O}
\newcommand{\Core}{C}   
\newcommand{\Comp}[1]{H_{#1}}    
\newcommand{\dlow}[2]{\underline{d}(#1,#2)}  
\newcommand{\dup}[2]{\overline{d}(#1,#2)}        
\newcommand{\dens}[2]{d(#1,#2)}                  
\newcommand{\trv}{\mathrm{trv}}                
\newcommand{\cov}{c}

\newcommand{\rest}[2]{#1[1\!\dots\! #2]}      

\newcommand{\fiber}[1]{\Pairs_{#1}}          
\newcommand{\setbuild}[2]{\left\{\,#1 : #2\,\right\}}

\theoremstyle{plain}
\newtheorem{theorem}{Theorem}[section]

\newtheorem{lemma}[theorem]{Lemma}

\newtheorem{fact}[theorem]{Fact}

\theoremstyle{definition}
\newtheorem{definition}[theorem]{Definition}

\theoremstyle{remark}
\newtheorem{remark}[theorem]{Remark}

\crefname{theorem}{Theorem}{Theorems}          \Crefname{theorem}{Theorem}{Theorems}
\crefname{proposition}{Proposition}{Propositions}
\Crefname{proposition}{Proposition}{Propositions}
\crefname{lemma}{Lemma}{Lemmas}                \Crefname{lemma}{Lemma}{Lemmas}
\crefname{corollary}{Corollary}{Corollaries}   \Crefname{corollary}{Corollary}{Corollaries}
\crefname{conjecture}{Conjecture}{Conjectures} \Crefname{conjecture}{Conjecture}{Conjectures}
\crefname{fact}{Fact}{Facts}                   \Crefname{fact}{Fact}{Facts}
\crefname{definition}{Definition}{Definitions} \Crefname{definition}{Definition}{Definitions}
\crefname{assumption}{Assumption}{Assumptions} \Crefname{assumption}{Assumption}{Assumptions}
\crefname{example}{Example}{Examples}          \Crefname{example}{Example}{Examples}
\crefname{problem}{Problem}{Problems}          \Crefname{problem}{Problem}{Problems}
\crefname{remark}{Remark}{Remarks}             \Crefname{remark}{Remark}{Remarks}
\crefname{claim}{Claim}{Claims}                \Crefname{claim}{Claim}{Claims}
\crefname{algorithm}{Algorithm}{Algorithms}    \Crefname{algorithm}{Algorithm}{Algorithms}

\newcommand{\FloodWord}{%
\textcolor{cyan!70!blue}{F}\textcolor{cyan!50!blue}{l}\textcolor{cyan!32!blue}{o}\textcolor{cyan!16!blue}{o}\textcolor{blue!85!black}{d}}
\newcommand{\HarvestWord}{%
\textcolor{yellow!75!orange}{H}\textcolor{yellow!58!orange}{a}\textcolor{yellow!40!orange}{r}\textcolor{orange!95!yellow}{v}\textcolor{orange!88!red!90!black}{e}\textcolor{orange!75!red!85!black}{s}\textcolor{orange!60!red!75!black}{t}}
\hypersetup{pdftitle={Flood and Harvest: The Provable Necessity of Trivia for Generating Valuable Mathematics via the Lens of Language Generation in the Limit}}

\begin{document}

% ---- Title page -----
\thispagestyle{empty}
\vspace*{-0.92in}
\begin{center}
  \makebox[\textwidth]{\includegraphics[width=0.702\paperwidth]{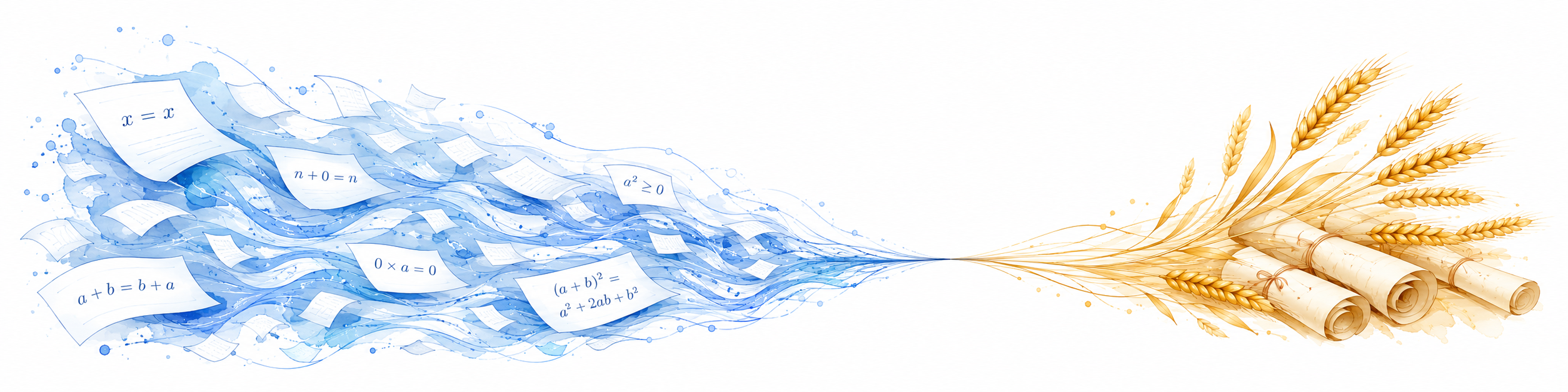}}\\[0.1em]
  {\Huge\sffamily\bfseries \FloodWord{} and \HarvestWord\par}
  \vspace{0.45em}
  {\Large\sffamily\bfseries\color{black!85}
   The Provable Necessity of Trivia for Generating Valuable Mathematics\\[0.2em]
   via the Lens of Language Generation in the Limit\par}
  \vspace{0.45em}
  \gradientrule
  \vspace{0.9em}
  {\normalsize
   \begin{tabular}{@{}l@{\hspace{3em}}l@{\hspace{3em}}l@{}}
   Xiaoyu Li$^{1}$ & Andi Han$^{2}$ & Dai Shi$^{3}$\\[0.35em]
   Zheng Gao$^{1}$ & Jiaojiao Jiang$^{1}$ & Junbin Gao$^{2}$
   \end{tabular}\par}
  \vspace{0.7em}
  {\small
   $^{1}$University of New South Wales \quad
   \texttt{\{xiaoyu.li2,\,zheng.gao1,\,jiaojiao.jiang\}@unsw.edu.au}\\[0.15em]
   $^{2}$University of Sydney \quad
   \texttt{\{andi.han,\,junbin.gao\}@sydney.edu.au}\\[0.15em]
   $^{3}$University of Cambridge \quad
   \texttt{ds2213@cam.ac.uk}\par}
  \vspace{0.45em}
  {\small\itshape\color{inkgray}\today\par}
\end{center}
\vspace{-0.6em}
\begin{abstract}
\small\linespread{1.0}\selectfont
AI systems coupled to proof assistants now generate formal mathematics at scale, and the gap
between what a checker can \emph{verify} and what a mathematician would \emph{value} has
become the binding constraint. We model the generation of valuable mathematics as
\emph{nested} language generation in the limit: a verifiable formal language $\Amb$, accessed
through a membership oracle (the proof checker), contains an unknown valuable language
$\Val\in\Hcal$ revealed only through an adversarial enumeration of a core $\Core\subseteq\Val$
of exact density $\alpha$ (the literature). Every output is \emph{valuable} ($\in\Val$),
\emph{trivial} ($\in\Amb\setminus\Val$), or a \emph{hallucination} ($\notin\Amb$). We settle
four questions. First, the verifier is not taste: the collections admitting generation with
breadth are exactly those of the oracle-free model, characterized fiber-wise by Angluin's
condition. Second, the verifier does buy \emph{sound coverage}, covering all unseen valuable
statements while asserting only valid ones: possible with it, impossible without it; it
relocates unavoidable errors from false to trivial. Third, and centrally, a sharp dichotomy on
the tight family: generators emitting finitely many trivia achieve optimal coverage
$\alpha/2$, while any infinite trivia allowance, even at vanishing rate, jumps the optimum to
$1-\alpha/2$ (both tight, for cores presented as the candidate intersection), and one
generator attains both ends. The transition is in trivia \emph{count}, not rate; the gap
$1-\alpha$ is the unrecorded mass. Fourth, both regimes instantiate in a compression model of
mathematics. A perfect verifier cannot substitute for taste: the unbounded stream of
correct-but-worthless statements is not an engineering accident but a \emph{provable
necessity}, since covering unrecorded valuable mathematics requires an infinite, but
asymptotically negligible, stream of certified trivia. The \emph{flood} of the title is that
stream; the \emph{harvest} is the unrecorded value that only the flood can buy.
\vspace{3mm}
\begin{poincarebox}
\begin{quote}
\itshape
``
What, in fact, is mathematical discovery? It does not consist in making new combinations with mathematical entities that are already known. That can be done by any one, and the combinations that could be so formed would be infinite in number, and the
greater part of them would be absolutely devoid of interest. 
Discovery consists precisely in not constructing useless combinations, but in constructing those that are useful, which are an infinitely small minority. Discovery is discernment, selection.''

\hfill --- Henri Poincaré\,$^\lozenge$

\end{quote}
\end{poincarebox}
\nopagebreak
{\footnotesize\color{inkgray}\hfill $^\lozenge$\,Quoted from \citeauthor{poincare1914science}'s discussion of mathematical discovery in \emph{Science and Method} \citeyearpar{poincare1914science}.\par}
\end{abstract}
\tableofcontents
\newpage
\section{Introduction}
\label{sec:intro}

Machine-generated mathematics has moved from aspiration to practice. Language models coupled
to proof assistants are run as tireless provers, from GPT-f through HyperTree search to
DeepSeek-Prover and Goedel-Prover
\citep{polu2020gptf,lample2022hypertree,deepseek2024prover,lin2025goedel}; olympiad-level
systems now solve hard competition problems
\citep{trinh2024alphageometry,chervonyi2025alphageometry2,hubert2026alphaproof}; AI-assisted
discovery has produced genuine new mathematics
\citep{davies2021advancing,romera2024funsearch,georgiev2025exploration}; large formalization
efforts have put substantial bodies of mathematics into machine-checkable form
\citep{avigad2023formal,aksenov2026compression}; and autonomous systems now generate
conjectures and theory at scale \citep{barkeshli2026structure,tsoukalas2025interestingness}.
The shared substrate is a membership oracle for a formal language, a checker that accepts a
string exactly when it encodes a valid derivation, against which such a system can emit, one
after another, an unbounded stream of statements each of which is, by construction, correct
(see \citet{li2024survey,ju2026ai} for surveys).

Practitioners face a problem the empirical literature documents and laments without
formalizing: the checker certifies \emph{validity}, not \emph{value}. It accepts a derivation;
it does not certify that the theorem is worth proving, a distinction Thurston famously drew
between proof and mathematical understanding \citep{thurston1994proof}. A generator can
therefore flood its output with statements that are formally impeccable and mathematically
worthless, or it can aim instead at the conjectures a mathematician would care about and risk
being wrong. And the flood is alarmingly cheap to manufacture: a short script written for
this paper generated, and machine-verified, the statement
\begin{equation*}
  \frac{\Gamma\!\big(\tfrac12\big)^{2}}{\pi}
  \;+\;
  \big(e^{i\pi} + 2\big)
  \big(\varphi^{2} - \varphi\big)
  \big({-2}\,\zeta(0)\big)
  \;=\;
  \exp\!\Bigg(\sum_{k=1}^{\infty} \frac{(-1)^{k+1}}{k}\Bigg),
\end{equation*}
in which Euler's identity, the golden ratio, the Gamma function, and the Riemann zeta
function all make cameo appearances, and none is needed. The verification is instant; the
mathematical content is $1+1=2$. Nothing in the checker's verdict distinguishes this
confection from a theorem. The gap between what is verifiable and what is valuable is, in
this record, the central bottleneck, and it has no formal account.

In this paper we study this problem theoretically. We ask the question no prior model
addresses. With a perfect verifier in hand, what can a generator provably guarantee about
\emph{value}, and what, if anything, does the verifier buy? We prove exactly what is and is
not achievable. We work in the framework of language generation in the limit
\citep{kleinberg2024language}, extended to a nested pair of languages; the extension is the
source of the new phenomena, so we describe it before stating results.

The title names the two regimes this theory prices. The \emph{flood} is the unbounded stream
of certified trivia that any broad generator must emit; the \emph{harvest} is the unrecorded
valuable mathematics that only such a stream can reach. \citet{tao2025machine} envisions
machine assistance under which ``mathematical explorations become possible at scales that are
not currently feasible'': fields that today ``study one or two equations at a time'' may come
to ``study hundreds of equations at once.'' Our results price that vision. Explored at scale,
mathematics yields its unrecorded value only to generators that flood, and the flood can be
made certified-true and asymptotically negligible in rate, but never finite
(\cref{fig:worlds}).
The boundary this draws for system design is clean: verification decides where exploration
may step, while examples of recorded value decide where it should, and no amount of the
former substitutes for the latter. The remainder of this introduction makes the model
precise, states the four results, and locates them in the literature.

\subsection{New model: nested language generation in the limit}
\label{sec:intro-model}

Fix a countable ground set $\gnd$ of strings with a canonical order. An instance is a pair
$(\Amb,\Val)$ with $\Val\subseteq\Amb\subseteq\gnd$, both infinite: the \emph{formal} language
$\Amb$ is the verifiable ambient (membership in $\Amb$ is decided by an oracle, the proof
checker), and the \emph{valuable} language $\Val$ is the target, which the generator never
sees directly. The adversary draws $(\Amb,\Val)$ from a countable collection, grants the
generator a membership oracle for $\Amb$, and enumerates a core $\Core\subseteq\Val$ of lower
density $\alpha$, the literature, which records only a fraction of what is valuable. Round by
round the generator reads the enumerated prefix $\Seen_t$, queries the oracle, and outputs a
fresh string $a_t\notin\Seen_t$. Each output falls into exactly one of three modes
(\cref{fig:worlds}):
\begin{center}
\begin{tabular}{@{}lll@{}}
\toprule
\textbf{mode} & \textbf{condition} & \textbf{meaning}\\
\midrule
hallucination & $a_t\notin\Amb$ & not even valid (the classical failure)\\
triviality    & $a_t\in\Amb\setminus\Val$ & valid but worthless (the new failure)\\
valuable      & $a_t\in\Val\setminus\Seen_t$ & new and valuable (hence also valid)\\
\bottomrule
\end{tabular}
\end{center}
\begin{figure}[t]
\centering
\resizebox{0.97\textwidth}{!}{% Figure: the nested worlds of machine-generated mathematics
\begin{tikzpicture}[
  >=Stealth, font=\footnotesize,
  dot/.style={circle, fill=gray!60, inner sep=1.1pt},
  star/.style={diamond, fill=olive!80!black, inner sep=1.5pt},
  rec/.style={circle, fill=black, inner sep=1.1pt},
  lab/.style={font=\scriptsize\itshape},
  ex/.style={font=\scriptsize, draw=gray!50, rounded corners=2pt,
             inner sep=2.5pt, fill=white, fill opacity=0.85, text opacity=1},
  extriv/.style={ex, draw=gray!55, text=black!75},
  exval/.style={ex, draw=olive!60!black, text=olive!35!black},
  exrec/.style={ex, draw=black!60},
  exhal/.style={ex, draw=red!50, text=red!60!black},
]
% ================= the worlds =================
\draw[rounded corners=10pt, gray!70]
  (2.6,0.2) rectangle (15.6,9.0);
\node[lab, anchor=north west, gray!80!black] at (2.78,8.92)
  {$U$: all statements};

\draw[blue!55, thick, fill=blue!5]
  (9.7,4.5) ellipse [x radius=5.3, y radius=3.6];
\node[lab, blue!60!black, anchor=north] at (9.7,7.62)
  {$F$: provable, machine-checkable};

\draw[olive!70!black, thick, fill=olive!12]
  (10.9,3.6) ellipse [x radius=3.0, y radius=1.95];
\node[lab, olive!50!black, anchor=south, font=\scriptsize\itshape] at (10.6,4.62)
  {$H$: valuable (unknown)};

\draw[black!70, densely dashed, fill=gray!22]
  (11.9,2.95) ellipse [x radius=1.72, y radius=1.08];
% recorded theorems, named and cited
\node[exrec, font=\tiny] at (11.9,3.62) {FLT \citeyearpar{wiles1995modular}};
\node[exrec, font=\tiny] at (11.9,3.18) {PNT \citeyearpar{hadamard1896distribution}};
\node[exrec, font=\tiny] at (11.9,2.74) {Green--Tao \citeyearpar{green2008primes}};
\node[exrec, font=\tiny] at (11.9,2.30) {Galois theory \citeyearpar{galois1846oeuvres}};
\node[lab, black!75, anchor=west, align=left] at (12.95,0.8)
  {$C$: the literature\\ (density $\alpha$)};
\draw[->, black!60, shorten >=1pt] (13.05,1.15) -- (12.6,2.05);

% ================= the generator =================
\node[draw, rounded corners, align=center, fill=white, inner sep=4pt]
  (gen) at (0.95,4.6) {generator\\ $+$ verifier $\mathbf 1_{F}$};

% ================= inhabitants =================
% --- trivia in F \ H (the flood), with two worked examples ---
\foreach \p in {(5.9,6.4),(7.3,6.9),(5.5,4.4),(7.6,2.2),(6.5,3.2),
                (9.0,7.0),(11.3,6.6),(13.6,4.85),(13.9,4.3),(5.9,2.9),(8.5,1.6),(11.6,1.55)}
  \node[dot] at \p {};
\node[extriv] at (6.6,5.5) {$1+2+\cdots+100 = 5050$};
\node[extriv] at (6.6,2.6) {$2592 = 2^5 \cdot 9^2$};
\node[extriv, font=\tiny] at (11.9,6.35) {$111111111^2 = 12345678987654321$};
\node[extriv, font=\tiny] at (12.75,5.78) {digits 762--767 of $\pi$ read $999999$};
\node[extriv] at (6.6,4.95) {the $1000$th prime is $7919$};
% --- valuable, unrecorded (the harvest) ---
\foreach \p in {(8.9,4.2),(8.7,2.4),(8.6,3.3),(9.9,4.35)}
  \node[star] at \p {};
\node[exval] at (9.45,3.95) {tomorrow's theorems};
\node[exval, densely dashed, align=center, font=\tiny] at (9.3,2.9)
  {this paper?\\ (pending peer review)};
% --- not provable: a falsehood and an independence ---
\node[exhal] at (4.05,1.3) {$2+2=5$};
\node[exhal, align=center] at (13.85,8.3) {CH\\ {\tiny (independent of ZFC)}};
\node[exhal, align=center] at (6.6,8.45) {Whitehead problem\\ {\tiny (independent of ZFC)}};
\node[exhal, align=center, font=\tiny] at (6.0,0.62)
  {Euler's sum-of-powers conjecture\\ (false: $27^5{+}84^5{+}110^5{+}133^5 = 144^5$)};

% ================= arrows from the generator =================
\draw[->, gray!70]        (gen.east) to[bend left=14]  (5.8,6.3);
\draw[->, gray!70]        (gen.east) to[bend right=10] (5.42,4.32);
\draw[->, olive!70!black] (gen.east) to[bend right=6]  (8.78,4.1);
% a hallucination attempt, stopped at the boundary of F
\draw[->, red!70, dashed] (gen.east) to[bend left=24] (4.55,6.85);
\node[red!70, font=\scriptsize] at (4.55,7.1) {$\times$};
\node[lab, red!60, anchor=south, align=center] at (3.95,7.35)
  {hallucinations\\ stopped here};

% ================= zone callouts (outside the frame) =================
\draw[->, gray!80] (15.78,5.7) -- (14.05,4.5);
\node[lab, align=left, gray!35!black, anchor=west] at (15.75,6.1)
  {the \emph{flood}:\\ valid but\\ worthless ($F \setminus H$)};
\draw[->, olive!60!black] (8.4,-0.42) -- (9.9,1.85);
\node[lab, align=center, olive!45!black, anchor=north] at (6.6,-0.25)
  {the \emph{harvest}: unrecorded value $H \setminus C$, mass $1-\alpha$};
\end{tikzpicture}}
\caption{The nested worlds of machine-generated mathematics, with inhabitants.
The verifier guards the boundary of $F$ and nothing more (\cref{thm:taste}):
within $F$ it cannot tell a checkable banality from a theorem worth recording.
Outside $F$ live the falsehoods, from the crude to the seductive (Euler's
sum-of-powers conjecture stood for two centuries before its counterexample),
and, more subtly, the independent statements: validity is always relative to a
fixed formal world, and neither the continuum hypothesis nor the Whitehead
problem is decided by ZFC. The literature $C$ records a density-$\alpha$ fraction
of the valuable language $H$; covering the rest, the harvest, provably requires
emitting infinitely many statements from the flood $F \setminus H$
(\cref{cor:dichotomy}). Placement of examples is illustrative; the identity
$2592 = 2^5 \cdot 9^2$ is \citet{dudeney1917amusements}'s ``printer's error'',
charming, checkable, and still worthless, and this paper itself waits at the edge of $C$, pending peer review.
Color code, here and throughout: blue for
the formal world, gold for the valuable, red for the false or adversarial, gray for
the trivial.}
\label{fig:worlds}
\vspace{-4mm}
\end{figure}
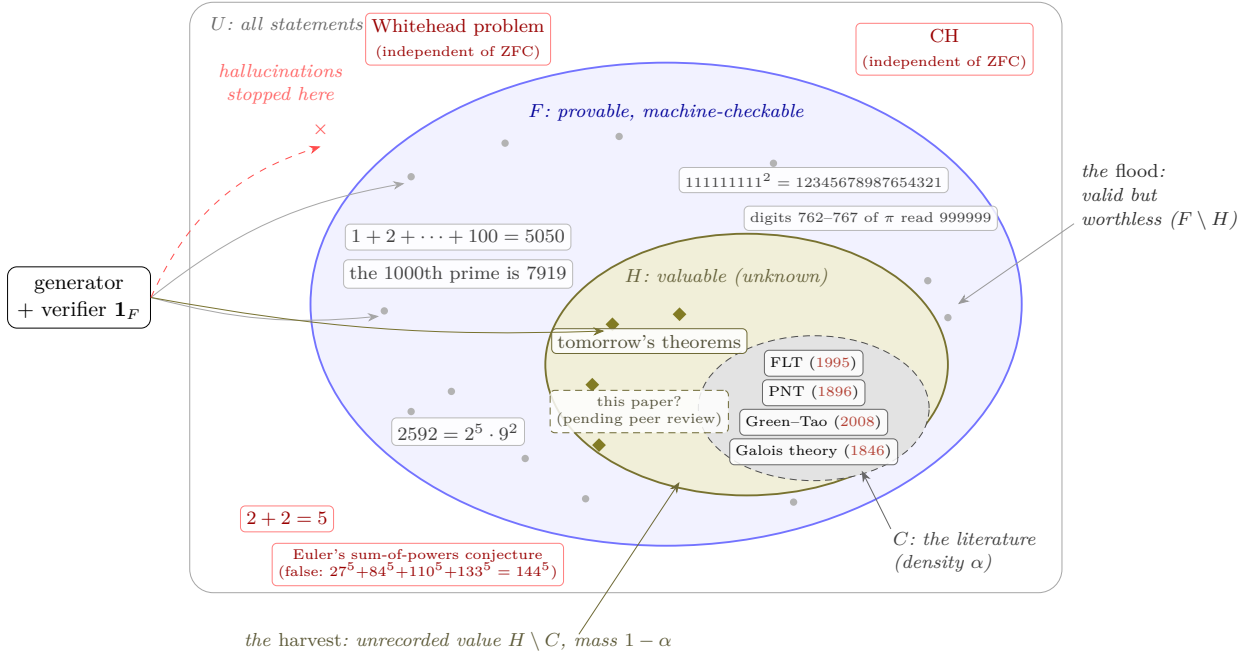

We measure coverage by the lower density $\cov=\dlow{\Out\cap\Val}{\Val}$ of the valuable
statements the run eventually produces, in the limit-set sense of
\citet{kleinberg2025density}. The verifier eliminates one failure mode (no output need ever be
a hallucination), and the question becomes how the remaining two, triviality and missed value,
trade against each other.

\subsection{Contributions}
\label{sec:intro-contrib}

We settle four questions; the headline is the third.

\begin{enumerate}
  \item \textbf{Verification is not taste} (\cref{thm:taste}). With a membership oracle for
    $\Amb$, the collections admitting generation with exact (resp.\ approximate) breadth
    (outputs eventually covering all unseen valuable statements) are \emph{exactly} those
    admitting it in the oracle-free model, characterized fiber-wise by Angluin's tell-tale
    condition (resp.\ its weak form). The verifier's entire informational content, for the
    in-the-limit breadth goal, is identification of the ambient $\Amb$; within a fiber of
    candidates sharing one $\Amb$ it discriminates nothing. Taste must be learned from
    examples, never from a checker.

  \item \textbf{What the verifier does buy} (\cref{thm:flood}). The oracle is not useless: it
    characterizes \emph{sound coverage}: covering every unseen valuable statement while
    asserting only valid ones. With the verifier, the exhaustive generator
    $G_n=\Amb\setminus\Seen_n$ achieves sound coverage for every countable collection; without
    it, there is a collection on which no generator can. The verifier thus relocates the
    unavoidable errors of broad generation from false ($\notin\Amb$) to merely trivial
    ($\in\Amb\setminus\Val$); it does not reduce their number, and it cannot locate value.

  \item \textbf{The trivia dichotomy} (\cref{lem:race}, \cref{thm:caps}, \cref{thm:harvest},
    \cref{cor:dichotomy}). This is the headline. The KW $\alpha/2$ frontier is exactly the
    finite-trivia-count slice of a count dichotomy: on the tight family $\Hfam$ with its
    canonical core presentation (the revealed core $\Core\subseteq K$ is the candidate
    intersection), infinite allowance, even at vanishing rate, yields $1-\alpha/2$, tight,
    with a matching \emph{necessity} theorem (any generator exceeding $\alpha/2$ must emit
    infinitely many trivia). The tight family has revealed core of exact density $\alpha$, and
    the phase transition separates two populations of generators:
  \begin{center}
  \begin{tabular}{@{}lcc@{}}
  \toprule
  \textbf{trivia allowance} & \textbf{optimal coverage at truth $K'$} & \textbf{tight?}\\
  \midrule
  finitely many   & $\alpha/2$     & yes (\cref{thm:caps,thm:harvest})\\
  infinitely many & $1-\alpha/2$ \ {\footnotesize(core $=$ candidate intersection, $\Core\subseteq K$)} & yes (\cref{thm:caps,thm:harvest})\\
  \bottomrule
  \end{tabular}
  \end{center}
A single sound generator $G^*$ attains both ends at every realized truth, with zero trivia
when the record is incomplete (\cref{thm:harvest}). The transition is in the trivia
\emph{count}, not its rate: any infinite allowance suffices even at trivia count $\lfloor\sqrt
N\rfloor$ in the first $N$ rounds, i.e.\ rate $O(N^{-1/2})$. The gap $1-\alpha$ between the
two optima is exactly the unrecorded mass $\Amb$-valuable but outside the core. Thus an
unbounded stream of correct-but-worthless theorems is \emph{necessary} for any generator
covering more than $\alpha/2$, and \emph{sufficient} at asymptotically negligible rate. Beyond
the tight family, the sweep is universal: it gives $\cov\ge 1-\dup{\Core}{\Val}$ for every
countable collection.

  \item \textbf{Instantiation in a compression model of mathematics} (\cref{thm:abelian},
    \cref{prop:free}, \cref{prop:precision}). We realize the dichotomy inside the
    macro-dictionary model of \citet{aksenov2026compression}. Structured abelian refinement
    chains are identifiable with \emph{singleton} tell-tales: in the structured regime,
    valuable generation needs no trivia at all (\cref{thm:abelian}). Free-monoid name-cut
    dictionary pairs realize the tight family at every rational $\alpha$, and an unbounded-cut
    construction realizes the dichotomy's extreme $0$-vs-$1$ point: finite-trivia coverage $0$
    against sweep coverage $1$ (\cref{prop:free}). The exhaustive generator's precision is
    polynomially small in radius, while the sweep's precision tends to $1$
    (\cref{prop:precision}). We use this model as substrate and motivation, not as a source of
    borrowed theorems.
\end{enumerate}

\subsection{Technical overview}
\label{sec:intro-technique}

The four results draw on three mechanisms.

\emph{Fixed-advice relativization} carries \cref{thm:taste}. Within a fiber, the candidates
sharing one ambient $\Amb$, the oracle is one fixed function $\mathbf 1_{\Amb}$, so a
relativized generator's adaptive query tree, of finite depth, resolves to an ordinary
set-valued function of the enumerated prefix. The oracle-free characterizations of
\citet{kalavasis2024characterizations} then apply verbatim, and a one-round identification
stage locates the correct fiber. We say plainly that this transfer is light by design: its
easiness \emph{is} the message, that verification is information-theoretically orthogonal to
taste. The technical weight lives in the next two mechanisms.

\emph{A sandwich-exclusion lemma, deployed two-sidedly,} carries the separation
\cref{thm:flood}. The lemma itself is a three-line containment argument: on a nested Gold
chain $\Val_j\uparrow\Val_\infty$ with fattened ambients, no single output set can sound-cover
both a chain member and the union against the same seen-set. The two-sidedness lives not in
the lemma but in how the adaptive phase construction deploys it: forcing any oracle-free
generator to fail on one or the other infinitely often, so that soundness must be denied
\emph{above} and \emph{below} the same transcript.

\emph{A race lemma and a sweep pointer} carry the dichotomy. The lower bounds
(\cref{lem:race}, \cref{thm:caps}) come from a racing adversary that enumerates the core from
the bottom, so that any generator and the adversary read the \emph{same} transcript under two
different truths, capping coverage at $1/2$ along the raced core and at $1-\alpha/2$
overall---a same-transcript two-readings cap. The upper bounds (\cref{thm:harvest}) come from
interleaving a sparse sweep schedule into the race: because density is a property of the limit
set, not of any per-prefix deadline, a vanishing-rate sweep of the candidate intersection
eventually claims every never-enumerated valuable string, recovering $1-\alpha/2$. The
matching constructions of \cref{thm:abelian} and \cref{prop:free} verify that the abstract
interface the lower bound requires is met by genuine compressible-language families.

\subsection{Related work}
\label{sec:intro-related}

\paragraph{Generation in the limit.}
The model originates with \citet{kleinberg2024language}: every countable collection is
generable in the limit, in contrast to Gold--Angluin identification, but breadth is
sacrificed. The breadth resource was then mapped:
\citet{kalavasis2025limits,kalavasis2024characterizations} characterize exact and approximate
breadth by Angluin's tell-tale condition and its weak form, the characterizations we
relativize in \cref{thm:taste}. Their pivotal positive result, that \emph{positive and
negative} examples of the target restore breadth for every collection \citep[Thm.\
3.13]{kalavasis2025limits}, marks precisely the contrast we exploit: true negative examples
collapse the Angluin boundary, whereas our verifier supplies negatives only \emph{outside}
$\Amb$ and none in the grey zone $\Amb\setminus\Val$, so it cannot move the boundary at all.
The oracle results of the line concern different oracles: \citet{charikar2025facets} show
membership queries to \emph{collection members} are insufficient, and \citet{bai2026noise}
show finite same-level feedback is worth nothing; our oracle queries an ambient
\emph{superset}, a different construction with a different verdict.

\paragraph{Density.}
\Citet{kleinberg2025density} introduce the limit-set density we use and achieve a universal
positive density. \Citet{kleinberg2026partial} sharpen this: full enumeration permits tight
lower density $1/2$, and partial enumeration of a core of density $\alpha$ permits the tight
bound $\alpha/2$. We show that $\alpha/2$ is exactly the \emph{finite-trivia} regime of our
larger dichotomy; their model reveals a fraction of the same target, whereas our nesting sits
between a \emph{target} and an \emph{ambient}, not between a revealed and an unrevealed part
of one language. \Citet{kleinberg2026banach} establish a Banach-density dichotomy by
Cantor--Bendixson rank, which we cite as a caution for embedded instantiations.

\paragraph{The errors-buy-coverage moral, elsewhere.}
\Citet{ganju2026timely} reach a kindred slogan, that sparse hallucination beats mode collapse,
in a \emph{timeliness} model with no nested pair, no verifier, and no distinction between
trivial and false. \Citet{anastasopoulos2026safe} study the opposite geometric task, avoiding
a harmful sublanguage rather than hitting a valuable one.

\paragraph{Verifier-assisted decoding.}
\Citet{botta2025verifier} also place a verifier beside a generator, but ask a different
question: a process verifier supplies prefix-completability for constrained decoding, and the
results are computational query complexity, with no unknown target and no breadth or density.
We share a word, not a model.

\paragraph{Other variants.}
The line includes generation through a learning-theoretic lens \citep{li2025generation},
union-closedness failures \citep{hanneke2025union}, representative generation
\citep{peale2025representative}, noisy examples \citep{raman2025noisy,li2026quantifying},
Pareto-optimal non-uniform generation \citep{charikar2025pareto}, contamination
\citep{mehrotra2025contamination}, privacy \citep{mehrotra2026private}, metric spaces
\citep{li2026metric}, agnostic generation and identification \citep{hogsgaard2026agnostic},
and complexity barriers \citep{arenas2025complexity}. None carries a verifiable ambient or the
triviality mode.

\paragraph{Mathematics as structured strings.}
Our instantiation is built on the compression model of \citet{aksenov2026compression} and the
structural program of \citet{barkeshli2026structure}; we use the model, not its theorems.
Related empirical and conceptual work on the structure and interestingness of mathematical
corpora
\citep{wernhard2025grammar,colton2000interestingness,herrmann2026interestingness,tsoukalas2025interestingness}
motivates the value question but proves no generation theorems. The empirical AI4Math line,
its benchmarks and its autoformalization and proof-search pipelines
\citep{zheng2022minif2f,yang2023leandojo,jiang2023draft}, supplies exactly the phenomena our
model formalizes.

\vspace{2mm}
We defer the discussion of classical inductive inference and language identification in the
limit, the broader line of work on language generation in the limit, and practical
developments in machine-generated mathematics to \cref{app:related}.

\section{The Model}
\label{sec:prelim}

We pose the model before stating any result, because every answer below changes with it. The
objects are languages (subsets of a fixed countable universe), the protocol is adversarial
enumeration in the limit, and the question is how much of a hidden \emph{valuable} language a
generator can cover without ever leaving a fixed \emph{formal} one.

\subsection{Ground set, generation in the limit, and breadth}
\label{sec:prelim-km}

Fix a countable \emph{ground set} $\gnd$ of statements with a fixed canonical total order
$\prec$, written $x_1 \prec x_2 \prec \cdots$. For an infinite $L \subseteq \gnd$ we write
$\rest{L}{M}$ for the set of its first $M$ elements under $\prec$ (so $\rest{L}{M}$ is always
an initial segment of $L$ in $\prec$-order, and indexing is $1$-based: $\rest{L}{1} =
\{\min_\prec L\}$). A \emph{language} is an infinite subset of $\gnd$; all languages in this
paper are infinite unless stated otherwise.

\paragraph{The Kleinberg--Mullainathan protocol.}
We use the generation-in-the-limit protocol of \citet{kleinberg2024language}. An adversary
fixes an unknown target language $\Val$ and a \emph{Kleinberg--Mullainathan (KM) enumeration}
of it: an infinite stream $w_1, w_2, \dots$ whose set of values is exactly $\Val$ (every
element of $\Val$ appears, repetitions allowed). At round $t$ the generator has seen the
prefix and we write $\Seen_t \defeq \{w_1, \dots, w_t\}$ for the \emph{seen-set}. The
generator must commit to an output before round $t+1$. Two output conventions appear in the
literature, and we keep both, in disjoint roles.

\paragraph{Set-valued generators (breadth; KMV conventions).}
A \emph{set-valued generator} $G$ maps the seen-set to a set $G_n(\Seen_n) \subseteq \gnd$;
this is the convention of \citet{kalavasis2024characterizations} (hereafter KMV), used for the
\emph{characterization} results of \cref{sec:res-taste,sec:res-flood}. We say $G$ achieves
\emph{exact breadth in the limit} on $\Val$ if there is $n^*$ with $G_n = \Val \setminus
\Seen_n$ for all $n \ge n^*$ (it outputs all and only the unseen valuable statements), and
\emph{approximate breadth in the limit} if eventually $G_n \subseteq \Val$ and $\abs{\Val
\setminus G_n} < \infty$. No computability is assumed of $G$: it is an arbitrary function of
the prefix (KMV Remark~2.3).

\paragraph{Element-based generators and the limit set (KW conventions).}
An \emph{element-based generator} maps the seen-set to a single new statement $a_t =
G(\Seen_t) \in \gnd$ with $a_t \notin \Seen_t$; this is the convention of
\citet{kleinberg2026partial} (hereafter KW), used for the \emph{density} results of
\cref{sec:res-dichotomy}. The object of study over an infinite run is the \emph{limit set} of
all outputs,
\begin{equation}\label{eq:limitset}
  \Out \defeq \{\, a_t : t \ge 1 \,\},
\end{equation}
the set of \emph{all} statements the generator ever emits across the entire run (KW
Definition~1.1). We emphasize that $\Out$ is a property of the whole infinite run, not of any
finite stage. To measure how much of a language a set covers we use prefix densities in the
canonical order: for $L' \subseteq L$,
\begin{equation}\label{eq:density}
  \dlow{L'}{L} \defeq \liminf_{N \to \infty} \frac{\abs{L' \cap \rest{L}{N}}}{N},
  \qquad
  \dup{L'}{L} \defeq \limsup_{N \to \infty} \frac{\abs{L' \cap \rest{L}{N}}}{N},
\end{equation}
writing $\dens{L'}{L}$ for the common value when the limit exists. Generators in this mode are
deterministic.

\paragraph{Angluin's conditions.}
The combinatorial quantity that controls breadth is Angluin's tell-tale condition
\citep{angluin1980inductive}. A countable collection $\Hcal$ of languages satisfies
\emph{Angluin's condition} if every $L \in \Hcal$ has a finite \emph{tell-tale} $T_L \subseteq
L$ such that no $L' \in \Hcal$ with $T_L \subseteq L'$ is a proper subset of $L$. It satisfies
the \emph{weak Angluin condition} if the tell-tale separates $L$ from every $L' \in \Hcal$
with $\abs{L \setminus L'} = \infty$ rather than from every proper subset. These are exactly
the conditions under which a Gold-style ascending chain $L_1 \subsetneq L_2 \subsetneq \cdots$
with union $\bigcup_j L_j$ in the collection \emph{fails} to be identifiable, since every
finite subset of the union lies in some proper sub-language of the chain
\citep{gold1967language}.

\subsection{Partial enumeration and the revealed core}
\label{sec:prelim-partial}

The density results take place in KW's \emph{partial} model \citep{kleinberg2026partial},
which weakens the protocol in exactly the way the mathematical reading demands: the adversary
need not enumerate all of $\Val$. Instead the adversary enumerates an infinite \emph{revealed
core} $\Core \subseteq \Val$ (the ``literature'': the valuable statements actually written
down) and may leave the rest of $\Val$ forever unseen. We parametrize the core by its lower
density in the target,
\begin{equation}\label{eq:alpha}
  \alpha \defeq \dlow{\Core}{\Val} ,
\end{equation}
and speak of an instance $(\Val, \Core, \text{enumeration})$ as \emph{legal} if the
enumeration is a repetition-free listing whose value-set is $\Core$. The relevant restated
results are the following.

We restate both results verbatim, version-locked to the public preprint, so the later argument
can quote them with no paraphrase. In their statement, \emph{generation in the limit with
partial enumeration} means that ``for some time $t^*$, its outputs $a_t$ satisfy $a_t \in K -
S_t$ for all $t \ge t^*$'' \citep[\S1]{kleinberg2026partial}, where $K$ is the true language
and $S_t$ the enumerated prefix.

\begin{fact}[Partial-enumeration generation; {\citep[Theorem~1.6 in arXiv version]{kleinberg2026partial}}]\label{fact:kw16}
There is an algorithm $\mathcal A$ that achieves generation in the limit with partial
enumeration with the following guarantee: ``for any adversarial enumeration $E$ of an infinite
subset $C$ of one of the languages $K \in \mathcal X$ such that $C$ has lower density at least
$\alpha > 0$ in $K$, the set of output strings $O(E, \mathcal A)$ $[\dots]$ has a lower
density in $K$ that is at least $\alpha/2$.'' Moreover, ``$\alpha/2$ also serves as an upper
bound on the lower density achievable by any algorithm.'' In our notation: for every countable
collection $\mathcal X$ there is an element-based generator with $\dlow{\Out}{\Val} \ge
\tfrac12\,\dlow{\Core}{\Val}$ on every legal instance, and from some round $t^*$ on it outputs
only elements of $\Val$.
\end{fact}

\begin{fact}[Chain validity in the partial model; {\citep[Lemma~2.5(1) in arXiv version]{kleinberg2026partial}}]\label{fact:kw25}
For the identified-intersection sequence $\mathcal I^{(t)}$ that KW's generator maintains,
``there is a finite time $T$ such that after which all the identified intersection $\mathcal
I^{(t)}$ are valid,'' where valid means $\mathcal I^{(t)} \subset K$. The statement and its
proof are in the partial-enumeration model. In our notation: there is a finite round $T_0$
with $\mathcal I^{(t)} \subseteq \Val$ for all $t \ge T_0$.
\end{fact}

\subsection{The nested model: validity versus value}
\label{sec:prelim-nested}

We now state our contribution to the model. Mathematical generation distinguishes two notions
that the single-language model conflates: a statement can be \emph{valid} (formally provable,
machine-checkable) without being \emph{valuable} (worth proving). We formalize this by
nesting.

\begin{definition}[Pair model]\label{def:pairmodel}
A \emph{pair model} is a countable collection $\Pairs = \{(\Amb_i, \Val_i)\}_i$ of pairs of
infinite languages with $\Val_i \subseteq \Amb_i \subseteq \gnd$. We call $\Amb_i$ the
\emph{formal} language (its membership oracle is a proof checker) and $\Val_i$ the
\emph{valuable} language (the unknown target). The \emph{fiber} over a set $F$ is
\begin{equation}\label{eq:fiber}
  \fiber{F} \defeq \setbuild{\Val_i}{\Amb_i = F},
\end{equation}
the collection of candidate targets sharing the formal world $F$ \emph{as a set}. The
adversary picks an index $z$, fixes $(\Amb,\Val) = (\Amb_z, \Val_z)$, and presents an
enumeration of its core (full enumeration of $\Val$ in the breadth setting).
\end{definition}

\begin{definition}[Verifier oracle]\label{def:verifier}
A \emph{relativized} (verifier-assisted) generator may, in each round, issue finitely many
adaptive membership queries to the indicator $\mathbf 1_{\Amb}$ of the true formal language (a
perfect proof checker for the realized ambient) before producing its output. The query budget
per round is finite but unbounded; no computability is assumed (KMV Remark~2.3 conventions).
The adversary is \emph{non-effective}: it may decide infinitary conditions on the generator's
behavior. This is the standard convention for impossibility results in this line
\citep{kalavasis2024characterizations}.
\end{definition}

\begin{definition}[Error trichotomy]\label{def:trichotomy}
At a round in which the generator emits $a$, with true pair $(\Amb, \Val)$ and seen-set
$\Seen_t$, the output is a \emph{hallucination} if $a \notin \Amb$ (formally false), a
\emph{triviality} if $a \in \Amb \setminus \Val$ (valid but worthless), and \emph{valuable} if
$a \in \Val \setminus \Seen_t$ (a genuinely new valuable statement; being in $\Val \subseteq
\Amb$, it is in particular valid). Throughout, \emph{valid} refers to membership in the formal
world $\Amb$, \emph{valuable} to membership in the target $\Val$. The \emph{triviality rate}
through round $N$ is
\begin{equation}\label{eq:trv}
  \trv_N \defeq \frac{\abs{\{\, t \le N : a_t \in \Amb \setminus \Val \,\}}}{N},
\end{equation}
and the \emph{trivia count} is $\abs{\Out \cap (\Amb \setminus \Val)}$. We distinguish count
from rate throughout: a generator may emit infinitely many trivia at rate tending to $0$.
\end{definition}

\begin{definition}[Coverage, soundness, sound coverage]\label{def:soundcov}
In the density setting, the \emph{coverage} of a generator at truth $\Val$ is $\cov \defeq
\dlow{\Out \cap \Val}{\Val}$, the lower density of the valuable statements it emits. A
generator is \emph{sound} if it never hallucinates: $a_t \in \Amb$ for all $t$, i.e.\ $\Out
\subseteq \Amb$. A set-valued generator achieves \emph{sound coverage in the limit} on $(\Amb,
\Val)$ if there is $n^*$ with
\begin{equation}\label{eq:soundcov}
  \Val \setminus \Seen_n \;\subseteq\; G_n(\Seen_n) \;\subseteq\; \Amb
  \qquad \text{for all } n \ge n^*:
\end{equation}
it covers every unseen valuable statement while asserting only formally valid ones.
\end{definition}

Sound coverage is the natural ``cover everything true and worthwhile, hallucinate nothing''
goal; \eqref{eq:soundcov} sandwiches the output between the unseen value (lower) and the
formal world (upper). The trichotomy of \cref{def:trichotomy} is what makes the nested model
richer than the single-language one: errors no longer collapse to a single kind, and
\cref{sec:res-flood} shows the verifier acts precisely on which kind occurs.

\paragraph{Restated characterizations (cited facts).}
The breadth results of \cref{sec:res-taste} reduce to KMV's oracle-free characterizations,
which we use verbatim.

\begin{fact}[Breadth characterizations; {\citep[Theorems~3.3 and 3.8]{kalavasis2024characterizations}}]\label{fact:kmv}
A countable collection $\Hcal$ admits a set-valued
generator achieving exact breadth in the limit if and only if $\Hcal$ satisfies Angluin's
condition (Theorem~3.3), and admits approximate breadth in the limit if and only if $\Hcal$
satisfies the weak Angluin condition (Theorem~3.8).
\end{fact}

\paragraph{The bridge between modes.}
The two output conventions are reconciled by the standard set-versus-element bridge (a KW
analogue of KMV Lemma~2.3): a set-valued breadth statement and an element-based density
statement about the same family agree on which statements are eventually emitted. We use
set-valued generators only in \cref{sec:res-taste,sec:res-flood} (where the question is
\emph{whether} a family is coverable) and element-based generators only in
\cref{sec:res-dichotomy} (where the question is \emph{how much} of it is covered).

\section{Main Results}
\label{sec:results}

We organize the results around one question: in the nested model, what does a perfect verifier
buy, and what is the price of covering more value than the verifier alone can locate?
\Cref{sec:res-taste} shows the verifier does \emph{not} buy taste: breadth is governed by the
same Angluin condition as without it. \Cref{sec:res-flood} shows what it does buy: sound
coverage, achievable for every model with the verifier and impossible in general without it,
by relocating errors from false to trivial. Then \cref{sec:res-dichotomy} prices that
relocation exactly --- a phase transition in trivia \emph{count} --- and
\cref{sec:res-instantiation} instantiates the whole picture as a case study in a compression
model of mathematics.

\subsection{Verification is not taste}
\label{sec:res-taste}

The first result locates the entire information content of a perfect verifier, for the breadth
goal. One might hope that a proof checker, by certifying validity, helps a generator find
\emph{value}, distinguishing the valuable target $\Val$ from the worthless remainder $\Amb
\setminus \Val$. It does not. The verifier's answers depend only on the formal world $\Amb$,
which is shared across a fiber; within a fiber it supplies zero bits about which candidate is
the true target. Breadth is therefore governed, fiber by fiber, by exactly the Angluin
condition of the oracle-free model.

\begin{restatable}[Verification is not taste]{theorem}{ThmTaste}\label{thm:taste}
Let $\Pairs = \{(\Amb_i, \Val_i)\}$ be a countable pair model (\cref{def:pairmodel}), with a
relativized generator (\cref{def:verifier}) that queries the perfect verifier $\mathbf
1_{\Amb}$.
\begin{enumerate}[label=\textup{(\alph*)}]
  \item Exact breadth in the limit on $\Pairs$ is achievable if and only if every fiber
    $\fiber{\Amb}$ satisfies Angluin's condition.
  \item Approximate breadth in the limit on $\Pairs$ is achievable if and only if every fiber
    satisfies the weak Angluin condition.
  \item In the single-ambient case $\Amb_i \equiv \Amb$, the relativized characterization
    coincides with the oracle-free characterization of \cref{fact:kmv}. In the multi-ambient
    case the verifier's entire information content, for the in-the-limit breadth goal, is
    exactly ambient identification: locating which $\Amb$ is realized, after which it adds
    nothing.
  \item \textup{(Zero-discrimination.)} For $\Val_i, \Val_j$ in the same fiber the oracle is
    the identical function $\mathbf 1_{\Amb}$; the realized query path and output at any
    prefix $\Seen_n$ are identical across the two instances. Every behavioral distinction
    within a fiber is driven by the enumeration alone.
\end{enumerate}
\end{restatable}

\begin{proof}[Proof sketch]
For necessity, restrict the adversary to one fiber: the oracle is then a single fixed
function, so stripping it yields an ordinary set-valued generator that would have to beat
\cref{fact:kmv} on a violating fiber. For sufficiency, a sample-independent first stage
queries canonical points to identify the ambient $\Amb$ in finite time, after which a fixed
per-fiber breadth generator (which exists by \cref{fact:kmv}) runs to completion. Part~(d) is
immediate from $\mathbf 1_{\Amb_i} = \mathbf 1_{\Amb_j}$. Full proof in \cref{app:taste}.
\end{proof}

Part~(c) holds because the verifier's answers are determined by $\Amb$ alone: in a
single-ambient model $\Amb$ is known a priori, and within a fiber the oracle discriminates
nothing (part~(d)). Note also that (c) is scoped to breadth: for other goals, such as sound
coverage in \cref{sec:res-flood}, the verifier does strictly more. The contrast that makes (c)
sharp: a single \emph{negative} example inside $\Amb \setminus \Val$ would collapse the
Angluin boundary entirely, but the verifier returns only ``valid'' on all of $\Amb$ and so
never supplies one. Verification certifies the world; it does not reveal the worth of a
statement within it.

\subsection{Sound coverage: what the verifier does buy}
\label{sec:res-flood}

If the verifier does not buy taste, what does it buy? Exactly sound coverage, and exactly by
relocating unavoidable errors. With the verifier, the exhaustive generator $G_n = \Amb
\setminus \Seen_n$ already sound-covers \emph{every} pair model: it asserts all unseen formal
statements, hence all unseen valuable ones, and never a falsehood. Without the verifier no
such guarantee is possible in general: there is a pair model on which no oracle-free generator
sound-covers. The price the exhaustive generator pays is visible in the trichotomy: when $\Amb
\setminus \Val$ is infinite, $\Amb \setminus \Seen_n$ contains infinitely many trivia. The
verifier has not removed the errors; it has moved them from the false region $\gnd \setminus
\Amb$ into the trivial region $\Amb \setminus \Val$.

\begin{figure}[t]
\centering
\resizebox{0.97\textwidth}{!}{% Figure: the sandwich-exclusion engine
\begin{tikzpicture}[>=Stealth, font=\footnotesize,
  bar/.style={draw=olive!60!black, fill=olive!18, rounded corners=1.5pt},
  tail/.style={pattern=north east lines, pattern color=red!55, draw=red!40, rounded corners=1.5pt},
  gband/.style={fill=blue!35, fill opacity=0.28, draw=blue!60!black, densely dashed, rounded corners=2pt},
]
% ---------- the union L* (top) ----------
\draw[draw=black!60, densely dashed, fill=gray!8, rounded corners=1.5pt]
  (0.5,4.4) rectangle (11.0,4.85);
\node[anchor=west, font=\scriptsize] at (11.2,4.62)
  {$L^* = \bigcup_j \Val_j$\quad (truth $(\Amb_\infty, L^*)$)};
% the finite seen-set S
\foreach \x in {1.5, 3.4, 7.7} \fill[black] (\x,4.62) circle (1.5pt);
\node[font=\tiny, anchor=north] at (3.4,4.12) {$\Seen$ (finite)};
\draw[->, gray!70, very thin, shorten >=2pt] (3.4,4.16) -- (3.4,4.52);
% demand 1: G covers L* minus S
\draw[gband] (0.62,4.33) rectangle (10.88,4.92);
\node[blue!55!black, font=\scriptsize, anchor=east] at (0.35,4.62)
  {$G \supseteq$};
\node[blue!55!black, font=\scriptsize, anchor=south west] at (0.55,5.05)
  {demand of the union: $G \supseteq L^* \setminus \Seen$, all but finitely many};
% ---------- the chain ----------
% H_3
\draw[bar]  (0.5,2.9) rectangle (9.3,3.35);
\draw[tail] (9.3,2.9) rectangle (11.0,3.35);
\node[anchor=west, font=\scriptsize] at (11.2,3.12) {$\Val_3$, missing $R_{\ge 3}$};
% H_2 (the realized one) with the G band forced inside
\draw[bar]  (0.5,1.7) rectangle (8.0,2.15);
\draw[tail] (8.0,1.7) rectangle (11.0,2.15);
\node[anchor=west, font=\scriptsize] at (11.2,1.92) {$\Val_2$, missing $R_{\ge 2}$};
\draw[gband] (0.62,1.63) rectangle (10.88,2.22);
\draw[red!70, very thick, rounded corners=2pt] (8.0,1.6) rectangle (10.92,2.25);
\node[red!60!black, font=\scriptsize, anchor=north west, align=left] at (6.2,1.28)
  {demand of the chain element: $G \subseteq \Val_j$ --- yet the band covers the tail:\\
   the \emph{infinite} $R_{\ge j}$ must hide inside the \emph{finite} $\Seen$:
   contradiction (\cref{lem:sandwich}).};
% H_1
\draw[bar]  (0.5,0.0) rectangle (6.0,0.45);
\draw[tail] (6.0,0.0) rectangle (11.0,0.45);
\node[anchor=west, font=\scriptsize] at (11.2,0.22) {$\Val_1$, missing $R_{\ge 1}$};
% nesting marks
\node[font=\scriptsize, gray!60!black] at (5.75,2.62) {$\subsetneq$};
\node[font=\scriptsize, gray!60!black] at (5.75,1.12) {$\subsetneq$};
\node[font=\scriptsize, gray!60!black] at (5.75,3.9) {$\subsetneq \ \cdots \ \nearrow$};
% fresh-index note
\node[font=\scriptsize, gray!30!black, anchor=north west, align=left] at (0.5,-0.45)
  {\cref{lem:fresh}: every finite $\Seen$ sits inside some chain element, so the
   adversary can pose the dilemma at ever-larger $j$, forever.};
\end{tikzpicture}}
\caption{The engine of the separation. The targets form a Gold chain
$\Val_1 \subsetneq \Val_2 \subsetneq \cdots$ (bars) with union $L^*$ (top); each
chain element misses an infinite tail $R_{\ge j}$ (hatched). Sound coverage of
the union forces the output set $G$ (dashed band) to cover all of $L^*$ but the
finite seen-set; sound coverage of a chain element forces the same $G$ inside
$\Val_j$, i.e.\ off its tail. The band over the hatched region (outlined) is
the contradiction (\cref{lem:sandwich}), and non-identifiability
(\cref{lem:fresh}) lets the adversary pose the dilemma infinitely often.}
\label{fig:sandwich}
\end{figure}

\begin{restatable}[The exhaustive generator and the separation]{theorem}{ThmFlood}\label{thm:flood}
In the pair model with the protocol of \cref{def:pairmodel}:
\begin{enumerate}[label=\textup{(\alph*)}]
  \item \textup{(Exhaustive generator)} With the verifier $\mathbf 1_{\Amb}$, the generator
    $G_n = \Amb \setminus \Seen_n$ achieves sound coverage in the limit on every countable
    pair model: at every round, it has zero hallucination, and (when $\Amb \setminus \Val$ is
    infinite) infinitely many trivia.
  \item \textup{(Separation)} There is a countable pair model $\Pairs^{\dagger}$ (all
    languages infinite, every $\Amb_i \setminus \Val_i$ infinite) on which no oracle-free
    set-valued generator (an arbitrary function of the prefix) achieves sound coverage in the
    limit.
  \item The verifier strictly enlarges the class of pair models admitting sound coverage, from
    a proper subclass to all of them; this is in contrast to breadth, where within a fiber it
    adds nothing \textup{(\cref{thm:taste}(c))}.
\end{enumerate}
\end{restatable}

\begin{proof}[Proof sketch]
Part~(a) is immediate from $\Val \subseteq \Amb$. For (b), the model $\Pairs^{\dagger}$ is
built on a Gold-style chain $\Val_1 \subsetneq \Val_2 \subsetneq \cdots$ with union $L^*$ in
the collection. A sandwich-exclusion lemma (\cref{lem:sandwich}; \cref{fig:sandwich}) shows no
output set can sound-cover both the union $(\Amb_\infty, L^*)$ and a chain element $(\Amb_j,
\Val_j)$ against the same finite seen-set; a fresh-index lemma (\cref{lem:fresh}) keeps the
chain non-identifiable. An adaptive phase construction then forces any oracle-free generator
to fail on one of the realized instances; the construction is a \emph{non-effective}
adversary, the standard convention for impossibility results in this line
(\cref{def:verifier}). Part~(c) combines (a) and (b). Full proof in \cref{app:flood}.
\end{proof}

\begin{remark}[Two regimes in part~(c)]
Where the oracle identifies the ambient (and hence the target), as on $\Pairs^{\dagger}$,
whose fibers are singletons, it suffices for sound coverage, and its absence makes sound
coverage impossible. Within a non-singleton fiber, by contrast, the oracle adds nothing about
which target is valuable (\cref{thm:taste}(d)), so there its power is exactly to keep the
unavoidable overshoot inside $\Amb$: relocating errors from false to trivial, reducing neither
their number nor locating value.
\end{remark}

\subsection{The trivia phase transition}
\label{sec:res-dichotomy}

Throughout this subsection the ambient is $\Amb = \N$ and the verifier is the constant
function: the dichotomy is a verifier-free theorem about nested targets. The verifier's role
in the paper is settled by \cref{thm:taste,thm:flood}, and what follows quantifies the
value/trivia tradeoff that remains after verification is granted.

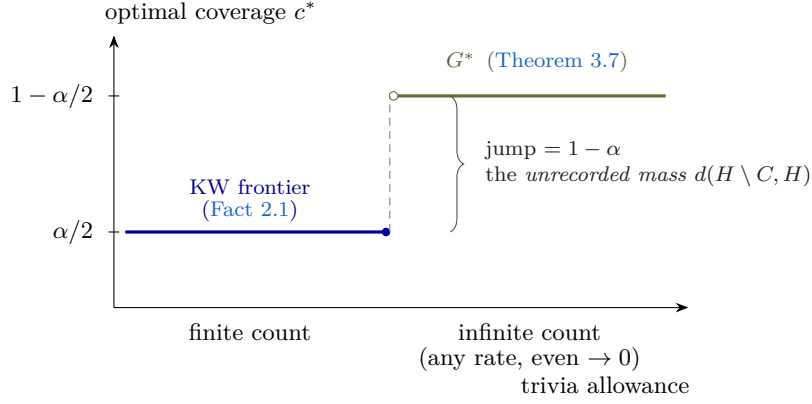
\begin{figure}[t]
\centering
% Figure: the trivia phase transition (step diagram)
\begin{tikzpicture}[>=Stealth, font=\footnotesize, scale=1.0]
% axes
\draw[->] (0,0) -- (7.6,0);
\node[anchor=north east, font=\footnotesize] at (7.75,-0.78) {trivia allowance};
\draw[->] (0,0) -- (0,3.6);
\node[anchor=south west, font=\footnotesize] at (-0.25,3.62) {optimal coverage $\cov^*$};
% y ticks
\draw (-0.07,1.0) -- (0.07,1.0) node[left=6pt] {$\alpha/2$};
\draw (-0.07,2.8) -- (0.07,2.8) node[left=6pt] {$1-\alpha/2$};
% x region labels
\node[below=2pt] at (1.8,0) {finite count};
\node[below=2pt, align=center] at (5.5,0) {infinite count\\[-2pt] (any rate, even $\to 0$)};
% the step
\draw[very thick, blue!60!black] (0.15,1.0) -- (3.6,1.0);
\draw[very thick, olive!60!black] (3.7,2.8) -- (7.3,2.8);
\fill[blue!60!black] (3.6,1.0) circle (1.6pt);
\draw[olive!60!black, fill=white] (3.7,2.8) circle (1.6pt);
\draw[densely dashed, gray] (3.65,1.0) -- (3.65,2.8);
% the jump brace
\draw[decorate, decoration={brace, amplitude=5pt}, gray!60!black]
  (4.45,2.8) -- (4.45,1.0);
\node[right=8pt, align=left, gray!25!black, font=\scriptsize] at (4.5,1.9)
  {jump $= 1-\alpha$\\ the \emph{unrecorded mass} $\dens{\Val \setminus \Core}{\Val}$};
% endpoints annotated
\node[font=\scriptsize, blue!50!black, align=center] at (1.8,1.42)
  {KW frontier\\[-2pt] (\cref{fact:kw16})};
\node[font=\scriptsize, olive!50!black, align=center] at (5.6,3.25)
  {$G^*$ \,(\cref{thm:harvest})};
\end{tikzpicture}
\caption{The trivia phase transition (\cref{cor:dichotomy}). The optimum
jumps by exactly the unrecorded mass $1-\alpha$ as the trivia allowance
crosses from finite to infinite \emph{count}; the rate at which trivia are
emitted is irrelevant, and the right endpoint is attained by the single
generator $G^*$ at vanishing rate. The colors are the title's: blue for the
raced core, gold for the harvested surplus.}
\label{fig:phase}
\end{figure}

We now price the relocation exactly, in the partial (density) model. The setting is a single
family on which the optimal guaranteed coverage is determined --- and on which a sharp
dichotomy appears: with only finitely many trivia the optimum is $\alpha/2$, and allowing
infinitely many trivia (even at rate tending to $0$) jumps it to $1 - \alpha/2$. The
transition is in trivia \emph{count}, not rate.

\paragraph{The tight family.}
We work over $\gnd = \N$ in numeric order. Let $K' \defeq \{\text{evens}\}$ with $K'[j] =
2(j-1)$ be the formal world's coverable part, and for a target density $\alpha \in (0,1)$
define
\begin{equation}\label{eq:Halpha}
  K \defeq \{\, K'[\lceil k/\alpha \rceil] : k \ge 1 \,\}
  \subsetneq K' , \qquad
  D \defeq K' \setminus K , \qquad \Amb = \N .
\end{equation}
The family is $\Hfam \defeq \{K, K'\}$, with $K$ the sparse \emph{revealed} target and $K'$
the full coverable world. The point of \eqref{eq:Halpha} is that it realizes \emph{exact}
density for \emph{every} real $\alpha$, not merely rational $\alpha$:
\begin{equation}\label{eq:exactdensity}
  \abs{K \cap \rest{K'}{N}} = \lfloor \alpha N \rfloor
  \quad\text{for all } N, \qquad\text{hence}\qquad \dens{K}{K'} = \alpha \text{ exactly},
\end{equation}
and moreover $K \cap \rest{K'}{N} = \rest{K}{m_N}$ with $m_N = \lfloor \alpha N \rfloor$,
while $\Amb \setminus K' = \{\text{odds}\}$ is infinite. The two readings $K$ and $K'$ share a
single transcript when the adversary enumerates $K$, which is what makes the caps bind for
every generator. Here $K$ plays the role of the valuable target $\Val$, $K'$ that of the
coverable part of $\Amb$, and $\Amb \setminus K'$ (the odd column) is the inexhaustible trivia
reservoir; the verifier is inert since $\Amb = \N$ is known.

The lower bounds rest on a self-contained race lemma (\cref{fig:race}). It says two things at
once: a generator following a greedy ``race'' rule covers at least half of a target, and an
adversary running a ``racing'' enumeration caps \emph{every} generator at half.

\begin{definition}[The racing adversary]\label{def:racingadv}
For infinite $K \subseteq \gnd$, the \emph{racing adversary} $A_K$ enumerates, at every round
$t \ne 2^n$, the $\prec$-least element of $K \setminus (\Seen_{t-1} \cup \Out_{t-1})$ (the
least string unused by either side), and at every round $t = 2^n$ the $\prec$-least element of
$(\Out_{t-1} \cap K) \setminus \Seen_{t-1}$, recycling a generator-claimed element so that the
enumeration stays complete (falling back to the fresh rule if no such element exists).
\end{definition}

\begin{restatable}[Race lemma]{lemma}{LemRace}\label{lem:race}
Let $K \subseteq \gnd$ be infinite and let $T_{\mathrm{race}} \subseteq \N$ have density
$1$, with $$w(t) \defeq t - \abs{T_{\mathrm{race}} \cap [1\dots t]} = o(t).$$
\begin{enumerate}[label=\textup{(R\arabic*)}]
  \item \textup{(Greedy cover)} If at every $t \in T_{\mathrm{race}}$ the generator outputs
    the $\prec$-least element of $K \setminus (\Seen_t \cup \Out_{t-1})$, then against every
    adversary enumeration, $\abs{\Out \cap \rest{K}{M}} \ge \tfrac12\,(M - w(\tau_M) - 1)$ for
    all $M$, where $\tau_M$ is the index of the $M$-th race round; hence $\dlow{\Out \cap
    K}{K} \ge 1/2$.
  \item \textup{(Racing cap)} The racing adversary $A_K$ of \cref{def:racingadv} is a legal
    repetition-free enumeration of $K$. Against it, every generator satisfies $\abs{\Out \cap
    \rest{K}{M}} \le M/2 + \log_2 M + 3$ for $M \ge 2$; hence $\dup{\Out \cap K}{K} \le 1/2$.
\end{enumerate}
\end{restatable}

\begin{proof}[Proof sketch]
For (R1), at any race round with $\rest{K}{M}$ not yet resolved the greedy output is a new
element of $\Out \cap \rest{K}{M}$ (the $\prec$-least unresolved element of $K$ lies in the
downward-closed $\rest{K}{M}$); counting contributions against ``steals'' into the seen-set
gives the bound. For (R2), the racing adversary alternates fresh enumeration with recycling
the generator's hoarded elements; an injection from output-elements to fresh-enumerated
elements, with $O(\log M)$ exceptions at powers of two, gives the cap. Full proof in
\cref{app:race}.
\end{proof}

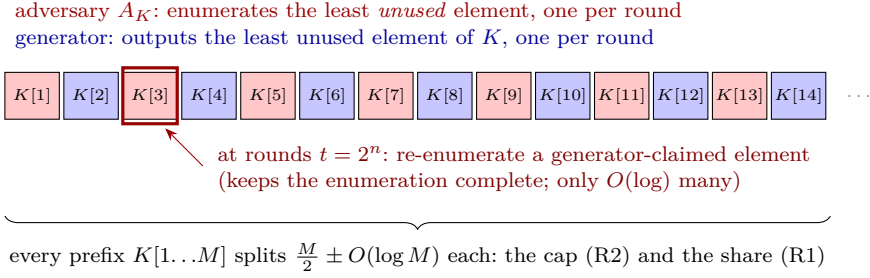
\begin{figure}[t]
\centering
% Figure: the race on the canonical order of K (Lemma race)
\begin{tikzpicture}[>=Stealth, font=\footnotesize]
% the canonical order of K as a row of cells
\foreach \i in {1,...,14} {
  \pgfmathsetmacro\x{0.78*\i}
  \ifodd\i
    \draw[fill=red!22]  (\x,0) rectangle ++(0.72,0.62);
  \else
    \draw[fill=blue!22] (\x,0) rectangle ++(0.72,0.62);
  \fi
  \node[font=\tiny] at (\x+0.36,0.31) {$K[\i]$};
}
% who claims what
\node[red!60!black,  font=\scriptsize, anchor=west] at (0.78,1.42)
  {adversary $A_K$: enumerates the least \emph{unused} element, one per round};
\node[blue!60!black, font=\scriptsize, anchor=west] at (0.78,1.06)
  {generator: outputs the least unused element of $K$, one per round};
% recycle annotation
\draw[red!60!black, very thick] (2.34,-0.06) rectangle ++(0.72,0.74);
\draw[->, red!60!black] (3.4,-0.62) -- (2.9,-0.12);
\node[red!50!black, font=\scriptsize, anchor=west, align=left] at (3.46,-0.66)
  {at rounds $t = 2^n$: re-enumerate a generator-claimed element\\
   (keeps the enumeration complete; only $O(\log)$ many)};
% the half/half brace
\draw[decorate, decoration={brace, mirror, amplitude=5pt}] (0.78,-1.28) -- (11.7,-1.28);
\node[font=\scriptsize, align=center, anchor=north] at (6.24,-1.5)
  {every prefix $\rest{K}{M}$ splits $\tfrac{M}{2} \pm O(\log M)$ each:
   the cap (R2) and the share (R1)};
\node[font=\tiny, gray] at (12.1,0.31) {$\cdots$};
\end{tikzpicture}
\caption{The race on the canonical order of $K$ (\cref{lem:race}). Each round
the adversary enumerates one element and the generator outputs one, so every
prefix splits roughly in half: no generator beats $1/2$ on a raced core (R2),
and the greedy racer secures $1/2$ against every enumeration (R1). The
power-of-two recycling steps keep the adversary's stream a legal full
enumeration at a cost of $O(\log M)$ per prefix.}
\label{fig:race}
\end{figure}

The race lemma drives the impossibility side of the dichotomy. It establishes that on $\Hfam$
no generator can exceed $1 - \alpha/2$ in coverage, and that a generator confined to finitely
many trivia cannot exceed $\alpha/2$.

\begin{restatable}[Impossibility side of the dichotomy]{theorem}{ThmCaps}\label{thm:caps}
On the tight family $\Hfam$ of \eqref{eq:Halpha}, with deterministic element-based generators
and limit-set coverage:
\begin{enumerate}[label=\textup{(\alph*)}]
  \item \textup{(Race cap)} For every deterministic generator there are two legal instances of
    $\Hfam$ sharing one transcript (the adversary is $A_K$ inside $K$), with coverage at most
    $1/2$ at truth $K$ and at most $1 - \alpha/2$ at truth $K'$ (revealed core $\Core = K$,
    density $\alpha$). Hence no generator guarantees coverage exceeding $1 - \alpha/2$ on
    $\Hfam$, nor exceeding $1/2$ at truth $K$.
  \item \textup{(Finite-trivia cap)} Every deterministic generator that emits finitely many
    trivia on every legal instance has coverage at most $\alpha/2$ on the truth-$K'$ racing
    instance.
  \item \textup{(Necessity of trivia)} Any generator guaranteeing coverage exceeding
    $\alpha/2$ on all of $\Hfam$ emits infinitely many trivia on some instance; with
    \cref{thm:harvest}\textup{(iii)}, $\alpha/2$ is exactly the finite-trivia frontier.
\end{enumerate}
\end{restatable}

The caps bind per instance for every generator; no adaptivity to the realized truth is used,
and the per-prefix forms with explicit constants appear in the proofs. The remaining
sparse-core dodge (a singleton-core variant) is left open.

\begin{proof}[Proof sketch]
Run any generator against $A_K$ (\cref{fig:fork}): a single transcript serves both the
truth-$K$ reading (full enumeration of $K$, capped at $1/2$ by \cref{lem:race}(R2)) and the
truth-$K'$ reading (core $\Core = K$ of exact density $\alpha$). For (a) at truth $K'$, split
$\rest{K'}{N} = \rest{K}{m_N} \sqcup (D \cap \rest{K'}{N})$ with $m_N = \lfloor \alpha N
\rfloor$ and cap the $K$-part by (R2). For (b), finite trivia at truth $K$ forces $\abs{\Out
\cap D} < \infty$, so a density-composition lemma carries the $\alpha/2$ cap to truth $K'$.
Full proof in \cref{app:caps}.
\end{proof}

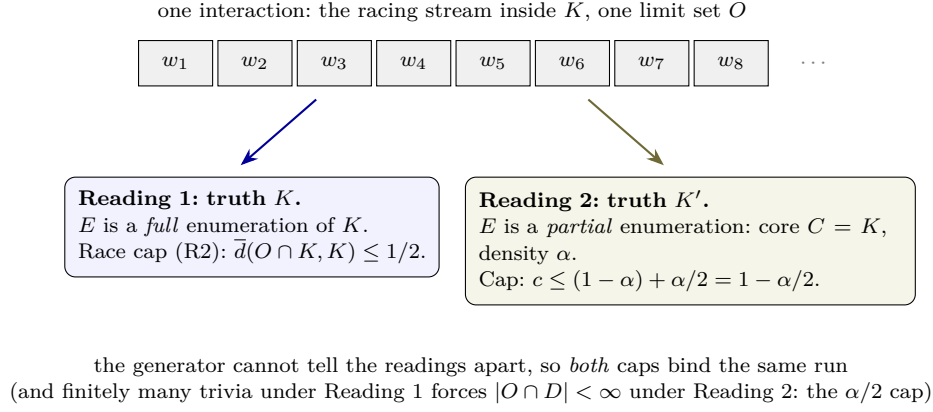
\begin{figure}[t]
\centering
% Figure: one transcript, two readings (the same-transcript caps)
\begin{tikzpicture}[>=Stealth, font=\footnotesize]
% the transcript tape
\foreach \i in {1,...,8} {
  \pgfmathsetmacro\x{1.05*\i}
  \draw[fill=gray!12] (\x,0) rectangle ++(0.98,0.6);
  \node[font=\scriptsize] at (\x+0.49,0.3) {$w_{\i}$};
}
\node[font=\scriptsize, gray] at (10.0,0.3) {$\cdots$};
\node[font=\scriptsize, anchor=south] at (5.2,0.72)
  {one interaction: the racing stream inside $K$, one limit set $\Out$};
% fork arrows
\draw[->, thick, blue!60!black]  (3.4,-0.18) -- (2.4,-1.05);
\draw[->, thick, olive!60!black] (7.0,-0.18) -- (8.0,-1.05);
% reading 1
\node[draw, rounded corners, align=left, fill=blue!5, inner sep=5pt,
      font=\scriptsize, anchor=north, text width=4.6cm] at (2.55,-1.2)
  {\textbf{Reading 1: truth $K$.}\\
   $E$ is a \emph{full} enumeration of $K$.\\
   Race cap (R2): $\dup{\Out \cap K}{K} \le 1/2$.};
% reading 2
\node[draw, rounded corners, align=left, fill=olive!7, inner sep=5pt,
      font=\scriptsize, anchor=north, text width=5.6cm] at (8.35,-1.2)
  {\textbf{Reading 2: truth $K'$.}\\
   $E$ is a \emph{partial} enumeration: core $\Core = K$, density $\alpha$.\\
   Cap: $\cov \le (1-\alpha) + \alpha/2 = 1 - \alpha/2$.};
% the punchline
\node[font=\scriptsize, align=center, anchor=north] at (5.45,-3.45)
  {the generator cannot tell the readings apart, so \emph{both} caps bind the same run\\
   (and finitely many trivia under Reading 1 forces $\abs{\Out \cap D} < \infty$ under
   Reading 2: the $\alpha/2$ cap)};
\end{tikzpicture}
\caption{One transcript, two readings (\cref{thm:caps}). The racing stream
inside $K$ is simultaneously a legal full enumeration of $K$ and a legal
partial enumeration of $K'$ with core $K$. A deterministic generator produces
one limit set, so the race cap binds at truth $K$ and the $1-\alpha/2$ cap at
truth $K'$; finitely many trivia at the first reading collapse the second to
$\alpha/2$.}
\label{fig:fork}
\end{figure}

The matching achievability is a single explicit generator, $G^*$, that races on the
conjunction of the consistent candidates and sweeps on their union. We define it together with
the universal sweep generator $G_{\mathrm{sw}}$, which gives the universal lower bound $\cov
\ge 1 - \dup{\Core}{\Val}$ for every countable collection.

\begin{figure}[t]
\centering
% Figure: the limit set fills a fixed prefix as the horizon grows
% (alpha = 1/3; cells = K'[1..24]; thick-bordered cells are K-elements;
%  fill fractions follow the measured series 0.304 / 0.610 / 0.830 -> 5/6)
\newcommand{\prefixrow}[3]{% y, list of filled cells, label/coverage
  \foreach \i in {1,...,24} {
    \pgfmathsetmacro\x{0.42*\i}
    \pgfmathparse{int(mod(\i,3))}
    \ifnum\pgfmathresult=1
      \draw[line width=0.9pt] (\x,#1) rectangle ++(0.38,0.38);
    \else
      \draw[gray!60] (\x,#1) rectangle ++(0.38,0.38);
    \fi
  }
  \foreach \i in #2 {
    \pgfmathsetmacro\x{0.42*\i}
    \fill[olive!55] (\x+0.05,#1+0.05) rectangle ++(0.28,0.28);
  }
  \node[font=\scriptsize, anchor=west] at (10.7,#1+0.19) {#3};
}
\begin{tikzpicture}[>=Stealth, font=\footnotesize]
% legend
\node[font=\scriptsize, anchor=west] at (0.42,2.95)
  {fixed prefix $\rest{K'}{N}$;\ thick cells $= K$ (raced),\ thin cells $= D$ (swept);\ filled $=$ claimed by $\Out$};
% three snapshots, fill counts ~ 7/24, 15/24, 20/24
\prefixrow{2.0}{{2,3,5,6,8,10,13}}{$T = 2\cdot 10^4$:\ \ $0.304$}
\prefixrow{1.2}{{2,3,5,6,8,9,11,12,14,15,17,18,20,23,4}}{$T = 2\cdot 10^5$:\ \ $0.610$}
\prefixrow{0.4}{{2,3,5,6,8,9,11,12,14,15,17,18,20,21,23,24,4,7,13,19}}{$T = 10^6$:\ \ $0.830$}
\draw[->, gray!60] (5.3,1.94) -- (5.3,0.86);
\node[font=\scriptsize, align=center, anchor=north] at (5.6,-0.0)
  {the horizon $T$ grows; the reading at the \emph{fixed} prefix climbs to
   $1-\alpha/2 = 5/6$\\
   (sweep rounds have density $0$: the trivia \emph{count} is infinite, the rate is not)};
\end{tikzpicture}
\caption{Coverage is a property of the limit set. Snapshots of the same fixed
prefix $\rest{K'}{N}$ ($\alpha = 1/3$) as the horizon grows; the measured
coverage (right, from \cref{app:numerics}) climbs to $1 - \alpha/2$. Sweep
rounds are sparse, so the trivia rate vanishes even though every $D$-cell is
eventually claimed. The snapshots also illustrate why simulations must fix the
prefix and grow the horizon.}
\label{fig:limitset}
\end{figure}
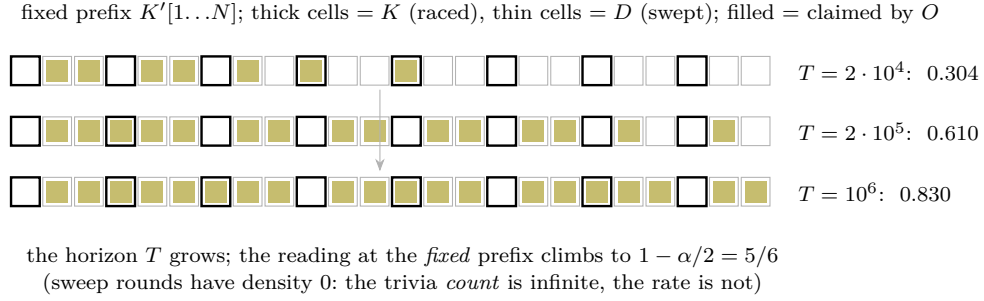

For a countable collection $\Hcal$, write $U_t = \bigcup \{ H' \in \Hcal : \Seen_t \subseteq
H'\}$ for the union of the consistent candidates and let $\mathcal I^{(t)}$ be KW's identified
intersection (\cref{fact:kw16}); the universal sweep generator is
\begin{equation}\label{eq:gharv}
  a_t = \begin{cases}
    \prec\text{-least of } U_t \setminus (\Seen_t \cup \Out_{t-1}),
      & t \in R \defeq \{ k^2 : k \ge 1\},\\[2pt]
    \prec\text{-least of } \mathcal I^{(t)} \setminus (\Seen_t \cup \Out_{t-1}),
      & t \notin R.
  \end{cases}
\end{equation}
(In spirit, $G_{\mathrm{sw}}$ is the gleaner of Millet's painting: it walks the field
in canonical order and gathers whatever the reapers have left behind.) On the tight family,
with $\Hcal_t = \{ H' \in \Hfam : \Seen_t \subseteq H'\}$, $U_t = \bigcup
\Hcal_t$, and $I_t = \bigcap \Hcal_t$, the tight generator is
\begin{equation}\label{eq:gstar}
  a_t = \begin{cases}
    \prec\text{-least of } U_t \setminus (\Seen_t \cup \Out_{t-1}),
      & t \in R \defeq \{k^2 : k \ge 1\}, \\[2pt]
    \prec\text{-least of } I_t \setminus (\Seen_t \cup \Out_{t-1}),
      & t \notin R.
  \end{cases}
\end{equation}

\begin{restatable}[Achievability side: sweep and the tight generator]{theorem}{ThmHarvest}\label{thm:harvest}
With the partial model and limit-set coverage:
\begin{enumerate}[label=\textup{(\roman*)}]
  \item \textup{(Universal sweep)} The generator $G_{\mathrm{sw}}$ of \eqref{eq:gharv} is
    sound (zero hallucination, unconditionally) and on every legal instance achieves $\Out
    \supseteq \Val \setminus \Core$, hence coverage $\cov \ge 1 - \dup{\Core}{\Val}$ ($\ge 1 -
    \alpha$ when $\dens{\Core}{\Val} = \alpha$ exists), with trivia rounds among $[1\dots N]$
    bounded by $T_0 + \lfloor \sqrt N \rfloor$ ($T_0$ finite, instance-dependent): the trivia
    rate tends to $0$ while the count is infinite in general (necessarily, by
    \cref{thm:caps}(c)).
  \item \textup{(Tight generator on $\Hfam$)} On the tight family $\Hfam$ the single
    deterministic generator $G^*$ of \eqref{eq:gstar} achieves, with zero hallucination:
    \begin{enumerate}[label=\textup{(\alph*)}]
  \item truth $K$ (any $\Core \subseteq K$, any order): $\cov \ge \max(1/2,\ 1 -
    \dup{\Core}{K})$, trivia rate $\le \lfloor \sqrt N \rfloor / N$;
  \item truth $K'$ with $\Core \subseteq K$ (covering all of \cref{thm:caps}(a)'s tight
    instances): $\cov \ge 1 - \alpha/2 - o(1)$, with zero trivia;
  \item truth $K'$ with $\Core \not\subseteq K$: $\cov \ge \max(1/2,\ 1 - \dup{\Core}{K'})$,
    zero trivia.
    \end{enumerate}
  \item \textup{(Finite-trivia side; restated)} \Cref{fact:kw16} gives a generator with
    $\dlow{\Out}{\Val} \ge \tfrac12 \dlow{\Core}{\Val}$ for every countable collection and
    eventually zero trivia; on $\Hfam$ it attains $\alpha/2$ at truth $K'$ and $1/2$ at truth
    $K$, matching \cref{thm:caps}'s caps.
\end{enumerate}
\end{restatable}

\begin{proof}[Proof sketch]
A pointer lemma drives (i): any $x \in \Val \setminus \Core$ stays in $U_t$ (as $\Val$ is
consistent) and is never enumerated, so if it were never output the sweep rounds (of which
there are infinitely many) would always emit a new element strictly below $x$, which is
impossible, as only finitely many positions lie below $x$. For (ii), $G^*$ races on the
conjunction $I_t$ (which equals $K$ while the seen-set sits inside $K$) and sweeps on the
union $U_t = K'$, so race covers half of $K$ and the sweep claims all of $D = K' \setminus K$;
at truth $K'$ the two disjoint contributions sum to $1 - \alpha/2$. A key consistency point:
at truth $K$ the sweep eventually claims all of $D$, so $\Out \cap D = D$ is infinite, which
is exactly why \cref{thm:caps}(b)'s finite-trivia cap does not apply to $G^*$. For (iii) we
need nothing about the internal mechanics of KW's generator: since it achieves generation in
the limit ($\exists t^*\,\forall t \ge t^*: a_t \in \Val \setminus \Seen_t$) and $\Val
\subseteq \Amb$, it is eventually sound and eventually zero-trivia, which is exactly what our
soundness and finite-trivia hypotheses require. Full proof in \cref{app:harvest}.
\end{proof}

\begin{remark}[Scope on part~(ii)] The unqualified ``$1 - \alpha/2$ at truth $K'$'' is
false for general $\Core \subseteq K'$: if $\Core = K'$ is enumerated by the racing adversary,
every generator is capped at $1/2$. Case (b)'s restriction $\Core \subseteq K$ is exactly what
the dichotomy needs.
\end{remark}

\begin{remark}[Consistency mechanism]\label{rmk:consistency}
At truth $K$, $G^*$ emits infinitely many elements of $D = \Amb \setminus K$ (the sweep covers
the union $K' \supseteq D$), so its trivia count there is infinite. This is the mechanism that
reconciles \cref{thm:caps}(b) with \cref{thm:harvest}(ii): the finite-trivia cap of $\alpha/2$
governs only generators whose trivia count is finite on \emph{every} instance, and $G^*$ is
not one of them. The infinitude of $\Out \cap D$ at truth $K$ is precisely the resource that
pays for coverage $1 - \alpha/2$ at truth $K'$.
\end{remark}

Combining the two sides gives the dichotomy in closed form.

\begin{restatable}[The trivia dichotomy]{corollary}{CorDichotomy}\label{cor:dichotomy}
On the tight family $\Hfam$ of \eqref{eq:Halpha} with revealed core of exact density $\alpha$,
the optimal guaranteed coverage of valuable mathematics is
\begin{equation}\label{eq:dichotomy}
  \cov^* =
  \begin{cases}
    \alpha/2, & \text{under any finite-trivia allowance, and}\\
    1 - \alpha/2, & \text{under any infinite-trivia allowance,}
  \end{cases}
\end{equation}
both bounds tight, the latter already attained at vanishing trivia rate. The transition is in
trivia count, not rate. The jump $1 - \alpha$ between the two regimes is exactly the
unrecorded mass --- the valuable statements never written into the core.
\end{restatable}

\begin{proof}
The finite-trivia value $\alpha/2$ is upper-bounded by \cref{thm:caps}(b) and attained by
\cref{thm:harvest}(iii). The infinite-trivia value $1 - \alpha/2$ is upper-bounded by
\cref{thm:caps}(a) and attained by $G^*$ at vanishing rate via \cref{thm:harvest}(ii)(b). The
gap is $(1 - \alpha/2) - \alpha/2 = 1 - \alpha = 1 - \dens{\Core}{\Val}$, the lower density of
$\Val \setminus \Core$ in $\Val$.
\end{proof}

\begin{table}[t]
\centering
\caption{The trivia phase transition on $\Hfam$ (revealed core $\Core \subseteq K$ of
exact density $\alpha$, presented as the candidate intersection). The optimum jumps by the
unrecorded mass $1-\alpha$ as the trivia allowance crosses from finite to infinite count,
even though the infinite-count optimum is attained at trivia \emph{rate} $\to 0$.}
\label{tab:dichotomy}
\begin{tabular}{@{}lccl@{}}
\toprule
Trivia allowance & Optimal coverage $\cov^*$ & Attained by & Tight by \\
\midrule
Finite count & $\alpha/2$ & \cref{fact:kw16} (KW) & \cref{thm:caps}(b) \\
Infinite count (rate $\to 0$) & $1 - \alpha/2$ \ {\footnotesize($G^*$, core $\Core \subseteq K$)} & $G^*$, \cref{eq:gstar} & \cref{thm:caps}(a) \\
\midrule
\multicolumn{2}{l}{Jump $= 1 - \alpha = $ unrecorded mass} & & \\
\bottomrule
\end{tabular}
\end{table}

\begin{remark}[Interface of the impossibility proofs]\label{rmk:interface}
The proofs of \cref{lem:race,thm:caps} use \emph{only} the following properties of the triple
$(K, K', \Amb)$: a countable ground set with a canonical order; nesting $K \subsetneq K'
\subseteq \Amb$ with all three infinite; $\Amb \setminus K'$ infinite; and $m_N \defeq \abs{K
\cap \rest{K'}{N}} = \alpha N + O(1)$ with $\alpha > 0$ exact. Any triple meeting this
interface inherits \cref{lem:race,thm:caps} verbatim. We invoke this in \cref{prop:free} to
transport the dichotomy into the compression model.
\end{remark}

\section{Case Study: Instantiation in a Compression Model of Mathematics}
\label{sec:res-instantiation}

The dichotomy was stated abstractly. We close more modestly: we exhibit, inside a concrete
model of mathematics-as-compression, the macro framework of \citet{aksenov2026compression}
(hereafter ABFM), concrete compressible-language families that meet the lower-bound interface
of \cref{rmk:interface}, together with a structured family where no trivia are needed. We
claim no more than this, not a theorem about all compressible languages. In ABFM a
``valuable'' statement is a \emph{compressible} string: one whose length under a macro
dictionary is far below its raw length. We are explicit about the ledger: no ABFM theorem is
an ingredient. ABFM supplies the definitional substrate (macros, wrapped length, geometric
dictionaries) and the motivating asymmetry between the abelian monoid $A_n$ and the free
monoid $F_n$; the proofs below are self-contained.

We work in the abelian monoid $A_n$ and the free monoid $F_n$, each with its canonical order
isomorphic to $\N$ (as \cref{lem:race,thm:caps} demand), and write $\Comp{D} \defeq \{\, w :
\abs{w}_{G \cup D} \le \theta_\kappa(\abs{w}_G)\,\}$ for the sublanguage compressible under a
macro dictionary $D$, where $\abs{w}_{G \cup D}$ is the wrapped length and $\theta_\kappa(L) =
\kappa \log_2(1+L)$ the compressibility cut (\cref{app:background} collects the definitions
and the macro framework).

The abelian regime is benign: compressible languages there form an identifiable chain, so
breadth needs no trivia at all.

\begin{restatable}[Abelian refinement chains are identifiable]{theorem}{ThmAbelian}\label{thm:abelian}
Fix $\kappa \ge 1$ and $b \ge 16\kappa^2$, and set $b_t \defeq b^{2^t}$ and $D_t \defeq \{
b_t^{\,j} a_i : i \le n,\ j \ge 1 \}$ (geometric dictionaries, nested by refinement: $D_{t+1}
\subseteq D_t$). Then $\Hcal_{\mathrm{ab}} \defeq \{ \Comp{D_t} \}_t$ is a strictly descending
chain of infinite languages satisfying Angluin's condition with the singleton tell-tales $T_t
= \{ a_1^{\,b_t + b_t^3} \}$. Consequently, by \cref{thm:taste}(a) (single fiber $\Amb =
A_n$), exact breadth in the limit is achievable on $\Hcal_{\mathrm{ab}}$: in the structured
abelian regime, valuable generation needs no trivia. \textup{(}The partial-enumeration
behavior of this chain is open.\textup{)}
\end{restatable}

\begin{proof}[Proof sketch]
A digit-sum formula computes the wrapped length, $\abs{w}_{G \cup D_t} = \sum_i s_{b_t}(N_i)$
(replace $b_t$ copies of $b_t^{\,j}$ by one $b_t^{\,j+1}$, terminating at the base-$b_t$
representation). Same-count representations transfer down the chain, giving nesting; the
witness $a_1^{\,b_t + b_t^3}$ has digit sum $2$ in base $b_t$ but $2 b_t$ in base $b_t^2$,
separating $\Comp{D_t}$ from $\Comp{D_{t+1}}$ because $2 b_t > 5 \kappa \log_2 b_t$ for $b \ge
16 \kappa^2$. Full proof in \cref{app:abelian}.
\end{proof}

In the free regime we exhibit families meeting the lower-bound interface
(\cref{rmk:interface}). Name-cut pairs realize the tight family at every rational density, and
an unbounded-cut variant realizes the extreme separation point: finite-trivia coverage $0$
against sweep coverage $1$. These are worked instances, not a statement about every
compressible-language pair.

Work in $F_n$ ($n \ge 2$) with letters $\{a, x\}$ and a target density $\alpha = p/q$. The
construction uses the macros $\mu_{m,i} \defeq x^i a\, x^{q+1-i} a\, x^{2^m}$ ($m \ge 1$, $1
\le i \le q$) and, under the bounded cut $\theta \equiv 1$ ($\Comp{D,1} = G \cup D$, the
nameable corpus), the name-cut pair
\begin{equation}\label{eq:namecut}
  K \defeq \{a, x\} \cup \{ \mu_{m,i} : i \le p\}
  \ \subsetneq\
  K' \defeq \{a, x\} \cup \{ \text{all } \mu_{m,i} \}.
\end{equation}

\begin{restatable}[Free-monoid pairs realize the dichotomy]{proposition}{PropFree}\label{prop:free}
In $F_n$ ($n \ge 2$), with letters $\{a, x\}$:
\begin{enumerate}[label=\textup{(\roman*)}]
  \item \textup{(Name-cut pairs at every rational $\alpha$)} The pair \eqref{eq:namecut}
    satisfies $\abs{K \cap \rest{K'}{N}} = \alpha N + O(q)$, so $\dens{K}{K'} = \alpha$
    exactly; it meets the interface of \cref{rmk:interface}, and the dichotomy
    (\cref{cor:dichotomy}) holds verbatim inside the compression model.
  \item \textup{(Obstruction: the bounded cut is forced)} For the unbounded cut
    $\theta_\kappa$ with dictionary-only wrapped length and nested tag-dictionaries
    ($\hat\mu_{m,i} = a a\, x^i a\, x^{q+1-i} a\, x^{2^m}$, tags $\le p$ versus all), the
    analogous pair $(\hat K, \hat K')$ has $\dup{\hat K}{\hat K'} = 0$ (unique parsing via
    $aa$-blocks; the total count of compressible strings is exponential in $L$, the $k$-block
    counting $N_k(L) \le q^k (\log_2 L + 1)^k$ bounds the small-$k$ mass, and the tag
    restriction truncates the rest; the good-word fraction vanishes at rate $\Theta(1/\log L)$
    along prefix scales). Nested tag-pairs under unbounded cuts realize no $\alpha \in (0,1)$.
  \item \textup{(Extreme separation point)} For that pair $(\hat K, \hat K')$, finite-trivia
    sound generators have coverage $0$ at truth $\hat K'$ (same-transcript transfer and
    $\dup{\hat K}{\hat K'} = 0$), while $G_{\mathrm{sw}}$ of \cref{thm:harvest}(i) achieves
    coverage $1$ with vanishing trivia rate. The free regime exhibits the maximal gap of the
    dichotomy: coverage $0$ versus $1$.
\end{enumerate}
\end{restatable}

Finally, the exhaustive generator is quantitatively imprecise: among formal statements within
a ball, only a polynomially small fraction is valuable.

\begin{restatable}[Imprecision of exhaustive generation]{proposition}{PropPrecision}\label{prop:precision}
In $A_n$ with base $c \ge 4$ and $\kappa < n(1 - 1/\log_2 c)$, the compressible sublanguage $H
= \Comp{D_c}$ satisfies
\begin{equation}\label{eq:precision}
  c_1\, r^{\,1/\log_2 c - n} \ \le\ \frac{\abs{H \cap B_G(r)}}{\abs{B_G(r)}}
  \ \le\ c_2\, r^{-\delta}, \qquad \delta = n - \kappa - n/\log_2 c > 0,
\end{equation}
where $B_G(r)$ is the wrapped-length ball of radius $r$. The fraction of compressible
statements is polynomially small in the radius $r$, equivalently exponentially small in ABFM's
compressed scale $s$ (since $r = c^{\Theta(s)}$; interpretation only). The exhaustive
generator's precision thus tends to $0$ polynomially in volume, while the sweep generator's
precision tends to $1$.
\end{restatable}

\begin{proof}[Proof sketch]
Two-sided counting of compressible multisets: the upper bound counts representations with
small digit sums against the volume of $B_G(r)$, the lower bound exhibits a positive density
of compressible points. Full proof in \cref{app:instantiation}.
\end{proof}

\begin{remark}[Numerical sanity checks]\label{rmk:numerics}
The density claims were sanity-checked numerically using the limit-set methodology: fix the
prefix size $N$, grow the run length $T$, and compute coverages in exact rational arithmetic
to avoid floating-point artifacts in the density limits. The race share converges to $1/2$,
and on $\Hfam$ with $\alpha = 1/3$ the truth-$K'$ coverage of $G^*$ converges to $5/6 = 1 -
\alpha/2$ while a pure-race generator stalls at $\alpha/2 \approx 0.17$. See
\cref{app:numerics} for the details.
\end{remark}

\section{Discussion}
\label{sec:discussion}

\paragraph{What the results say for practice.}
The picture for verifier-coupled mathematical generation is sharp and, we think, not the one a
practitioner expects. A proof checker guarantees validity and makes every error benign (under
sound coverage no output is ever false, \cref{thm:flood}), but it supplies no information
about value (\cref{thm:taste}). Covering unrecorded valuable mathematics is therefore not
free: by \cref{cor:dichotomy} any generator reaching beyond coverage $\alpha/2$ must emit
infinitely many certified trivia, and a generator that does so is the \emph{provably optimal
shape} of a broad generator, not a defect to be engineered away. What can be engineered away
is its visible cost: those trivia are needed only at asymptotically negligible rate
(\cref{thm:harvest}). The corollary for system design is that selectivity---taste, the
decision of which valid statements are worth recording---cannot come from the verifier and
must come from examples; ``filter the trivia'' is a post-processing question about value,
orthogonal to the validity the checker already certifies.

\paragraph{The count budget collapses (a finding, not an open problem).}
One might hope for a finer hierarchy between the finite and the infinite trivia regimes,
indexed by a count budget $g(N)$ on the trivia emitted in the first $N$ rounds. There is none.
Our achievability proof shows that the sweep claims every never-enumerated valuable string
under \emph{any} unbounded trivia allowance, so every $g(N)\to\infty$ already buys the full
$1-\alpha/2$ (\cref{thm:harvest}); the finite-trivia cap $\alpha/2$ is exactly the
complementary case (\cref{thm:caps}). The dichotomy is genuinely \emph{count-finite versus
count-infinite}, full stop, and not the coarse shadow of a smoother trade-off.

\paragraph{Open problems.}
\begin{enumerate}
  \item \emph{The singleton sparse-core race.} On the tight family the recorded core is the
    conjunction of the candidates and is raceable at exactly $1/2$. When the core is an
    unstructured sparse set (a singleton collection), a bottom-racing adversary outruns any
    core-blind generator, and the optimal share of the recorded mass is governed by
    combinatorics we have not resolved. We conjecture the optimum is $1-\alpha$ for $\alpha\le
    1/2$; settling it is a cat-and-mouse problem about which the present techniques are
    silent.
  \item \emph{Partial enumeration of structured chains.} \cref{thm:abelian} establishes
    breadth for abelian refinement chains in the breadth model; the behavior of their relative
    densities under partial enumeration, whether the structured regime that needs no trivia
    for breadth also escapes them for density, is open.
  \item \emph{Statistical rates.} We work distribution-free and in the limit. A relativization
    of the rate analysis of \citet{kalavasis2025limits} to the nested pair, quantifying how
    fast coverage approaches its optimum, is left open.
  \item \emph{Computational and efficient versions.} Our generators are arbitrary functions of
    the prefix; whether the dichotomy survives a computability or efficiency constraint, in
    the spirit of the complexity barriers of \citet{arenas2025complexity}, is open in the
    nested model.
  \item \emph{Oracle-free sound coverage.} \cref{thm:flood}(b) exhibits one pair model on
    which no oracle-free generator achieves sound coverage, but we do not characterize which
    pair models admit it. We ask for an Angluin-type condition on $\Pairs$ (a combinatorial
    property of the candidate pairs) deciding whether oracle-free sound coverage
    (equivalently, the sandwich property of \eqref{eq:soundcov}) is achievable; our separation
    gives one failure, a full characterization is open.
\end{enumerate}

\paragraph{Limitation.}
The dichotomy is stated for the tight family $\Hfam$, where the recorded core is the
conjunction of candidates; the universal sweep bound $\cov\ge 1-\dup{\Core}{\Val}$ holds for
every countable collection, but the matching lower bound $1-\alpha/2$ does not, and core
structure, not merely density, decides the optimum, as the singleton race above shows. The
compression instantiation is a worked realization, not a theorem about all compressible
languages: \cref{prop:free} exhibits families meeting the lower-bound interface, not a
generic-language statement. And throughout, value is modeled as an unknown language, not
defined; the contribution is the price of covering it, not a theory of what makes mathematics
worth doing.

\section{Outlook: We Must Generate. We Will Select.}
\label{sec:coda}

A century ago, \citet{poincare1914science} located mathematical discovery not in the
production of true statements, which ``can be done by any one,'' but in the discernment that
selects among them. Our results give that observation a quantitative edge in the machine age.
A perfect verifier secures the border of the formal world: within it nothing is false, and
every unavoidable error has been demoted to mere triviality (\cref{thm:flood}). About which
truths deserve to be kept, however, the verifier is silent (\cref{thm:taste}), and the
dichotomy prices that silence. A generator that would harvest the unrecorded valuable
mathematics must flood; the flood can be thinned to a vanishing fraction of the stream, but
never dammed to a finite count (\cref{thm:harvest}, \cref{cor:dichotomy}). Flood and harvest
are not a defect and its remedy. They are two faces of a single process, and the first cannot
be dispensed with if the second is to be complete. What remains, for machines as it has
always been for mathematicians, is taste: the selection that must follow correctness and that
no checker can supply. If machines are to help expand mathematics, they must first be allowed
to generate more truth than anyone cares to know, and we must learn to find, inside that
flood, the part worth keeping. Hilbert's confidence, broadcast to the world in
\citeyear{hilbert1930naturerkennen}, has lost none of its force. Our theorems only add a
quiet corollary: the knowing will not arrive alone.
\begin{poincarebox}
\begin{quote}
\itshape
``Wir m\"ussen wissen. Wir werden wissen.''

``We must know. We will know.''

\hfill --- David Hilbert, K\"onigsberg address \citeyearpar{hilbert1930naturerkennen}
\end{quote}
\end{poincarebox}

% ---- Bibliography --------
\bibliographystyle{plainnat} 
\bibliography{refs}

% ---- Appendix ------
\newpage
\appendix
\begin{center}
    \textbf{\Huge Appendix}
\end{center}
\section{Notation}
\label{app:notation}

We collect the symbols used throughout. The conventions are those of \cref{sec:prelim}; this
table is a reference, not a redefinition, and every entry points to the section where the
object is introduced.

\renewcommand{\arraystretch}{1.15}
\begin{center}\small
\begin{tabular}{@{}ll@{}}
\toprule
Symbol & Meaning \\
\midrule
$\gnd$ & countable ground set of statements, fixed canonical order $\prec$; $1$-based $\rest{L}{N}$ \\
$\Amb$ & ambient \emph{formal} (verifiable) language; membership oracle is the proof checker \\
$\Val$ & \emph{valuable} target language, $\Val \subseteq \Amb$ \\
$\Pairs$ & pair model $\{(\Amb_i, \Val_i)\}$ (\cref{def:pairmodel}) \\
$\fiber{F}$ & fiber over $F$: candidates $\{\Val_i : \Amb_i = F\}$ sharing one ambient \\
$\Hcal$ & a countable collection of candidate targets \\
$\Hcal_t$ & candidates consistent with $\Seen_t$: $\{H' : \Seen_t \subseteq H'\}$ \\
$\Seen_t$ & seen-set: values enumerated by the adversary through round $t$ \\
$a_t$ & generator's output at round $t$ (element-based, $a_t \notin \Seen_t$) \\
$\Out$ & limit set $\{a_t : t \ge 1\}$ of all outputs over the run \\
$\Core$ & revealed core: infinite $\Core \subseteq \Val$ the adversary enumerates \\
$\alpha$ & lower density $\dlow{\Core}{\Val}$ of the core in the target \\
$\dlow{\cdot}{\cdot},\ \dup{\cdot}{\cdot}$ & lower / upper prefix density (liminf / limsup of $\abs{L' \cap \rest{L}{N}}/N$) \\
$\cov$ & coverage at truth $\Val$: $\cov = \dlow{\Out \cap \Val}{\Val}$ \\
trichotomy & hallucination ($a \notin \Amb$) / triviality ($a \in \Amb \setminus \Val$) / valuable ($a \in \Val$) \\
$\trv_N$ & triviality rate through round $N$ (\cref{def:trichotomy}) \\
$\Hfam$ & tight family $\{K, K'\}$ of \eqref{eq:Halpha} \\
$K \subsetneq K'$ & sparse revealed target $K$ inside coverable world $K'$; $D = K' \setminus K$ \\
$m_N$ & $\abs{K \cap \rest{K'}{N}} = \lfloor \alpha N \rfloor$, the $K$-count in the $K'$-prefix \\
$I_t = \bigcap \Hcal_t$ & conjunction of consistent candidates \\
$U_t = \bigcup \Hcal_t$ & union of consistent candidates \\
$\mc I^{(t)}$ & KW's identified intersection (\cref{fact:kw25}) \\
$G_{\mathrm{sw}}$ & universal sweep generator \eqref{eq:gharv} \\
$G^*$ & tight generator on $\Hfam$ \eqref{eq:gstar} \\
$R = \{k^2\}$ & sweep schedule (density-$0$, infinite) \\
$A_K$ & racing adversary (\cref{def:racingadv}) \\
$E_M$ & fresh-exhaustion round for $\rest{K}{M}$ (\cref{app:race}) \\
$T_{\mathrm{race}},\ w(t)$ & race rounds (density $1$) and the non-race deficit $w(t) = o(t)$ \\
$\theta_\kappa(L) = \kappa \log_2(1+L)$ & compressibility cut (\cref{app:background}) \\
$\Comp{D}$ & $D$-compressible sublanguage $\{w : \abs{w}_{G \cup D} \le \theta_\kappa(\abs{w}_G)\}$ \\
$A_n,\ F_n$ & free abelian monoid on $n$ generators; free monoid on $n$ letters \\
\bottomrule
\end{tabular}
\end{center}
\renewcommand{\arraystretch}{1}

\section{Map of the Results}
\label{app:map}

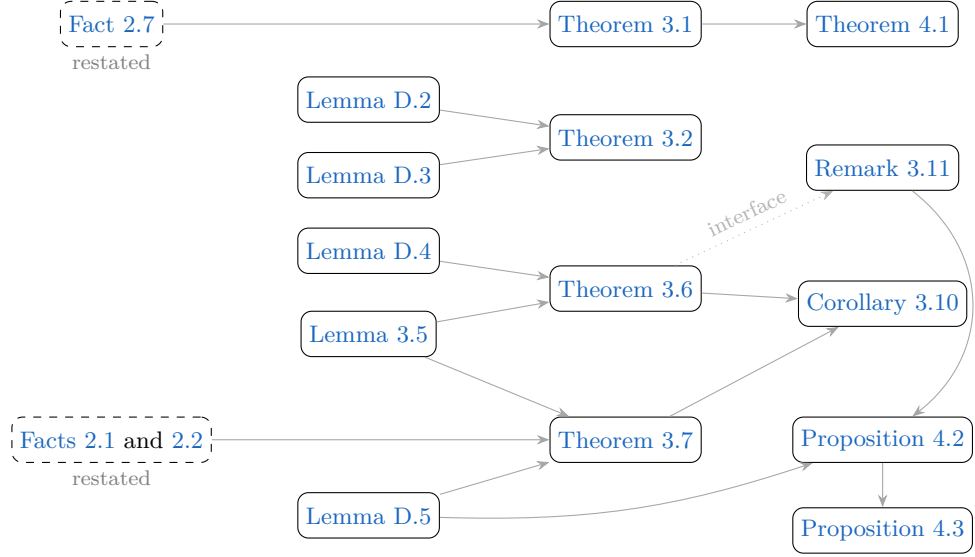
\begin{figure}[!htp]
\centering
\begin{tikzpicture}[
  >=Stealth,
  every node/.style={font=\footnotesize},
  box/.style={draw, rounded corners, inner sep=3pt, minimum height=6mm, align=center},
  fact/.style={box, dashed,
    label={[font=\scriptsize, gray, label distance=-1pt]below:restated}},
  edge/.style={->, gray!70},
  iface/.style={->, dotted, gray!60},
]
% column x: facts 0, lemmas 3.4, theorems 6.8, right 10.2
% Row 0: taste chain
\node[fact] (kmv)   at (0,0)        {\cref{fact:kmv}};
\node[box]  (taste) at (6.8,0)      {\cref{thm:taste}};
\node[box]  (abel)  at (10.2,0)     {\cref{thm:abelian}};
% flood block
\node[box] (sand)  at (3.4,-1.0)    {\cref{lem:sandwich}};
\node[box] (fresh) at (3.4,-2.0)    {\cref{lem:fresh}};
\node[box] (flood) at (6.8,-1.5)    {\cref{thm:flood}};
% caps block
\node[box] (comp)  at (3.4,-3.0)    {\cref{lem:composition}};
\node[box] (race)  at (3.4,-4.1)    {\cref{lem:race}};
\node[box] (caps)  at (6.8,-3.5)    {\cref{thm:caps}};
% harvest block
\node[fact] (kw)   at (0,-5.5)      {\cref{fact:kw16,fact:kw25}};
\node[box]  (ptr)  at (3.4,-6.5)    {\cref{lem:pointer}};
\node[box]  (harvest) at (6.8,-5.5) {\cref{thm:harvest}};
% right column
\node[box] (iface) at (10.2,-1.9)   {\cref{rmk:interface}};
\node[box] (dich)  at (10.2,-3.7)   {\cref{cor:dichotomy}};
\node[box] (free)  at (10.2,-5.5)   {\cref{prop:free}};
\node[box] (prec)  at (10.2,-6.7)   {\cref{prop:precision}};
% edges
\draw[edge] (kmv)   -- (taste);
\draw[edge] (taste) -- (abel);
\draw[edge] (sand)  -- (flood);
\draw[edge] (fresh) -- (flood);
\draw[edge] (comp)  -- (caps);
\draw[edge] (race)  -- (caps);
\draw[edge] (race)  -- (harvest);
\draw[edge] (ptr)   -- (harvest);
\draw[edge] (kw)    -- (harvest);
\draw[edge] (caps)    -- (dich);
\draw[edge] (harvest) -- (dich);
\draw[iface] (caps) -- (iface) node[midway, above, sloped] {\scriptsize interface};
\draw[edge] (iface) to[bend left=52, looseness=1.15] (free);
\draw[edge] (ptr)   to[bend right=10] (free);
\draw[edge] (free)  -- (prec);
\end{tikzpicture}
\caption{Dependency map. Dashed boxes are restated prior facts; solid boxes are
results of this paper. The impossibility theorem (\cref{thm:caps}) and the
achievability theorem (\cref{thm:harvest}) meet at the dichotomy
(\cref{cor:dichotomy}); the case study (\cref{thm:abelian}, \cref{prop:free},
\cref{prop:precision}) reads the dichotomy off the interface of
\cref{rmk:interface}.}
\label{fig:map}
\end{figure}
\Cref{fig:map} traces the dependency structure. The map has four layers. The restated facts
supply the two external inputs: KMV's breadth characterization (\cref{fact:kmv}) and KW's
partial-enumeration guarantees (\cref{fact:kw16,fact:kw25}). The lemmas are the load-bearing
mechanisms: sandwich-exclusion and fresh-index (\cref{lem:sandwich,lem:fresh}) for the
separation, the race lemma and density-composition (\cref{lem:race,lem:composition}) for the
lower bound, and the pointer lemma (\cref{lem:pointer}) for the sweep. The theorems assemble
these: taste (\cref{thm:taste}) from KMV, the separation (\cref{thm:flood}) from the two chain
lemmas, the caps (\cref{thm:caps}) and the sweep theorem (\cref{thm:harvest}) from the race,
composition, and pointer lemmas. The caps and the sweep theorem meet at the dichotomy
(\cref{cor:dichotomy}). Finally the case study instantiates the abstract picture:
\cref{thm:abelian} reads off taste, while \cref{prop:free} reads the dichotomy off the
interface (\cref{rmk:interface}) and reuses the sweep pointer for its extreme separation
point.

\section{Mathematical Background}
\label{app:background}

The case study of \cref{sec:res-instantiation} works inside two monoids and the macro
framework of \citet{aksenov2026compression} (ABFM). We collect the definitions here so the
body can name them in one line. We use this model as a substrate only: no theorem of ABFM is
an ingredient in any proof below.

\subsection{Monoids and canonical orders}
\label{app:bg-monoids}

A \emph{monoid} is a set with an associative binary operation and an identity element. Two
monoids carry the case study.

\emph{The free abelian monoid $A_n$} on generators $a_1, \dots, a_n$ is the set of formal
products $\prod_{i=1}^n a_i^{N_i}$ with $N_i \in \N$, multiplied coordinate-wise; it is
isomorphic to $(\N^n, +)$, with each element identified with its exponent vector $(N_1, \dots,
N_n)$. The \emph{raw length} (the generator count) is $\abs{w}_G = \sum_{i=1}^n N_i$. We order
$A_n$ by degree-then-lex: first by $\abs{w}_G$, then lexicographically on the exponent vector.

\emph{The free monoid $F_n$} on $n$ letters is the set of finite words over the alphabet, with
concatenation as the operation and the empty word as identity; $\abs{w}_G$ is the word length.
We order $F_n$ by length-then-lex.

Both orders are order-isomorphic to $\N$ (\cref{fig:monoids}; each has finitely many elements
below any given one, and the order is total), which is exactly the structure the density
machinery of \cref{sec:prelim-km} requires: a countable ground set with a canonical order in
which $\rest{L}{N}$ is a well-defined initial segment.

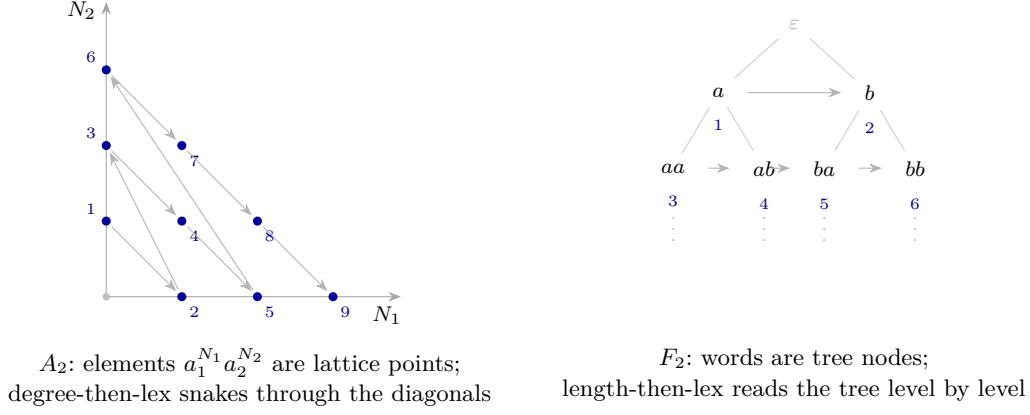
\begin{figure}[t]
\centering
% Figure: the two monoids and their canonical orders (Appendix B.1)
\begin{tikzpicture}[>=Stealth, font=\footnotesize]
% ================= Panel A: A_2 as the lattice N^2 =================
\begin{scope}
\draw[->, gray!70] (0,0) -- (3.9,0) node[below left=0pt and -4pt, black] {\scriptsize $N_1$};
\draw[->, gray!70] (0,0) -- (0,3.9) node[below left=-4pt and 0pt, black] {\scriptsize $N_2$};
% dots with order indices (degree, then ascending lex on (N1,N2))
\foreach \x/\y/\n in {1/0/2, 1/1/4, 2/0/5, 1/2/7, 2/1/8, 3/0/9} {
  \fill[blue!60!black] (\x,\y) circle (1.7pt);
  \node[below right=0pt and -1pt, font=\tiny, blue!50!black] at (\x,\y) {\n};
}
\foreach \x/\y/\n in {0/1/1, 0/2/3, 0/3/6} {
  \fill[blue!60!black] (\x,\y) circle (1.7pt);
  \node[above left=-1pt and 0pt, font=\tiny, blue!50!black] at (\x,\y) {\n};
}
\fill[gray!50] (0,0) circle (1.4pt);
% the order path snaking through the diagonals
\draw[->, gray!60, shorten >=3pt, shorten <=3pt] (0,1) -- (1,0);
\draw[->, gray!60, shorten >=3pt, shorten <=3pt] (1,0) -- (0,2);
\draw[->, gray!60, shorten >=3pt, shorten <=3pt] (0,2) -- (1,1);
\draw[->, gray!60, shorten >=3pt, shorten <=3pt] (1,1) -- (2,0);
\draw[->, gray!60, shorten >=3pt, shorten <=3pt] (2,0) -- (0,3);
\draw[->, gray!60, shorten >=3pt, shorten <=3pt] (0,3) -- (1,2);
\draw[->, gray!60, shorten >=3pt, shorten <=3pt] (1,2) -- (2,1);
\draw[->, gray!60, shorten >=3pt, shorten <=3pt] (2,1) -- (3,0);
\node[align=center, anchor=north] at (1.9,-0.55)
  {$A_2$: elements $a_1^{N_1} a_2^{N_2}$ are lattice points;\\
   degree-then-lex snakes through the diagonals};
\end{scope}
% ================= Panel B: F_2 as the word tree =================
\begin{scope}[xshift=7.2cm]
\node[gray!60, font=\scriptsize] (eps) at (1.9,3.6) {$\varepsilon$};
\node[font=\scriptsize] (a) at (0.9,2.7) {$a$};
\node[font=\scriptsize] (b) at (2.9,2.7) {$b$};
\node[font=\scriptsize] (aa) at (0.3,1.7) {$aa$};
\node[font=\scriptsize] (ab) at (1.5,1.7) {$ab$};
\node[font=\scriptsize] (ba) at (2.3,1.7) {$ba$};
\node[font=\scriptsize] (bb) at (3.5,1.7) {$bb$};
\node[gray!60, font=\tiny] at (0.3,1.0) {$\vdots$};
\node[gray!60, font=\tiny] at (1.5,1.0) {$\vdots$};
\node[gray!60, font=\tiny] at (2.3,1.0) {$\vdots$};
\node[gray!60, font=\tiny] at (3.5,1.0) {$\vdots$};
\draw[gray!40] (eps) -- (a); \draw[gray!40] (eps) -- (b);
\draw[gray!40] (a) -- (aa);  \draw[gray!40] (a) -- (ab);
\draw[gray!40] (b) -- (ba);  \draw[gray!40] (b) -- (bb);
% order indices
\foreach \nd/\n in {a/1, b/2, aa/3, ab/4, ba/5, bb/6}
  \node[below=1pt, font=\tiny, blue!50!black] at (\nd.south) {\n};
% BFS order arrows
\draw[->, gray!60, shorten >=5pt, shorten <=5pt] (a) -- (b);
\draw[->, gray!60, shorten >=5pt, shorten <=5pt] (aa) -- (ab);
\draw[->, gray!60, shorten >=5pt, shorten <=5pt] (ab) -- (ba);
\draw[->, gray!60, shorten >=5pt, shorten <=5pt] (ba) -- (bb);
\node[align=center, anchor=north] at (1.9,-0.55)
  {$F_2$: words are tree nodes;\\
   length-then-lex reads the tree level by level};
\end{scope}
\end{tikzpicture}
\caption{The two monoids and their canonical orders (\cref{app:bg-monoids}). Left:
$A_2 \cong (\N^2,+)$; degree-then-lex enumerates the lattice diagonal by diagonal,
so the order is isomorphic to $\N$. Right: $F_2$; length-then-lex reads the word
tree level by level. In both cases $\rest{L}{N}$ is a well-defined initial segment,
which is all the density machinery of \cref{sec:prelim} needs.}
\label{fig:monoids}
\end{figure}

\subsection{The macro framework}
\label{app:bg-macros}

Following ABFM, a \emph{dictionary} $D$ is a finite or infinite set of designated monoid
elements, called \emph{macros}. Generation is allowed to use the \emph{augmented generating
set} $G \cup D$: a factorization of $w$ is an expression of $w$ as a product of elements of $G
\cup D$. The \emph{wrapped length}
\begin{equation}\label{eq:wraplen}
  \abs{w}_{G \cup D} \defeq \min \set*{\, k : w = g_1 \cdots g_k,\ g_j \in G \cup D \,}
\end{equation}
is the minimum number of symbols from $G \cup D$ in any factorization of $w$, and the
\emph{raw} (unwrapped) length $\abs{w}_G$ uses $G$ alone. A macro that abbreviates a long
element shortens its wrapped length below its raw length: this is the precise sense in which
``compressible'' means ``short under the dictionary.''

For $A_n$ the relevant dictionaries are \emph{geometric}: $D_c = \{\, c^{\,j} a_i : i \le n,\
j \ge 1 \,\}$ provides single-letter powers at every scale $c^j$. For this dictionary the
wrapped length has a closed form,
\begin{equation}\label{eq:bg-digitsum}
  \abs{w}_{G \cup D_c} = \sum_{i=1}^n s_c(N_i),
\end{equation}
where $s_c(N)$ is the sum of the base-$c$ digits of $N$ (\cref{fig:macro}) (the digit-sum
lemma, \cref{lem:digitsum}; the minimizing factorization is the base-$c$ representation of
each coordinate).

\begin{figure}[t]
\centering
\resizebox{0.97\textwidth}{!}{% Figure: wrapped length and the compressibility cut, by example (Appendix B.2)
% A_1 with the geometric dictionary D_3 = {a^3, a^9, a^27, ...}; kappa = 1.
\begin{tikzpicture}[>=Stealth, font=\footnotesize]
% ---- row 1: a^27 ----
\foreach \i in {1,...,27} {
  \pgfmathsetmacro\x{0.32*\i}
  \draw[fill=olive!25] (\x,2.2) rectangle ++(0.28,0.28);
}
\draw[decorate, decoration={brace, amplitude=4pt}] (0.32,2.62) -- (8.92,2.62);
\node[font=\scriptsize, above=6pt] at (4.62,2.62) {one macro: $a^{27} \in D_3$};
\node[font=\scriptsize, anchor=west, align=left] at (9.4,2.34)
  {$\abs{a^{27}}_{G \cup D_3} = s_3(27) = 1 \le \log_2 28$:\\
   \textcolor{olive!50!black}{compressible ($\in \Comp{D_3}$)}};
% ---- row 2: a^25 ----
\foreach \i in {1,...,25} {
  \pgfmathsetmacro\x{0.32*\i}
  \draw[fill=gray!20] (\x,0.7) rectangle ++(0.28,0.28);
}
% braces for 9+9+3+3+1
\draw[decorate, decoration={brace, mirror, amplitude=3pt}] (0.32,0.62) -- (3.16,0.62);
\draw[decorate, decoration={brace, mirror, amplitude=3pt}] (3.24,0.62) -- (6.08,0.62);
\draw[decorate, decoration={brace, mirror, amplitude=3pt}] (6.16,0.62) -- (7.08,0.62);
\draw[decorate, decoration={brace, mirror, amplitude=3pt}] (7.16,0.62) -- (8.08,0.62);
\draw[decorate, decoration={brace, mirror, amplitude=3pt}] (8.16,0.62) -- (8.44,0.62);
\node[font=\tiny, below=4pt] at (1.74,0.62) {$a^9$};
\node[font=\tiny, below=4pt] at (4.66,0.62) {$a^9$};
\node[font=\tiny, below=4pt] at (6.62,0.62) {$a^3$};
\node[font=\tiny, below=4pt] at (7.62,0.62) {$a^3$};
\node[font=\tiny, below=4pt] at (8.30,0.62) {$a$};
\node[font=\scriptsize, anchor=west, align=left] at (9.4,0.84)
  {$\abs{a^{25}}_{G \cup D_3} = s_3(25) = 5 > \log_2 26$:\\
   \textcolor{black!60}{not compressible (just misses the cut)}};
\end{tikzpicture}}
\caption{Wrapped length and the compressibility cut, by example
(\cref{app:bg-macros}; $A_1$, geometric dictionary $D_3 = \{a^3, a^9, a^{27},
\dots\}$, $\kappa = 1$). The wrapped length is the base-$3$ digit sum
\eqref{eq:bg-digitsum}: $a^{27}$ compresses to a single macro and passes the cut,
while $a^{25} = a^9 a^9 a^3 a^3 a$ needs five symbols and just misses it.
Compressibility is thin, which is the precision phenomenon
\cref{prop:precision} quantifies.}
\label{fig:macro}
\end{figure}
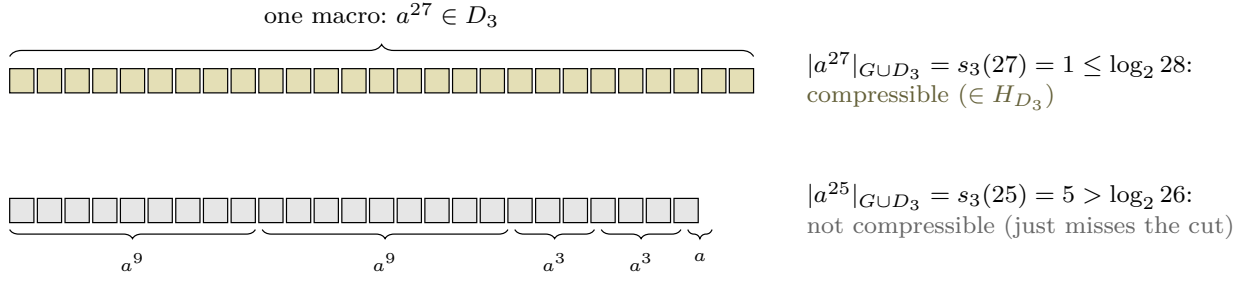

The \emph{compressibility cut} is the threshold
\begin{equation}\label{eq:bg-cut}
  \theta_\kappa(L) \defeq \kappa \log_2(1 + L),
\end{equation}
and the \emph{compressible sublanguage} is
\begin{equation}\label{eq:bg-comp}
  \Comp{D} \defeq \set*{\, w : \abs{w}_{G \cup D} \le \theta_\kappa(\abs{w}_G) \,},
\end{equation}
the elements whose wrapped length is logarithmic in their raw length. A \emph{bounded cut}
$\theta \equiv 1$ (used in \cref{prop:free}(i)) collapses $\Comp{D, 1}$ to the nameable corpus
$G \cup D$ itself.

We are explicit about the ledger: the framework supplies definitions (macros, wrapped length,
geometric dictionaries) and the motivating asymmetry between $A_n$ and $F_n$, and nothing
more. The proofs in \cref{app:instantiation} are self-contained; no ABFM theorem is invoked.

\subsection{The usage of ABFM models in the case study}
\label{app:bg-guide}

\Cref{sec:res-instantiation} uses these objects in two shapes. In $A_n$ it builds
\emph{refinement chains}: nesting the geometric dictionaries by $D_{t+1} \subseteq D_t$ gives
a descending chain of compressible languages $\Comp{D_t}$ separated by digit-sum witnesses
(\cref{thm:abelian}). In $F_n$ it builds \emph{name-cut pairs}: words $x^i a\, x^{q+1-i} a\,
x^{2^m}$ whose unique parsing lets an $a$-position cut the corpus into a sparse target $K$ and
a full world $K'$ at any rational density (\cref{prop:free}).

\section{Deferred Proofs}
\label{app:proofs}

This appendix restates each result with its original number and gives the full proof.

\subsection{Verification is not taste (\texorpdfstring{\cref{thm:taste}}{Theorem 1})}
\label{app:taste}

\begin{statementbox}
\ThmTaste*
\end{statementbox}
\begin{proof}
We prove (a) and (b) together (necessity, then sufficiency), then (c) and (d) and the edge
cases.

\medskip\noindent\textbf{Necessity.}\quad
Suppose some fiber $\fiber{\Amb^*}$ violates the relevant condition. Write $\Hcal \defeq
\fiber{\Amb^*}$; it is countable, nonempty, and all its members are infinite. Restrict the
adversary to instances $(\Amb^*, \Val)$ with $\Val \in \Hcal$. On every such instance the
verifier handed to a relativized generator $G$ is always the single fixed function $\mathbf
1_{\Amb^*}$. Define the oracle-stripped ordinary generator
\begin{equation}\label{eq:strip}
  G'_n(\Seen_n) \defeq G_n^{\,\mathbf 1_{\Amb^*}}(\Seen_n).
\end{equation}
This is well defined: in round $n$ the generator's adaptive query tree has finite depth, and
the fixed function $\mathbf 1_{\Amb^*}$ resolves every query, selecting one leaf and one
output. Thus $G'_n$ is an arbitrary set-valued function of the prefix, exactly the object KMV
quantify over (\cref{fact:kmv}; their Remark~2.3 imposes no computability). Extensional
agreement holds across all of $\Hcal$: for every fiber instance and every enumeration,
\begin{equation}\label{eq:extagree}
  G_n^{\,\mathbf 1_{\Amb_z}}(\Seen_n) = G'_n(\Seen_n)
  \quad\text{for all } n \tag{$*$}
\end{equation}
because $\Amb_z = \Amb^*$ for every $z$ indexing a member of the fiber, so the oracle is the
identical function. Now if $G$ achieved exact (resp.\ approximate) breadth in the limit on all
of $\Pairs$, then by \eqref{eq:extagree} the ordinary generator $G'$ would achieve it on every
target in $\Hcal$, contradicting \cref{fact:kmv}, since $\Hcal$ violates Angluin's (resp.\ the
weak Angluin) condition. Hence the condition is necessary in each part.

\medskip\noindent\textbf{Sufficiency.}\quad
Assume every fiber satisfies the relevant condition.

\emph{Stage~1 (ambient identification; sample-independent).} In round $n$ the generator
queries the canonical point $x_n$ to the verifier. Call index $i$ \emph{ambient-consistent at
$n$} if $\Amb_i \cap \{x_1, \dots, x_n\} = \Amb \cap \{x_1, \dots, x_n\}$, and let $\hat z_n$
be the least ambient-consistent index (the set is nonempty since the true index $z$
qualifies).

\begin{lemma}[Stage-1 stabilization]\label{lem:s1}
$\hat z_n$ stabilizes in finite time to $z_0 \defeq \min\{ i : \Amb_i = \Amb \}$.
\end{lemma}

\begin{proof}
$z_0$ is ambient-consistent at every $n$, so $\hat z_n \le z_0$ always. Each $i < z_0$ has
$\Amb_i \ne \Amb$ (by minimality of $z_0$), hence there is a least canonical point $x_{m(i)}
\in \Amb_i \mathbin\triangle \Amb$, and $i$ is permanently inconsistent from round $m(i)$ on.
Taking $N_1 \defeq \max_{i < z_0} m(i) < \infty$, every $i < z_0$ is eliminated by round
$N_1$, so $\hat z_n = z_0$ for all $n \ge N_1$. (No distinctness of the $\Amb_i$ is needed;
duplicates share the fiber.)
\end{proof}

\emph{Stage~2 (fiber generation).} For each ambient set $\Amb$ that occurs, fix, by the
achievability half of \cref{fact:kmv} applied to the fiber $\fiber{\Amb}$, an ordinary
generator $\Gamma^{(\Amb)}$ achieving exact (resp.\ approximate) breadth in the limit on
$\fiber{\Amb}$. (The choice ranges over the countably many occurring ambient sets; we use the
axiom of choice over this countable index, removable by taking the $\prec$-lexicographically
least generator in a fixed enumeration; see the closing remark.) The constructions' internal
mechanics are absorbed into the fixed function level; no computability is claimed.

Define the composite relativized generator
\begin{equation}\label{eq:composite}
  G_n^{\,\mathbf 1_{\Amb}}(\Seen_n) \defeq \Gamma^{(\Amb_{\hat z_n})}_n(\Seen_n).
\end{equation}
For $n \ge N_1$, \cref{lem:s1} gives $\Amb_{\hat z_n} = \Amb$, so
\begin{equation}\label{eq:dagger}
  G_n^{\,\mathbf 1_{\Amb}}(\Seen_n) = \Gamma^{(\Amb)}_n(\Seen_n)
  \quad\text{for } n \ge N_1. \tag{$\dagger$}
\end{equation}
The true target $\Val = \Val_z$ lies in the true fiber $\fiber{\Amb}$ (witnessed by $z$), and
$\Seen_n$ is the genuine prefix of a genuine KM enumeration of $\Val$: the generator never
restarts the enumeration at $N_1$. By the guarantee of $\Gamma^{(\Amb)}$ there is a finite
round $n^*_\Gamma$ from which the breadth property holds; for $n \ge \max(N_1, n^*_\Gamma)$,
\eqref{eq:dagger} transfers it to $G$. The pre-stabilization rounds form a finite prefix,
absorbed by ``in the limit''. This proves (a) and (b).

\medskip\noindent\textbf{(c).}\quad
In the single-ambient case the verifier is the fixed function $\mathbf 1_{\Amb}$ from
round~$1$, so \eqref{eq:extagree} gives relativized $\iff$ oracle-free, and both $\iff$ (weak)
Angluin by \cref{fact:kmv}. Thus, for the breadth goal, the verifier's only information
content in the pair model is the identification of $\Amb$ performed in Stage~1; within a fiber
it changes nothing. (The scope is essential: \cref{thm:flood} shows that for the
sound-coverage goal the verifier does strictly more.)

\medskip\noindent\textbf{(d).}\quad
For $\Val_i, \Val_j$ in the same fiber, $\mathbf 1_{\Amb_i} = \mathbf 1_{\Amb_j} = \mathbf
1_{\Amb}$ literally. The adaptive query tree, resolved by $\mathbf 1_{\Amb}$ on a fixed prefix
$\Seen_n$, yields the identical path and identical output regardless of which of $\Val_i,
\Val_j$ is the true target. Hence any behavioral distinction within a fiber is driven by the
enumeration alone.

\medskip\noindent\textbf{Edge cases.}\quad
\textsc{(E1) Duplicates.} Angluin's conditions are properties of the underlying set system and
are invariant under repetition of languages in the collection. \textsc{(E2) Finite fibers.} A
finite fiber satisfies Angluin's condition: for each ordered pair $i \ne j$ with $\Val_i
\setminus \Val_j \ne \emptyset$, place a witness $t_{ij} \in \Val_i \setminus \Val_j$ into the
tell-tale $T_i$ (at most $m-1$ elements for a fiber of size $m$); then any $\Val_j \supseteq
T_i$ forces $\Val_i \subseteq \Val_j$, so $\Val_j$ is not a proper subset of $\Val_i$.
\textsc{(E3)} Only occurring ambients are ever guessed, and occurring fibers are nonempty
(they contain the realized target). \textsc{(E4)} The same valuable language appearing under
two different ambients constitutes two different instances in two different fibers; Stage~1
separates them by the ambient query.

\medskip\noindent\textbf{Choice avoidance.}\quad
The use of the axiom of choice in Stage~2 is cosmetic. Fix once and for all a canonical
enumeration of all ordinary generators (e.g.\ a $\prec$-lexicographic order on finite
descriptions); for each occurring ambient $\Amb$, let $\Gamma^{(\Amb)}$ be the
lexicographically least generator achieving the required breadth on $\fiber{\Amb}$. This is a
definable choice, so no appeal to the axiom of choice is needed.
\end{proof}

\subsection{The exhaustive generator and the separation (\texorpdfstring{\cref{thm:flood}}{Theorem 2})}
\label{app:flood}

We first record the two structural lemmas underlying the separation, then give the phase
construction.

\begin{lemma}[Sandwich-exclusion]\label{lem:sandwich}
Work over $\gnd = \N \times \{0,1\}$. Let $L^* = \gnd_0$ (column $0$), let the spine be the
even column-$0$ indices $R = \{(2m, 0) : m \ge 0\}$, and set $\Val_j = L^* \setminus R_{\ge
j}$ with $R_{\ge j} = \{(2m, 0) : m \ge j\}$. Let $\Amb_j = \Val_j \cup A_j$ and $\Amb_\infty
= L^* \cup A_\infty$ with $A_j, A_\infty \subseteq \gnd_1$ (column $1$) infinite. Then for
every $j$, every finite $\Seen \subseteq \gnd$, and every $G \subseteq \gnd$, $G$ cannot
satisfy both
\begin{enumerate}[label=\textup{(\roman*)}]
  \item $L^* \setminus \Seen \subseteq G \subseteq \Amb_\infty$,
  \item $\Val_j \setminus \Seen \subseteq G \subseteq \Amb_j$.
\end{enumerate}
\end{lemma}

\begin{proof}
Suppose both hold. From the lower side of (i) and the upper side of (ii), $L^* \setminus \Seen
\subseteq G \subseteq \Amb_j = \Val_j \cup A_j$. Intersect with $\gnd_0$ (column $0$): since
$L^* \subseteq \gnd_0$ and $A_j \subseteq \gnd_1$ are disjoint columns,
\begin{equation*}
  L^* \setminus \Seen = (L^* \setminus \Seen) \cap \gnd_0
  \subseteq (\Val_j \cup A_j) \cap \gnd_0 = \Val_j .
\end{equation*}
Hence $L^* \setminus \Val_j \subseteq \Seen$, i.e.\ $R_{\ge j} \subseteq \Seen$. But $R_{\ge
j}$ is infinite and $\Seen$ is finite, a contradiction. (The set $A_\infty$ is never used; the
contradiction lives entirely in column $0$.)
\end{proof}

\begin{lemma}[Fresh-index availability]\label{lem:fresh}
With the parameters of \cref{lem:sandwich}: for every finite $T \subseteq L^*$ and every $N
\in \N$ there is $j \ge N$ with $T \subseteq \Val_j \subsetneq L^*$ and $L^* \setminus \Val_j
= R_{\ge j}$ infinite. Consequently $\{\Val_j\}_j \cup \{L^*\}$ is a Gold-style ascending
chain with union $L^*$, non-identifiable from positive data, and $L^*$ violates Angluin's
condition in this collection.
\end{lemma}

\begin{proof}
Since $T$ is finite, $T \cap R$ is finite, so for all sufficiently large $j$ we have $T \cap
R_{\ge j} = \emptyset$, i.e.\ $T \subseteq L^* \setminus R_{\ge j} = \Val_j$; any such $j \ge
N$ works. For every $j$, $R_{\ge j}$ is infinite, so $\Val_j \subsetneq L^*$ with infinite
difference; and $\Val_j$ contains all the odd column-$0$ indices, so it is infinite. The chain
is ascending, $\Val_1 \subseteq \Val_2 \subseteq \cdots$, with $\bigcup_j \Val_j = L^*
\setminus \bigcap_j R_{\ge j} = L^*$. This is the textbook non-identifiable family: every
finite $T \subseteq L^*$ lies in a proper sub-language $\Val_j$ of $L^*$ in the collection,
which is exactly the Angluin violation at $L^*$.
\end{proof}

\begin{statementbox}
\ThmFlood*
\end{statementbox}
\begin{proof}
\medskip\noindent\textbf{(a).}\quad
Since $\Val \subseteq \Amb$ we have $\Val \setminus \Seen_n \subseteq \Amb \setminus \Seen_n
\subseteq \Amb$, so $G_n = \Amb \setminus \Seen_n$ satisfies \eqref{eq:soundcov} at every
round (with $n^* = 0$), on every countable pair model, with zero hallucination. When $\Amb
\setminus \Val$ is infinite, the trivia count $\abs{(\Amb \setminus \Val) \setminus \Seen_n}$
is infinite at every $n$.

\medskip\noindent\textbf{(b).}\quad
Take $\Pairs^{\dagger} = \{(\Amb_j, \Val_j)\}_j \cup \{(\Amb_\infty, L^*)\}$ with the
parameters of \cref{lem:sandwich,lem:fresh}, choosing the $A_j, A_\infty$ pairwise almost
disjoint (this is a modeling convenience and is not load-bearing: \cref{lem:sandwich} is
purely column-$0$). All languages are infinite, $\Val_i \subsetneq \Amb_i$, and $\Amb_i
\setminus \Val_i = A_i$ is infinite.

Fix any oracle-free set-valued generator $G$. We construct an adversary that defeats it. Fix
the canonical repetition-free enumeration $E^*$ of $L^*$, and maintain the invariant that
after phase $\ell$ the shown set is an initial segment $\{e_1, \dots, e_{t_\ell}\}$ of $E^*$.
In phase $\ell$, choose (by \cref{lem:fresh}) an index $j_\ell > \max(j_{\ell-1}, t_{\ell-1})$
with the current shown set contained in $\Val_{j_\ell} \subsetneq L^*$.

\emph{Subphase A.} Continue the enumeration in $E^*$-order, but restricted to $\Val_{j_\ell}$
(append not-yet-shown elements of $\Val_{j_\ell}$ in $E^*$-order). Two cases.
\begin{itemize}
  \item \emph{(A-fail)} If $G$ never sound-covers $(\Amb_{j_\ell}, \Val_{j_\ell})$ at any
    finite round of this continuation, declare $(\Amb_{j_\ell}, \Val_{j_\ell})$ the true pair.
    The produced sequence is a legal KM enumeration of $\Val_{j_\ell}$: its initial segment
    lies in $\Val_{j_\ell}$ by choice of $j_\ell$, and the tail exhausts $\Val_{j_\ell}$.
    Since sound coverage in the limit would in particular hold at some finite round, its
    negation means $G$ fails sound coverage at infinitely many rounds, so $G$ fails in the
    limit. Done.
  \item \emph{(A-succeed)} Let $t'_\ell$ be the first round at which $G_{t'_\ell}$
    sound-covers $(\Amb_{j_\ell}, \Val_{j_\ell})$. By \cref{lem:sandwich}, at this round $G$
    does \emph{not} sound-cover $(\Amb_\infty, L^*)$. \hfill$(\ddagger)$
\end{itemize}

\emph{Subphase B (only in the (A-succeed) case).} Let $P_A$ be the maximum $E^*$-position
among all elements shown through round $t'_\ell$ (a finite number). Append, in $E^*$-order,
the not-yet-shown elements $e_k$ up to some
\begin{equation}\label{eq:tell}
  t_\ell \ \ge\ \max\bigl(P_A,\ t_{\ell-1}+1,\ \ell\bigr),
\end{equation}
restoring the initial-segment invariant: all shown elements lie in $\{e_1, \dots,
e_{t_\ell}\}$, repetition-free. The bound \eqref{eq:tell} clears $P_A$ because subphase A's
filtered enumeration may race ahead in $E^*$-position.

\emph{Conclusion.} If some phase ends in (A-fail), $G$ fails as shown. Otherwise every phase
completes via (A-succeed), so $t_\ell \to \infty$ by \eqref{eq:tell}, and the limit sequence
$E^{\dagger}$ shows every element of $L^*$ exactly once, a legal KM enumeration of $L^*$. At
each of the infinitely many rounds $t'_\ell$, $(\ddagger)$ holds, so $G$ fails sound coverage
for $(\Amb_\infty, L^*)$ at infinitely many rounds, hence fails in the limit. Declare
$(\Amb_\infty, L^*)$ the true pair. Either way, $G$ fails on $\Pairs^{\dagger}$.

\medskip\noindent\textbf{(c).}\quad
Part (a) shows the with-verifier class is all countable pair models; part (b) exhibits a model
outside the without-verifier class. Strictness follows. The breadth contrast is
\cref{thm:taste}(c). For the error-relocation reading: under sound coverage all outputs lie in
$\Amb$ (no hallucination), and when $\Amb \setminus \Val$ is infinite the trivia count
$\abs{(\Amb \setminus \Val) \setminus \Seen_n}$ is infinite (the exhaustive generator); by
\cref{thm:taste}(d) the verifier supplies zero bits about $\Val$ within a fiber, so it cannot
reduce errors or locate value. What it buys is exactly keeping the unavoidable overshoot
inside $\Amb$. The two regimes of part (c) are genuinely distinct: $\Pairs^{\dagger}$ has
singleton fibers, so there the verifier identifies the target and \emph{eliminates} errors;
the relocation reading is the within-fiber regime.
\end{proof}

\subsection{The race lemma (\texorpdfstring{\cref{lem:race}}{Lemma 3})}
\label{app:race}

\begin{statementbox}
\LemRace*
\end{statementbox}
\begin{proof}
Call $x$ \emph{resolved at $t$} if $x \in \Seen_t \cup \Out_t$.

\medskip\noindent\textbf{(R1).}\quad
Let $t_M$ be the first round at which all of $\rest{K}{M}$ is resolved.
\begin{enumerate}[label=\textup{Step \arabic*.}, leftmargin=3.4em]
  \item At any race round $t \in T_{\mathrm{race}}$ with $\rest{K}{M}$ not yet resolved before
    the output, the output is a \emph{new} element of $\Out \cap \rest{K}{M}$. Indeed the rule
    outputs the $\prec$-least element of $K \setminus (\Seen_t \cup \Out_{t-1})$, and the
    $\prec$-least unresolved element of $K$ lies in the downward-closed set $\rest{K}{M}$
    (some element of $\rest{K}{M}$ is unresolved, and the global minimum over $K$ of
    unresolved elements is $\preceq$ it, hence in $\rest{K}{M}$).
  \item Consequently $t_M \le \tau_M$: each of the (at most) $M$ race rounds up to $t_M$ that
    finds $\rest{K}{M}$ unresolved contributes a distinct element of $\rest{K}{M}$, and $M$
    such contributions exhaust $\rest{K}{M}$.
  \item At most one race round $\le t_M$ is non-contributing: by minimality of $t_M$,
    $\rest{K}{M}$ is unresolved at every race round strictly before $t_M$, so each such round
    contributes; only the round $t_M$ itself may be non-contributing (it may resolve the last
    element via the seen-set). Hence $\abs{\Out \cap \rest{K}{M}} \ge t_M - w(t_M) - 1$, where
    $w(t_M)$ counts the non-race rounds up to $t_M$.
  \item \emph{Steal bound.} Each $x \in \rest{K}{M} \setminus \Out$ is resolved only by
    entering $\Seen_{t_M}$; distinct elements have distinct first-enumeration rounds $\le
    t_M$, so $\abs{\rest{K}{M} \setminus \Out} \le t_M$.
\end{enumerate}
Adding Step 3 and Step 4, $2\abs{\Out \cap \rest{K}{M}} \ge M - w(t_M) - 1 \ge M - w(\tau_M) -
1$ (using $t_M \le \tau_M$ and monotonicity of $w$). Since $\tau_M = M + w(\tau_M) \sim M$ and
$w = o(t)$, dividing by $M$ and taking the liminf gives $\dlow{\Out \cap K}{K} \ge 1/2$.

\medskip\noindent\textbf{(R2).}\quad
\emph{Legality of $A_K$.} At $t \ne 2^n$ the adversary enumerates the $\prec$-least element of
$K \setminus (\Seen_{t-1} \cup \Out_{t-1})$ (a fresh round); at $t = 2^n$ it enumerates the
$\prec$-least element of $(\Out_{t-1} \cap K) \setminus \Seen_{t-1}$ (a recycle round,
draining the generator's hoarded backlog from the bottom; if empty, it uses the fresh rule as
a fallback). We claim every $x \in K$ is enumerated at some finite round, so $A_K$ is a legal
repetition-free enumeration of $K$. The $\prec$-least un-enumerated element of $K$ advances
monotonically: fresh rounds push it forward, and recycle rounds drain the $\Out$-backlog from
the bottom, so no element below the current pointer can be hoarded forever. Hence every fixed
$x \in K$ is enumerated at a finite round. No rate is claimed or needed: legality is eventual
only.

\emph{The cap.} Fix $M \ge 2$. Let $F_M$ be the set of fresh-enumerated elements in
$\rest{K}{M}$ and $G_M = \Out \cap \rest{K}{M}$.
\begin{enumerate}[label=\textup{(C\arabic*)}, leftmargin=3.4em]
  \item $F_M \cap G_M = \emptyset$ (an element is either enumerated by the adversary or output
    by the generator first, in the sense that an output element $a_t \notin \Seen_t$ is not
    among the fresh-enumerated by definition of the rule), so $\abs{F_M} + \abs{G_M} \le M$.
  \item \emph{Injection $G_M \to F_M$ up to exceptions.} Map the entry round $t$ of an
    $\Out$-element to round $t+1$. If $t+1$ is a fresh round and $\rest{K}{M}$ is not yet
    fresh-exhausted, the fresh rule at $t+1$ enumerates an element of $F_M$. There are exactly
    two kinds of exception, which we count separately:
    \begin{enumerate}[label=\textup{(\roman*)}, leftmargin=2.4em]
  \item \emph{power-of-two recycle rounds:} rounds $t$ with $t+1 = 2^n$. Up to the
    fresh-exhaustion round $E_M$ (defined next) there are at most $\lfloor \log_2 E_M \rfloor
    + 1$ of these;
  \item \emph{the single post-exhaustion threshold:} after $E_M$ the fresh rule no longer
    enumerates a new element of $\rest{K}{M}$, costing $O(1)$ further exceptions.
    \end{enumerate}
Hence the total number of exceptions is at most $\lfloor \log_2 E_M \rfloor + 1 + O(1) \le
\log_2 E_M + 2$.
  \item \emph{The fresh-exhaustion round $E_M$.} Define $E_M$ as the round by which every
    element of $\rest{K}{M}$ is used (lies in $\Seen \cup \Out$), equivalently the round by
    which the fresh pointer has passed $K[M]$, skipping any $\Out$-hoarded elements. We bound
    $E_M$ non-circularly. Each fresh round outputs the $\prec$-least un-enumerated element of
    $K$, so fresh claims strictly $\prec$-increase the fresh pointer; within $M$ fresh rounds
    the pointer therefore passes $K[M]$. Writing $T \defeq E_M$, the number of non-fresh
    (recycle) rounds up to $T$ is at most $\lfloor \log_2 T \rfloor + 1$ (one per power of two
    $\le T$), so the number of fresh rounds up to $T$ is at least $T - \log_2 T - 1$.
    Requiring this to reach $M$ gives the implicit bound $T \le M + \log_2 T + 1$; solving it
    yields $E_M \le M + \log_2(M+1) + 2$.
\end{enumerate}
Combining (C1)--(C3): $\abs{\Out \cap \rest{K}{M}} = \abs{G_M} \le \abs{F_M} + (\log_2 E_M +
2) \le M/2 + \log_2 M + 3$, using $E_M \le M + \log_2(M+1) + 2$ so that $\log_2 E_M \le \log_2
M + O(1)$. Dividing by $M$ and taking the limsup gives $\dup{\Out \cap K}{K} \le 1/2$.

\emph{Caution on $E_M$.} The round $E_M$ must \emph{not} be read as the round of full
enumeration of $\rest{K}{M}$: against a hoarding generator, full enumeration of $\rest{K}{M}$
can take exponentially many rounds (the recycle trickle drains the backlog one element per
power-of-two round). The cap argument needs only fresh-exhaustion, which advances at rate $1$
per fresh round.
\end{proof}

\subsection{Impossibility side of the dichotomy (\texorpdfstring{\cref{thm:caps}}{Theorem 4})}
\label{app:caps}

\begin{statementbox}
\ThmCaps*
\end{statementbox}
We first record the composition lemma that carries densities between $K$ and $K'$.

\begin{lemma}[Density composition]\label{lem:composition}
Let $K \subseteq K'$ be infinite with $\dens{K}{K'} = \alpha > 0$ exact. For any $A \subseteq
K$,
\begin{equation*}
  \dup{A}{K'} = \alpha \cdot \dup{A}{K},
  \qquad
  \dlow{A}{K'} = \alpha \cdot \dlow{A}{K}.
\end{equation*}
\end{lemma}

\begin{proof}
Listings restrict to the canonical order, so $A \cap \rest{K'}{N} = A \cap \rest{K}{m_N}$ with
$m_N \defeq \abs{K \cap \rest{K'}{N}}$. Write $X_m \defeq \abs{A \cap \rest{K}{m}}/m$. Then
\begin{equation*}
  \frac{\abs{A \cap \rest{K'}{N}}}{N}
  = \frac{\abs{A \cap \rest{K}{m_N}}}{m_N} \cdot \frac{m_N}{N}
  = X_{m_N} \cdot \frac{m_N}{N}.
\end{equation*}
The sequence $(m_N)$ is nondecreasing with increments in $\{0,1\}$ and $m_N \to \infty$, so it
hits every large integer; hence $\limsup_N X_{m_N} = \limsup_m X_m$ and likewise for
$\liminf$. Since $m_N/N \to \alpha > 0$, the limits factor as claimed.
\end{proof}

\begin{proof}[Proof of \cref{thm:caps}]
\medskip\noindent\textbf{(a).}\quad
Run any deterministic generator $G$ against the racing adversary $A_K$ of \cref{lem:race}(R2).
The transcript is a single object $\Out$, and any verifier is fixed advice ($\Amb = \N$ under
both readings), so the two readings $K$ and $K'$ share it.

\emph{Legality of both readings.} As truth $K$, $A_K$ is a full repetition-free enumeration of
$K$ (\cref{lem:race}(R2)). As truth $K'$, the same stream is a legal partial enumeration of
$K'$ with revealed core $\Core = K$ of exact density $\dens{K}{K'} = \alpha$
(\eqref{eq:exactdensity}).

\emph{Truth $K$.} By \cref{lem:race}(R2), $\cov \le \dup{\Out \cap K}{K} \le 1/2$, i.e.\
coverage $\le 1/2 + O(\log M)/M$.

\emph{Truth $K'$.} Split $\rest{K'}{N} = \rest{K}{m_N} \sqcup (D \cap \rest{K'}{N})$ with $m_N
= \lfloor \alpha N \rfloor$ (by \eqref{eq:exactdensity}). Then
\begin{align*}
  \abs{\Out \cap \rest{K'}{N}}
  &\le \abs{\Out \cap \rest{K}{m_N}} + \abs{D \cap \rest{K'}{N}}\\
  &\le \bigl(\tfrac{m_N}{2} + \log_2 m_N + 3\bigr) + (N - m_N)
  = N - \tfrac{m_N}{2} + O(\log N),
\end{align*}
using \cref{lem:race}(R2) on the $K$-part. Dividing by $N$ and taking the liminf gives $\cov
\le 1 - \alpha/2$, i.e.\ coverage $\le 1 - \alpha/2 + O(\log N)/N$.

\emph{Per-instance.} The caps bind for every generator on the single instance $A_K$; no
adaptivity to the realized truth is used. Hence no generator guarantees coverage exceeding $1
- \alpha/2$ on $\Hfam$, nor exceeding $1/2$ at truth $K$.

\medskip\noindent\textbf{(b).}\quad
Run $G$ against the same $A_K$, so the transcript is shared. Suppose $G$ emits finitely many
trivia on every legal instance; in particular at the legal truth-$K$ instance, $\abs{\Out \cap
(\Amb \setminus K)} < \infty$, so $C_0 \defeq \abs{\Out \cap D} < \infty$. At truth $K'$,
\begin{equation*}
  \abs{\Out \cap \rest{K'}{N}}
  \le \abs{\Out \cap K \cap \rest{K'}{N}} + C_0 .
\end{equation*}
Applying \cref{lem:composition} with $A = \Out \cap K$ and then \cref{lem:race}(R2),
\begin{equation*}
  \cov = \dlow{\Out \cap K'}{K'}
  \le \dup{\Out \cap K}{K'}
  = \alpha \cdot \dup{\Out \cap K}{K}
  \le \alpha \cdot \tfrac12 = \tfrac{\alpha}{2}.
\end{equation*}

\medskip\noindent\textbf{(c).}\quad
The contrapositive of (b): any generator guaranteeing coverage exceeding $\alpha/2$ on all of
$\Hfam$ must emit infinitely many trivia on some instance. Exactness of the finite-trivia
frontier at $\alpha/2$ follows from \cref{thm:harvest}(iii), which attains $\alpha/2$ at truth
$K'$ with eventually zero trivia.
\end{proof}

\subsection{Achievability: sweep and the tight generator (\texorpdfstring{\cref{thm:harvest}}{Theorem 5})}
\label{app:harvest}

\begin{statementbox}
\ThmHarvest*
\end{statementbox}
\begin{proof}
\medskip\noindent\textbf{Pointer lemma.}\quad
We first prove the fact that drives both generators.

\begin{lemma}[Pointer lemma]\label{lem:pointer}
Suppose at each sweep round $t \in R$ the generator outputs the $\prec$-least element of $V_t
\setminus (\Seen_t \cup \Out_{t-1})$ for a family $V_t$ with $\Val \subseteq V_t$ for all $t$,
and $R$ is infinite. Then for every $x \in \Val$ never enumerated by the adversary, $x \in
\Out$.
\end{lemma}

\begin{proof}
Fix such an $x$ at $\Val$-position $p$. Since $\Val \subseteq V_t$ always, $x \in V_t$; since
$x$ is never enumerated, $x \notin \Seen_t$ for all $t$. Suppose $x \notin \Out$. At every
sweep round $t \in R$, the output is the $\prec$-least element of $V_t \setminus (\Seen_t \cup
\Out_{t-1})$, which is $\preceq x$ (as $x$ itself is available) and $\ne x$ (as $x \notin
\Out$), hence strictly $\prec x$ --- a new element of $\Val$ strictly below position $p$. But
only $p-1$ positions lie below $x$, and $R$ is infinite, a contradiction. Hence $x \in \Out$.
\end{proof}

\paragraph{(i) Universal sweep $G_{\mathrm{sw}}$.}
Recall $U_t = \bigcup\{H' \in \Hcal : \Seen_t \subseteq H'\}$ and the sweep schedule $R =
\{k^2\}$. Apply \cref{lem:pointer} with $V_t = U_t$: for any $x \in \Val \setminus
\Core$, $\Val$ is consistent ($\Seen_t \subseteq \Core \subseteq \Val$) so $\Val \in
\Hcal_t$ and $x \in U_t$, while $x \notin \Core$ is never enumerated. Hence $\Out \supseteq
\Val \setminus \Core$. Therefore
\begin{equation*}
  \abs{\Out \cap \rest{\Val}{N}} \ge N - \abs{\Core \cap \rest{\Val}{N}},
  \qquad\text{so}\qquad
  \cov \ge 1 - \dup{\Core}{\Val},
\end{equation*}
which is $\ge 1 - \alpha$ when $\dens{\Core}{\Val} = \alpha$ exists. \emph{Trivia rate.} By
\cref{fact:kw25} there is a finite $T_0$ with $\mathcal I^{(t)} \subseteq \Val$ for all $t \ge
T_0$, so every non-sweep output from $T_0$ on lies in $\Val$; the only possible trivia rounds
among $[1\dots N]$ are the $\le T_0$ early non-sweep rounds and the $\lfloor \sqrt N \rfloor$
sweep rounds, giving a trivia count $\le T_0 + \lfloor \sqrt N \rfloor$ and rate $\to 0$. The
count is infinite in general, necessarily so, by \cref{thm:caps}(c). \emph{Soundness.} $U_t
\subseteq \bigcup \Hcal \subseteq \Amb$; $\mathcal I^{(t)}$ is contained in some candidate
$\subseteq \Amb$; the fallback lies in $\Amb$. Hence $\Out \subseteq \Amb$: zero
hallucination.

(The restated KW lemma is used only for the non-sweep half of the trivia bound. A bare
conjunction default would fail on a general $\Hcal$: for $\Hcal = \{\N \setminus \{n\}\}$ the
conjunction of consistent candidates is just $\Seen_t$. On $\Hfam$, $G^*$ below avoids this
restated lemma entirely.)

\medskip\noindent\textbf{(ii) Tight generator $G^*$ on $\Hfam$.}\quad
Recall $\Hcal_t = \{H' \in \Hfam : \Seen_t \subseteq H'\}$, $I_t = \bigcap \Hcal_t$, $U_t =
\bigcup \Hcal_t$. The structure is:
\begin{equation*}
  \Seen_t \subseteq K \ \Longrightarrow\ \Hcal_t = \{K, K'\},\ I_t = K,\ U_t = K';
  \qquad
  \Seen_t \not\subseteq K \ \Longrightarrow\ \Hcal_t = \{K'\},\ I_t = U_t = K'.
\end{equation*}

\emph{(a) Truth $K$.} Here every enumerated element is in $\Core \subseteq K$, so $\Seen_t
\subseteq K$ always and $I_t = K$. The non-sweep rounds form the set $\N \setminus R$, of
density $1$ with $w(t) = \lfloor \sqrt t \rfloor = o(t)$; on them $G^*$ runs
\cref{lem:race}(R1)'s greedy race on $K$ (the $\prec$-least element of $I_t \setminus (\Seen_t
\cup \Out_{t-1}) = K \setminus (\Seen_t \cup \Out_{t-1})$). Thus $\dlow{\Out \cap K}{K} \ge
1/2$, i.e.\ $\cov \ge 1/2$. Applying \cref{lem:pointer} with $V_t = U_t = K' \supseteq K$
gives $\Out \supseteq K \setminus \Core$, so $\cov \ge 1 - \dup{\Core}{K}$; together $\cov \ge
\max(1/2, 1 - \dup{\Core}{K})$. \emph{Trivia.} Race outputs lie in $K$; the sweep covers $U_t
= K' \supseteq D$, so trivia rounds among $[1\dots N]$ number $\le \lfloor \sqrt N \rfloor$,
rate $\to 0$. Note $\Out \cap D = D$ is \emph{infinite}: the sweep eventually claims all of
$D$ via \cref{lem:pointer} (each $d \in D$ lies in $U_t = K'$ and is never enumerated at truth
$K$). This infinitude is exactly why \cref{thm:caps}(b)'s finite-trivia cap does not apply to
$G^*$.

\emph{(b) Truth $K'$ with $\Core \subseteq K$.} The transcript is identical to case (a):
$\Core \subseteq K$ means $\Seen_t \subseteq K$ throughout, so $G^*$ cannot and need not
distinguish the two truths, and runs (R1)'s race on $K$ exactly as before. By
\cref{lem:pointer} with $\Val = K'$ (every $x \in K' \setminus \Core$ is never enumerated and
lies in $U_t = K'$), $\Out \supseteq K' \setminus \Core \supseteq D$. For coverage, split
$\rest{K'}{N} = \rest{K}{m_N} \sqcup (D \cap \rest{K'}{N})$ ($m_N = \lfloor \alpha N
\rfloor$); the sweep claims all $N - m_N$ elements of $D \cap \rest{K'}{N}$ and the race
claims at least $m_N/2 - O(\sqrt N)$ of $\rest{K}{m_N}$, and the two parts are disjoint, so
\begin{equation*}
  \abs{\Out \cap \rest{K'}{N}} \ge (N - m_N) + \tfrac{m_N}{2} - O(\sqrt N)
  = N\bigl(1 - \tfrac{\alpha}{2}\bigr) - O(\sqrt N),
\end{equation*}
giving $\cov \ge 1 - \alpha/2 - o(1)$. \emph{Trivia.} Race outputs $\subseteq K \subseteq K'$
and sweep outputs $\subseteq K'$, so \emph{zero} trivia at truth $K'$.

\emph{(c) Truth $K'$ with $\Core \not\subseteq K$.} Let $t_0$ be the first (finite) round with
$w_{t_0} \notin K$. From $t_0$ on, $\Seen_t \not\subseteq K$, so $I_t = U_t = K'$. Apply
\cref{lem:race}(R1) with $K \leftarrow K'$ and $T_{\mathrm{race}} = (\N \setminus R) \setminus
[1\dots t_0]$ (still density $1$): $\cov \ge 1/2$. \Cref{lem:pointer} with $V_t = U_t = K'$
gives $\cov \ge 1 - \dup{\Core}{K'}$; together $\cov \ge \max(1/2, 1 - \dup{\Core}{K'})$. All
outputs lie in $K'$, so zero trivia.

\emph{Scope.} The unqualified ``$1 - \alpha/2$ at truth $K'$'' is false for general $\Core
\subseteq K'$: if $\Core = K'$ enumerated by the racing adversary, the cap of
\cref{lem:race}(R2) holds with $K \leftarrow K'$ and pins every generator at $1/2$. Case (b)'s
restriction $\Core \subseteq K$ is exactly the regime the dichotomy needs.

\medskip\noindent\textbf{(iii) Finite-trivia side (restated).}\quad
By \cref{fact:kw16}, KW's generator achieves $\dlow{\Out}{\Val} \ge \tfrac12
\dlow{\Core}{\Val}$ for every countable collection, never consulting $\Amb$; generation in the
limit gives $\exists t^*\, \forall t \ge t^*: a_t \in \Val \setminus \Seen_t$, so the trivia
rounds are confined to $t < t^*$: eventually zero. Because $\Val \subseteq \Amb$, the same
statement $a_t \in \Val$ for $t \ge t^*$ makes the generator \emph{eventually sound} ($a_t \in
\Amb$) as well; this is exactly what our model's soundness (eventually $a_t \in \Amb$) and
finite-trivia (eventually $a_t \notin \Amb \setminus \Val$) hypotheses require, and it follows
from the generation guarantee alone. No claim about the pre-$t^*$ outputs (and in particular
no inference from the internal pod construction of KW's generator) is needed for the
load-bearing path. On $\Hfam$ this yields $\alpha/2$ at truth $K'$ and $1/2$ at truth $K$,
matching the caps of \cref{thm:caps}.

\begin{remark}[Non-load-bearing aside]\label{rmk:pods}
KW's pods are in fact candidate-contained at all times, so one expects the pre-$t^*$ outputs
also to lie in $\Amb$ here; we record this only as an aside, as the eventual-soundness
argument above already supplies everything the dichotomy uses.
\end{remark}
\end{proof}

\subsection{The instantiation (\texorpdfstring{\cref{thm:abelian,prop:free,prop:precision}}{Theorems 7--9})}
\label{app:instantiation}

\subsubsection{Abelian refinement chains (\texorpdfstring{\cref{thm:abelian}}{Theorem 7})}
\label{app:abelian}

\begin{statementbox}
\ThmAbelian*
\end{statementbox}
\begin{proof}
Write $w = \prod_i a_i^{N_i} \in A_n \cong (\N^n, +)$, identified with its exponent vector
$(N_1, \dots, N_n)$. We use three lemmas.

\begin{lemma}[Digit-sum formula]\label{lem:digitsum}
For the geometric dictionary $D_c = \{ c^{\,j} a_i : i \le n, j \ge 1 \}$ (with $c = b_t$),
the wrapped length is $\abs{w}_{G \cup D_c} = \sum_{i=1}^n s_c(N_i)$, where $s_c(N)$ is the
sum of the base-$c$ digits of $N$.
\end{lemma}

\begin{proof}
Each coordinate $a_i^{N_i}$ is factored independently (the dictionary words are single-letter
powers). For one coordinate: a factoring corresponds to a representation of $N_i$ as a sum of
powers of $c$ with nonnegative integer coefficients (multiplicities), and its symbol count is
the sum of those multiplicities. Replacing $c$ copies of the macro $c^{\,j} a_i$ by one copy
of $c^{\,j+1} a_i$ preserves the represented power and strictly decreases the symbol count, so
every minimizing representation has all multiplicities $< c$; the unique representation with
all multiplicities $< c$ is the base-$c$ representation of $N_i$. Hence the minimum symbol
count is its digit sum $s_c(N_i)$. Summing over coordinates gives the formula.
\end{proof}

\begin{lemma}[Nesting]\label{lem:nesting}
$\Comp{D_{t+1}} \subseteq \Comp{D_t}$, so $\{\Comp{D_t}\}_t$ is a descending chain.
\end{lemma}

\begin{proof}
Since $b_{t+1} = b_t^2$, we have $D_{t+1} \subseteq D_t$. A richer dictionary cannot increase
the minimal factoring length, so $\abs{w}_{G \cup D_t} \le \abs{w}_{G \cup D_{t+1}}$ for every
$w$. Therefore every $w$ with $\abs{w}_{G \cup D_{t+1}} \le \theta_\kappa(\abs{w}_G)$ also
satisfies $\abs{w}_{G \cup D_t} \le \theta_\kappa(\abs{w}_G)$: membership transfers downward
and the chain is descending.
\end{proof}

\begin{lemma}[Witness / singleton tell-tale]\label{lem:witness}
$w_t \defeq a_1^{\,b_t + b_t^3} \in \Comp{D_t} \setminus \Comp{D_{t+1}}$, and $T_t = \{w_t\}$
is a tell-tale for $\Comp{D_t}$ in the chain.
\end{lemma}

\begin{proof}
By \cref{lem:digitsum}, $\abs{w_t}_{G \cup D_t} = s_{b_t}(b_t + b_t^3) = 2$ (the number $b_t +
b_t^3$ has base-$b_t$ digits $1$ at positions $1$ and $3$). In base $b_{t+1} = b_t^2$, the
same number is $b_t + b_t \cdot b_t^2 = b_t(1 + b_t^2)$, whose base-$b_t^2$ digit sum is $2
b_t$ (a digit $b_t$ at positions $0$ and $1$). The raw length is $\abs{w_t}_G = b_t + b_t^3$,
so $\theta_\kappa(\abs{w_t}_G) = \kappa \log_2(1 + b_t + b_t^3) \le 5\kappa \log_2 b_t$ for
large $b_t$. Then $\abs{w_t}_{G \cup D_t} = 2 \le \theta_\kappa$, so $w_t \in \Comp{D_t}$;
while $\abs{w_t}_{G \cup D_{t+1}} = 2 b_t > 5\kappa \log_2 b_t \ge \theta_\kappa$ since $b \ge
16\kappa^2$ forces $2 b_t > 5\kappa \log_2 b_t$, so $w_t \notin \Comp{D_{t+1}}$. As the chain
is descending, $\{w_t\}$ separates $\Comp{D_t}$ from every proper sub-language
$\Comp{D_{t'}}$, $t' > t$: any $\Comp{D_{t'}} \ni w_t$ would need $t' \le t$, so no proper
sub-language contains the tell-tale. The languages are infinite (e.g.\ all $a_1^{c^j}$ have
wrapped length $1$).
\end{proof}

The three lemmas give a strictly descending chain of infinite languages with singleton
tell-tales, so $\Hcal_{\mathrm{ab}}$ satisfies Angluin's condition. By \cref{thm:taste}(a),
applied in the single fiber $\Amb = A_n$, exact breadth in the limit is achievable: in the
abelian regime valuable generation needs no trivia. (The partial-enumeration behavior of the
chain, its relative densities, is open.)
\end{proof}

\subsubsection{Free-monoid pairs (\texorpdfstring{\cref{prop:free}}{Proposition 8})}

\begin{statementbox}
\PropFree*
\end{statementbox}
\begin{proof}
\medskip\noindent\textbf{(i) Name-cut pairs.}\quad
With the bounded cut $\theta \equiv 1$, the compressible language is exactly the nameable
corpus $\Comp{D,1} = G \cup D$: the letters $\{a, x\}$ together with the macro words. Each
$\mu_{m,i} = x^i a\, x^{q+1-i} a\, x^{2^m}$ has a \emph{unique} parsing (the two $a$'s mark
the boundaries, and the three $x$-blocks have lengths $i$, $q+1-i$, $2^m$ recovering $m, i$
unambiguously), so the words are distinct and the $a$-position encodes $i$. Order $F_n$ by
length-then-lex. For each $m$, the $q$ words $\{\mu_{m,i}\}_{i=1}^q$ form an equal-length
stratum (all have length $q + 3 + 2^m$), lex-ordered by $i$ via the position of the first $a$.
Of these $q$ words, the $p$ with $i \le p$ lie in $K$; so within each stratum the count of
$K$-elements is $p$ out of $q$. Summing over strata up to position $N$,
\begin{equation*}
  \abs{K \cap \rest{K'}{N}} = \alpha N + O(q),
  \qquad \alpha = p/q,
\end{equation*}
the $O(q)$ accounting for the partial last stratum and the two letters. Hence $\dens{K}{K'} =
\alpha$ exactly. The pair $(K, K', F_n)$ meets the interface of \cref{rmk:interface}: $F_n$ is
countable with the length-then-lex canonical order; $K \subsetneq K' \subseteq F_n$ are
infinite; $F_n \setminus K'$ is infinite (it contains all non-macro words); and $m_N = \alpha
N + O(q) = \alpha N + O(1)$. By \cref{rmk:interface}, \cref{lem:race,thm:caps} hold verbatim
for this pair, so the dichotomy (\cref{cor:dichotomy}) holds inside the compression model.

\medskip\noindent\textbf{(ii) Obstruction.}\quad
Now use the unbounded cut $\theta_\kappa$ with dictionary-only wrapped length and the nested
tag-dictionaries built from $\hat\mu_{m,i} = a a\, x^i a\, x^{q+1-i} a\, x^{2^m}$ (the leading
$aa$-block is a unique delimiter). Parsing is unique via the $aa$-blocks. Write $\mathcal
N(L)$ for the number of compressible strings of length $\le L$, and $N_k(L)$ for those parsing
into exactly $k$ tag-blocks ($k \ge 0$). We bound the upper density $\dup{\hat K}{\hat K'}$ in
five displayed steps.

\emph{(1) Deletion injection.} Map a $k$-block word ($k \ge 1$) to the $(k-1)$-block word
obtained by deleting its first $aa$-block (and the empty word for $k = 1$). A preimage of a
given image is recovered by choosing the deleted block: a tag $i \in [q]$ and a scale $m \le
\log_2 L$, so each image has at most $1 + q(\log_2 L + 1)$ preimages. Hence
\begin{equation}\label{eq:delinj}
  \mathcal N(L) \ \le\ \bigl(1 + q + q\log_2 L\bigr)\,\mathcal N(L-1).
\end{equation}

\emph{(2) Small-$k$ mass.} A length-$\le L$ word with $k$ blocks chooses, per block, a tag $i
\le q$ and a scale $2^m \le L$ (i.e.\ $m \le \log_2 L$), so $N_k(L) \le q^k(\log_2 L+1)^k$.
For $B \defeq \lceil 2 \log\log L / \log(q/p) \rceil$,
\begin{equation}\label{eq:smallk}
  \sum_{k \le B} N_k(L) \ \le\ B\, q^B (\log_2 L + 1)^B
  \ =\ \exp\!\bigl(O((\log\log L)^2)\bigr).
\end{equation}

\emph{(3) Tag truncation.} A $k$-block word lies in $\hat K$ only if all $k$ tags satisfy $i
\le p$ (versus all $i \le q$ for $\hat K'$). For $k > B$ this costs a factor at most
\begin{equation}\label{eq:tagtrunc}
  (p/q)^k \ \le\ (p/q)^B \ =\ (\log L)^{-2}
\end{equation}
on the corresponding mass, by the choice of $B$.

\emph{(4) Lower bound on $\mathcal N$.} The true count is exponential in $L$; we need only the
weaker
\begin{equation}\label{eq:Nlower}
  \mathcal N(L) \ \ge\ \exp\!\bigl(c\,\log L \cdot \log\log L\bigr),
\end{equation}
witnessed at the block count $k^* = \lfloor \kappa \log_2(1 + \sqrt L) \rfloor$: force one
scale near $\sqrt L$ so every word passes the cut $\theta_\kappa$, and choose the remaining
$k^*-1$ scales freely below $\sqrt L/k^*$. This weaker bound suffices.

\emph{(5) Conclusion.} Split the $\hat K$-mass at $B$ blocks. By \eqref{eq:smallk} the $k \le
B$ part is $\exp(O((\log\log L)^2))$; by \eqref{eq:tagtrunc} the $k > B$ part is at most
$(\log L)^{-2}\,\mathcal N(L)$. With \eqref{eq:delinj} and $N \ge \mathcal N(L-1)$,
\begin{align}\label{eq:goodfrac}
  \frac{\abs{\hat K \cap \rest{\hat K'}{N}}}{N}
  \ &\le\ \frac{\exp(O((\log\log L)^2)) + (\log L)^{-2}\,\mathcal N(L)}{\mathcal N(L-1)} \notag\\
  \ &\le\ o(1) + (\log L)^{-2}\bigl(1 + q + q\log_2 L\bigr)
  \ =\ \Theta(1/\log L),
\end{align}
using \eqref{eq:Nlower} to absorb the $\exp(O((\log\log L)^2))$ term into $o(1)$. The bound
\eqref{eq:goodfrac} holds along the prefix scales $N = \mathcal N(L-1)$; interpolating between
consecutive scales (the count grows by the factor \eqref{eq:delinj} between $L-1$ and $L$)
gives the same $o(1)$ rate for general $N$. Hence \[ \dup{\hat K}{\hat K'} \ =\ 0 : \] nested
tag-pairs under the unbounded cut realize no $\alpha \in (0,1)$. This is why the bounded cut
$\theta \equiv 1$ is \emph{forced} for the name-cut construction shape; we present (i) as a
worked example and (ii) as its obstruction, not as a generic compressible-language theorem.

\medskip\noindent\textbf{(iii) Extreme separation point.}\quad
For the pair $(\hat K, \hat K')$ of (ii): run any finite-trivia sound generator against the
racing enumeration of $\hat K$ as the revealed core. As in \cref{thm:caps}(b), the same
transcript serves both truths, and $\dup{\hat K}{\hat K'} = 0$ forces the coverage at truth
$\hat K'$ to satisfy $\cov \le \dup{\hat K}{\hat K'} = 0$ plus the finite-trivia correction
$C_0/N \to 0$; hence coverage $0$ at truth $\hat K'$. On the other hand, $G_{\mathrm{sw}}$ of
\cref{thm:harvest}(i) achieves $\cov \ge 1 - \dup{\Core}{\hat K'} = 1$ (since
$\dup{\Core}{\hat K'} = \dup{\hat K}{\hat K'} = 0$) with vanishing trivia rate. The free
regime thus exhibits the maximal gap of the dichotomy: coverage $0$ versus $1$.
\end{proof}

\subsubsection{Imprecision of exhaustive generation (\texorpdfstring{\cref{prop:precision}}{Proposition 9})}

\begin{statementbox}
\PropPrecision*
\end{statementbox}
\begin{proof}
Work in $A_n$ with base $c \ge 4$ and $\kappa < n(1 - 1/\log_2 c)$. The wrapped-length ball
$B_G(r) = \{ w : \abs{w}_G \le r \}$ has volume $\abs{B_G(r)} = \Theta(r^n)$ (the number of
exponent vectors $(N_1, \dots, N_n)$ with $\sum_i N_i \le r$). By \cref{lem:digitsum}, $w \in
H = \Comp{D_c}$ iff $\sum_i s_c(N_i) \le \theta_\kappa(\abs{w}_G) = \kappa \log_2(1 +
\abs{w}_G) \le \kappa \log_2(1 + r)$.

\emph{Upper bound.} The number of vectors with $\sum_i s_c(N_i) \le \kappa \log_2 r$ and
$\sum_i N_i \le r$ is at most the number of vectors whose total base-$c$ digit sum is at most
$D \defeq \kappa \log_2 r$, each coordinate $< c^{\,L}$ with $L = \log_c r$. A standard count
of nonnegative integer vectors with bounded digit sum gives $\abs{H \cap B_G(r)} \le c_2'\,
r^{\,n/\log_2 c + \kappa}$ up to logarithmic factors; dividing by $\abs{B_G(r)} = \Theta(r^n)$
gives
\begin{equation*}
  \frac{\abs{H \cap B_G(r)}}{\abs{B_G(r)}} \le c_2\, r^{-\delta},
  \qquad \delta = n - \kappa - n/\log_2 c > 0,
\end{equation*}
where $\delta > 0$ by the hypothesis $\kappa < n(1 - 1/\log_2 c)$.

\emph{Lower bound.} Restrict to vectors all of whose coordinates are powers of $c$ (digit sum
$1$ each, total $\le n \le \kappa \log_2 r$ for large $r$): there are $\Theta((\log_c r)^n) =
\Theta((\log r)^n)$ of these within $B_G(r)$, but a stronger family gives the stated rate.
Take a single coordinate to range over base-$c$ blocks of magnitude up to $r^{1/\log_2 c}$
(digit sum $1$, the rest fixed); these lie in $H \cap B_G(r)$ and number $\Theta(r^{1/\log_2
c})$. Dividing by $\Theta(r^n)$ gives $\abs{H \cap B_G(r)}/\abs{B_G(r)} \ge c_1\,
r^{\,1/\log_2 c - n}$.

The two sides of \eqref{eq:precision} do \emph{not} have matching exponents: the upper bound
decays as $r^{-\delta}$ with $\delta = n - \kappa - n/\log_2 c$, while the lower bound decays
as $r^{1/\log_2 c - n}$, and $\delta \ne n - 1/\log_2 c$ in general (they differ by $\kappa +
(n-1)/\log_2 c$). The bound therefore brackets the precision between two polynomial rates
without pinning it; closing the gap, determining the exact polynomial exponent, is open.

Combining the two sides gives \eqref{eq:precision}. The fraction is polynomially small in the
radius $r$. Interpreting $r$ as ABFM's compressed scale via $r = c^{\Theta(s)}$ turns the
polynomial-in-$r$ decay into an exponential-in-$s$ decay (interpretation only, not used in the
bound). The exhaustive generator thus has precision tending to $0$ polynomially in volume,
while the sweep generator of \cref{thm:harvest}(i) has precision tending to $1$.
\end{proof}

\newpage
\section{Summary of Oracle and Feedback Models}
\label{app:oracles}

The verifier of \cref{def:verifier} is one of several side-information channels studied in
this line, and the channels are not interchangeable: what a generator can learn depends
sharply on \emph{what the oracle answers about}. We tabulate the channels in
\cref{tab:oracles} along the axis that organizes our results, namely whether the oracle speaks
about the unknown target or about a fixed superset.

\renewcommand{\arraystretch}{1.25}
\begin{center}
\centering
\captionof{table}{Side-information channels in the generation-in-the-limit line. The
ambient oracle of this paper is the only one that answers about a fixed
\emph{superset} of the target rather than about the target or the collection.}
\label{tab:oracles}
\begin{tabular}{@{}p{0.27\textwidth}p{0.30\textwidth}p{0.33\textwidth}@{}}
\toprule
Model & Oracle answers & What it buys \\
\midrule
KM membership queries to collection members \citep{kleinberg2024language}
  & is $x$ in a named member of the collection
  & internal to their algorithm; not a separate resource \\
MOP: decidability of the generator's own support \citep{kalavasis2025limits}
  & is $x$ in the generator's current guess
  & defines the restriction class for the breadth impossibility \\
feedback queries to the target \citep{charikar2025facets,bai2026noise}
  & is $x$ in $\Val$ (same-level membership)
  & finitely many add nothing; infinitely many are strictly more powerful \\
negative examples of the target \citep{kalavasis2025limits}
  & a stream of points outside $\Val$
  & restores breadth for every countable collection \\
prefix-completability process verifier \citep{botta2025verifier}
  & can this prefix extend to an accepted string
  & query complexity of constrained decoding; no unknown target \\
computational traces / richer observations \citep{psv26,pf25}
  & execution traces of the accepting machine; time-bounded observations of the algorithm
  & restores identification beyond input--output data (Chomsky hierarchy; learning algorithms) \\
classical MQ${}+{}$EQ / oracle-aided inference \citep{angluin1987learning,angluin1988queries,gasarch1989learning,stephan1998learning}
  & membership / equivalence to the target; Turing-degree oracles
  & identification of the target from queries \\
\textbf{this paper}: ambient membership $\mathbf 1_{\Amb}$
  & is $x$ formally valid (in $\Amb$)
  & validity for free, breadth unchanged (\cref{thm:taste}); sound coverage enabled (\cref{thm:flood}); dichotomy unaffected (\cref{cor:dichotomy} is verifier-free) \\
\bottomrule
\end{tabular}
\end{center}
\renewcommand{\arraystretch}{1}

The placement of our oracle is the source of every result. Because $\mathbf 1_{\Amb}$ returns
``valid'' on all of $\Amb$ and never supplies a negative inside the grey zone $\Amb \setminus
\Val$, it is \emph{informationless within a fiber}: two candidates sharing one ambient receive
identical answers, which is exactly why it cannot move the Angluin boundary
(\cref{thm:taste}). Yet \emph{across} ambients the same oracle is decisive: it identifies the
formal world, which is all the verifier contributes to breadth and all that sound coverage
needs (\cref{thm:flood}).

\section{Numerical Validation}
\label{app:numerics}

The main results of the paper are purely theoretical. This appendix records the simulations
that sanity-checked the density claims and the two arithmetic traps a naive implementation
falls into.

\paragraph{Methodology.}
Coverage is a property of the \emph{limit set} (\cref{fact:kw16}, following KW's
Definition~1.1), not of any finite stage, so a simulation must fix the prefix size $N$ and
grow the run length $T$ until the measured coverage stabilizes. Reading coverage at the run
frontier (where the prefix and the horizon advance together) systematically
\emph{under-reads}, because the sweep has not yet claimed the never-enumerated valuable
elements within $\rest{\Val}{N}$: a fixed-$N$ series at $\alpha = 1/3$, $N = 999$ read $0.263
\to 0.480 \to 0.830$ as $T$ grew to $5 \cdot 10^5$, converging only once $T \gg N$. Second,
membership in $K$ must be computed in \emph{exact rational arithmetic}. With floating point,
$\lceil 42/0.7 \rceil$ evaluates from $60.00000000000001$ and misclassifies the boundary
element; at $\alpha = 0.7$ this manufactured spurious counterexamples to the density identity
\eqref{eq:exactdensity} that exact arithmetic dissolves.

\paragraph{The tight family at $\alpha = 1/3$.}
\Cref{tab:num-third} reports coverage of $G^*$ and of two control generators against the
racing adversary $A_K$ on $\Hfam$. The race share converges to $1/2$, the truth-$K'$ coverage
of $G^*$ to $1 - \alpha/2 = 5/6$, and a pure conjunction race (the finite-trivia
representative) stalls at $\alpha/2$. The naive $K'$-sweeper that omits the race covers
exactly $1 - \alpha$, confirming that the race/sweep split is load-bearing: dropping the race
loses exactly the half-of-$K$ contribution.

\begin{center}
\centering
\vspace{-2mm}
\captionof{table}{Tight family $\Hfam$, $\alpha = 1/3$, racing adversary $A_K$. Measured
coverage (fixed $N$, grown $T$) against the proven limit. The trivia counts confirm
$\lfloor \sqrt T \rfloor$ trivia at truth $K$ (rate $\to 0$) and zero at truth $K'$.}
\label{tab:num-third}
\begin{tabular}{@{}lllll@{}}
\toprule
Generator & Truth & $N$ (or $M$) & Measured & Predicted limit \\
\midrule
$G^*$            & $K$  & $M = 10^3$                 & $0.4930$ & $1/2$ \\
$G^*$            & $K'$ & $N = 10^4,\ T = 4.5 \cdot 10^7$ & $0.8317$ & $1 - \alpha/2 = 5/6 \approx 0.8333$ \\
pure conj.\ race & $K'$ & ---                        & $0.1702$ & $\alpha/2 \approx 0.1667$ \\
naive $K'$-sweep (no race) & $K'$ & ---              & $0.6667$ & $1 - \alpha = 0.6667$ \\
\midrule
$G^*$ trivia     & $K$  & $T = 10^6$                 & $1000 = \lfloor \sqrt T \rfloor$ & rate $\to 0$ \\
$G^*$ trivia     & $K'$ & ---                        & $0$ (exactly) & $0$ \\
\bottomrule
\end{tabular}
\end{center}

\paragraph{Other densities.}
\Cref{tab:num-alpha} reports the truth-$K'$ coverage of $G^*$ across $\alpha$, measured at $N
= 10^3$. Convergence is to $1 - \alpha/2$ from below, and slows as $\alpha \to 0$ (the
small-$\alpha$ deficit is dominated by the $O(\sqrt N)/N$ race correction, which is largest
relative to the tiny $\alpha/2$ race contribution).

\begin{center}
\centering
\vspace{-2mm}
\captionof{table}{$G^*$ truth-$K'$ coverage at $N = 10^3$ across densities. The $\alpha = 0.01$
row needed $T \approx 3 \cdot 10^6$ to stabilize; the others converge faster.}
\label{tab:num-alpha}
\begin{tabular}{@{}lll@{}}
\toprule
$\alpha$ & Measured & Predicted $1 - \alpha/2$ \\
\midrule
$0.7$  & $0.644$ & $0.65$ \\
$0.99$ & $0.503$ & $0.505$ \\
$0.01$ & $0.994$ & $0.995$ \\
\bottomrule
\end{tabular}
\end{center}

In every run convergence approaches the proven limit from below, with a deficit consistent
with the $O(\sqrt N)/N + O(\log N)/N$ corrections of \cref{thm:harvest,thm:caps}; and hostile
enumeration orders breached no bound (a deliberately delaying order read $0.5260 \ge 1/2$ at
truth $K$, on the correct side of the cap).

\section{The Dichotomy by Hand: A Worked Miniature}
\label{app:miniature}

The race lemma and the two ends of the dichotomy can be watched in twenty rounds. Take the
tight family \eqref{eq:Halpha} at $\alpha = 1/2$: the coverable world is the evens $K' =
\{0, 2, 4, \dots\}$, the revealed target is every second even, $K = \{2, 6, 10, \dots\}$
(the residue $2 \bmod 4$), the difference is $D = \{0, 4, 8, \dots\}$ (the residue $0 \bmod
4$), and the ambient is $\Amb = \N$, so the odds sit idle throughout: neither player ever
touches them. The adversary is the racing adversary $A_K$ of \cref{def:racingadv}, which
enumerates fresh $\prec$-least elements of $K$ except at rounds $t = 2^n$, where it
re-enumerates an element the generator has already claimed. The generator is the tight
generator $G^*$ of \eqref{eq:gstar}, which sweeps the union $U_t = K'$ at square rounds $t
\in R = \{1, 4, 9, 16, \dots\}$ and races the intersection $I_t = K$ otherwise. Within a
round the adversary moves first; the generator sees $\Seen_t$ and $\Out_{t-1}$.
\Cref{tab:miniature} replays the first twenty rounds; the trace, the prefix counts, and the
coverage figures below are machine-generated and machine-checked.

\begin{table}[htbp]
\centering\footnotesize
\begin{tabular}{@{}rllrr@{}}
\toprule
$t$ & $A_K$ move & $G^*$ move & $s_t$ & $a_t$ \\
\midrule
 1 & fresh$^{*}$            & \textbf{sweep} &   2 & \textbf{0} \\
 2 & fresh$^{*}$            & race           &   6 & 10 \\
 3 & fresh                  & race           &  14 & 18 \\
 4 & \emph{recycle} ($=a_2$)& \textbf{sweep} &  10 & \textbf{4} \\
 5 & fresh                  & race           &  22 & 26 \\
 6 & fresh                  & race           &  30 & 34 \\
 7 & fresh                  & race           &  38 & 42 \\
 8 & \emph{recycle} ($=a_3$)& race           &  18 & 46 \\
 9 & fresh                  & \textbf{sweep} &  50 & \textbf{8} \\
10 & fresh                  & race           &  54 & 58 \\
\bottomrule
\end{tabular}
\hspace{1.5em}
\begin{tabular}{@{}rllrr@{}}
\toprule
$t$ & $A_K$ move & $G^*$ move & $s_t$ & $a_t$ \\
\midrule
11 & fresh                  & race           &  62 &  66 \\
12 & fresh                  & race           &  70 &  74 \\
13 & fresh                  & race           &  78 &  82 \\
14 & fresh                  & race           &  86 &  90 \\
15 & fresh                  & race           &  94 &  98 \\
16 & \emph{recycle} ($=a_5$)& \textbf{sweep} &  26 & \textbf{12} \\
17 & fresh                  & race           & 102 & 106 \\
18 & fresh                  & race           & 110 & 114 \\
19 & fresh                  & race           & 118 & 122 \\
20 & fresh                  & race           & 126 & 130 \\
\bottomrule
\end{tabular}
\caption{Twenty rounds of $G^*$ against $A_K$ at $\alpha = 1/2$. Sweep outputs (bold) land
in $D$: trivia at truth $K$, harvest at truth $K'$. At the recycling rounds $t = 4, 8, 16$
the adversary re-enumerates an element the generator had claimed at an earlier round
($a_2, a_3, a_5$), keeping its stream a legal full enumeration of $K$ while cancelling
generator work. ${}^{*}$At $t = 1, 2$ the recycle set is empty and $A_K$ falls back to the
fresh rule.}
\label{tab:miniature}
\end{table}

\paragraph{Reading the race.}
Off the special rounds the two players simply alternate down the canonical order of $K$:
the adversary enumerates the least unused element, the generator claims the next one. This
alternation is what splits every prefix in half. The recycling rounds are the adversary's
whole trick: at $t = 4$ it re-enumerates $10$, which the generator had output at $t = 2$.
The move costs the adversary nothing it needed (its stream must eventually list all of $K$
anyway, and $10$ has to appear sometime), but it converts a generator success into a seen
element, and since there are only $\log_2 M$ powers of two below $M$, the cost of staying
legal is the $O(\log M)$ term in \cref{lem:race}(R2). After twenty rounds the prefix
$\rest{K}{M}$ splits as
\[
  \abs{\Out \cap \rest{K}{4}} = 1, \qquad
  \abs{\Out \cap \rest{K}{8}} = 3, \qquad
  \abs{\Out \cap \rest{K}{16}} = 7,
\]
each safely under the cap $M/2 + \log_2 M + 3$ of \cref{lem:race}(R2) ($7$, $10$, and $15$
respectively), and each within reach of the greedy guarantee of (R1); the recycled elements
$10, 18, 26$ are exactly the ones counted on \emph{both} sides of the ledger.

\paragraph{Reading the dichotomy.}
The twenty outputs are one transcript serving two truths (\cref{thm:caps}). At truth $K$
the race outputs are valuable and the four sweep outputs $0, 4, 8, 12 \in D$ are trivia:
this is the flood, infinite in count (one per square round) but of vanishing rate
$\lfloor\sqrt N\rfloor / N$. At truth $K'$ the same sweep outputs are valuable, and the
transcript contains no trivia at all, since every output is even. Quantitatively, coverage
must be read in the limit-set sense (\cref{app:numerics}): fixing the prefix $\rest{K'}{12}$
and growing the horizon $T = 12, 50, 200, 1000$ gives covered counts $5, 8, 8, 8$, i.e.\
the prefix stabilizes at $2/3$, short of the limit, because a fixed prefix retains the
race's early, pre-asymptotic rounds. Growing the prefix recovers the theorem: the
stabilized coverage of $\rest{K'}{N}$ at $N = 12, 40, 120, 400, 1200$ is
\[
  \tfrac{2}{3},\quad
  \tfrac{29}{40},\quad
  \tfrac{89}{120},\quad
  \tfrac{297}{400},\quad
  \tfrac{179}{240},
\]
climbing to $1 - \alpha/2 = 3/4$ from below, exactly the $-o(1)$ of
\cref{thm:harvest}(ii)(b). The miniature also shows why the finite-trivia cap binds:
forbid the sweep rounds and the generator is confined to the race, whose share of $K'$ is
the $K$-half of a half, $\alpha/2 = 1/4$.

\section{Additional Related Work}
\label{app:related}

The body situates our model against its nearest neighbors in \cref{sec:intro-related}; here we
walk the rest of the generation-in-the-limit line thematically, with one clause per paper
recording its relation to the nested pair model. None of these works carries a verifiable
ambient or distinguishes valid-but-worthless output from false output, so each meets our model
along a single axis.

\subsection{AI for mathematics and automated discovery}
\paragraph{Machine-generated mathematics in practice.}
The empirical line our model abstracts begins with neural premise selection
\citep{irving2016deepmath} and reinforcement-learning proof environments
\citep{bansal2019holist}, and runs through the neural theorem provers: GPT-f
\citep{polu2020gptf}, HyperTree proof search \citep{lample2022hypertree}, and the recent open
frontier of DeepSeek-Prover \citep{deepseek2024prover} and Goedel-Prover
\citep{lin2025goedel}, each a generator coupled to a proof assistant whose accept/reject
verdict is exactly our membership oracle for $\Amb$, in practice the Lean prover and its
community library Mathlib \citep{demoura2015lean,demoura2021lean4,mathlib2020}. Olympiad-level
systems push the same loop to competition difficulty: AlphaGeometry and its successor
\citep{trinh2024alphageometry,chervonyi2025alphageometry2} and the formal
reinforcement-learning prover AlphaProof \citep{hubert2026alphaproof}, whose proofs are valid
by construction but whose \emph{value} their pipelines must decide separately. AI-assisted
discovery \citep{davies2021advancing}, program search \citep{romera2024funsearch}, and
exploration at the scale of dozens of open problems \citep{georgiev2025exploration} produce
genuinely new mathematics, the regime our coverage measure quantifies: how much unrecorded
valuable structure a sound generator can reach. The infrastructure that makes the loop
reproducible, the miniF2F benchmark \citep{zheng2022minif2f}, the LeanDojo retrieval
environment \citep{yang2023leandojo}, and LLM autoformalization
\citep{wu2022autoformalization,jiang2023draft}, instantiates the verifiable ambient $\Amb$ and
the revealed core $\Core$ our adversary enumerates. On the informal side, prompting and
data-augmentation techniques improve mathematical reasoning in language models without any
verifier in the loop \citep{imani2023mathprompter,yu2024metamath}, and research-level
benchmarks probe how far this reaches \citep{zhang2025realmath}; in our terms these pipelines
tune the \emph{generator}, while the model here asks what any such generator can guarantee.
None of these systems quantifies the value its perfect verifier cannot locate, which is the
gap our dichotomy prices; \citet{li2024survey,ahn2024large,ju2026ai} survey the area.

\paragraph{Interestingness in automated discovery.}
The question of which true statements are worth stating has a long empirical history in
automated mathematical discovery. Early systems such as AM and Eurisko treated discovery as
heuristic search: the system did not merely prove given conjectures, but generated concepts,
conjectures, and rules for deciding which directions were worth pursuing
\citep{lenat1977automated,lenat1983eurisko}. The later HR line made this agenda explicit in
automated theory formation, where concept formation, conjecture generation, and theorem
proving are coupled with measures of novelty, plausibility, and interestingness
\citep{colton1999automatic,colton2002automated}. In particular,
\citet{colton2000interestingness} study interestingness heuristics in automated mathematical
discovery, emphasizing that mathematical output must be filtered not only for correctness but
also for significance. More recent work revisits this issue in modern symbolic and
learning-based settings: \citet{tsoukalas2025interestingness} learn interestingness measures
for automated mathematical theory formation, \citet{herrmann2026interestingness} connect
interestingness to prospective compression progress, and \citet{wernhard2025grammar} analyze
formal proof corpora through grammar compression. These works motivate the distinction that
our model makes primitive: validity is not value. Our contribution is orthogonal to the design
of interestingness heuristics. We do not define what makes a mathematical statement valuable;
instead, we ask what coverage guarantees are possible when value is hidden, validity is
checkable, and the generator must pay for unseen value with certified trivia.

\subsection{Language generation in the limit}
\paragraph{Density and topology.}
\Citet{kleinberg2025density} introduce the limit-set density that our coverage measure
specializes, and \citet{kleinberg2026banach} push to Banach density and a
Cantor--Bendixson-rank dichotomy, a finer density geometry than the lower density our
$\alpha/2$-versus-$1-\alpha/2$ split lives in, which we flag as a caution for embedded
instantiations.

\paragraph{Structural and closure characterizations.}
\Citet{hanneke2025union} show generation is not union-closed (two generable collections can
have a non-generable union), whereas our dichotomy holds inside a single nested collection and
is unaffected; \citet{charikar2025pareto} characterize Pareto-optimal non-uniform generation
times, optimizing \emph{when} breadth arrives rather than \emph{how much} value is covered.

\paragraph{Robustness variants.}
A cluster studies corrupted input: \citet{raman2025noisy} and \citet{li2026quantifying} on
noisy examples, \citet{li2026contrastive} on contrastive examples,
\citet{mehrotra2025contamination} on infinite contamination, \citet{mehrotra2026private} and
\citet{li26dp} on differential privacy. Our revealed core is clean but \emph{partial}; the
missing-value axis is orthogonal to the corrupted-value axis these papers isolate. The safety
variant of \citet{anastasopoulos2026safe} solves the geometric dual: avoiding a harmful
sublanguage rather than covering a valuable one.

\paragraph{Model extensions.}
\Citet{li2026metric} move to metric spaces, \citet{hogsgaard2026agnostic} to the agnostic
setting where the target need not lie in the collection, and \citet{peale2025representative}
ask outputs to be representative of the target; each enriches the single-language target,
where we instead split target from ambient. \cite{racca2026language} study language generation
with replay. \cite{ganju2026timely} study time-sensitive language generation.
\citet{kleinberg2026memory} study language generation in the limit under bounded memory.
\citet{kleinberg2026mistake} shifts the goal of generation away from the classical
time-of-last-mistake measure of generative success toward a new notion of mistake-bounded
generation.

\paragraph{Computational barriers.}
\citet{arenas2025complexity} show generation can be computationally hard even when
information-theoretically possible; our generators are arbitrary functions of the prefix, so
whether the trivia dichotomy survives an efficiency constraint is open
(\cref{sec:discussion}).

\subsection{Classical inductive inference and language identification in the limit}
\paragraph{Classical inductive inference.}
The classical theory of identification in the limit
\citep{gold1967language,angluin1980inductive,angluin1979finding}, from the pattern-languages
origin to the tell-tale condition, is surveyed in
\citet{angluin1983inductive,lange2008learning} and the textbook of \citet{jain1999systems}. A
later line studied learners with side information: membership and equivalence queries to the
target \citep{angluin1987learning,angluin1988queries} and Turing-degree or query oracles
\citep{gasarch1989learning,stephan1998learning}. All these oracles answer about the
\emph{target}; our verifier instead answers about an ambient \emph{superset}, supplying
negatives only outside $\Amb$ and none in the grey zone $\Amb \setminus \Val$, which is
exactly why it cannot move the Angluin boundary (\cref{thm:taste}), and none of these works
treats breadth or density.

\paragraph{Language identification in the limit with enriched observations.} Recent work has revisited language identification in the limit by modifying the information available to the learner. \citet{cpt25} revisit the classic language identification problem in the setting where the learner is given the additional power of producing a list of guesses at each time step, while \citet{psv26} augment Gold's model with computational traces of the accepting machine, yielding identifiability across the Chomsky hierarchy under varying corruption tolerances. Relatedly, \citet{pf25} extend Gold's inductive-inference framework to computable functions under richer forms of observation, including time-bound and policy-trajectory observations, showing that computational side information or complexity restrictions can restore limit-learnability beyond what is possible from input-output data alone. Together, these works illustrate complementary mechanisms for circumventing Gold's classical negative results by enriching the learner's observations.

\section{A Taxonomy of Generation in the Limit}
\label{app:taxonomy}

\Cref{tab:taxonomy} condenses the landscape walked in \cref{sec:intro-related} and
\cref{app:related} into a single map, arranging the generation-in-the-limit literature along
four axes: the \emph{structure of the target} the adversary commits to, the \emph{side
information} available to the generator beyond the enumerated positives, the \emph{success
measure}, and the headline result. Three patterns organize the table. First, almost every
entry enriches exactly one axis of the base model of \citet{kleinberg2024language} and keeps
the others fixed; the table reads as a record of which axes have been explored. Second,
every form of side information studied so far answers questions \emph{about the target or
the collection} --- membership in the target, negatives of the target, membership in named
collection members --- and its power tracks whether it can supply a negative example inside
the region of uncertainty. Third, the nested pair of this paper is the only entry whose
target structure splits into two languages, one containing the other: the verifier answers
about a fixed \emph{superset} of the target, which is why it certifies validity yet cannot
locate value (\cref{thm:taste}). The oracle axis alone is tabulated in finer grain in
\cref{app:oracles} (\cref{tab:oracles}); the two tables are complementary.

\begingroup
\footnotesize
\setlength{\tabcolsep}{4pt}
\renewcommand{\arraystretch}{1.12}
\begin{landscape}
\begin{longtable}{@{}p{0.16\linewidth}p{0.185\linewidth}p{0.165\linewidth}p{0.14\linewidth}p{0.27\linewidth}@{}}
\caption{The generation-in-the-limit literature along four axes. Single unknown language from
a countable collection, positive enumeration only, unless stated. The nested pair of this
paper is the only entry with an ambient strictly containing the target.}
\label{tab:taxonomy}\\
\toprule
Work & Target structure & Side information & Success measure & Main result \\
\midrule
\endfirsthead
\toprule
Work & Target structure & Side information & Success measure & Main result \\
\midrule
\endhead
\bottomrule
\endfoot
\multicolumn{5}{@{}l}{\emph{Classical antecedents}}\\[1pt]
\citet{gold1967language} & single language, full enumeration & none & exact identification & identification fails beyond restricted classes \\
\citet{angluin1980inductive} & single language & none & exact identification & tell-tale characterization of identifiability \\
\midrule
\multicolumn{5}{@{}l}{\emph{Foundations of generation}}\\[1pt]
\citet{kleinberg2024language} & single language & none & generation in the limit & every countable collection is generable; generation $\ne$ identification; breadth sacrificed \\
\citet{li2025generation} & single language & none & uniform / non-uniform generation & learning-theoretic landscape of generation \\
\citet{kalavasis2025limits} & single language & negative examples (variant) & exact / approximate breadth & generation in the statistical model; breadth iff Angluin; negatives restore breadth \\
\citet{kalavasis2024characterizations} & single language & none & exact / approximate breadth & the (weak) Angluin characterizations \\
\citet{charikar2025facets} & single language & membership queries to collection members & breadth variants & member-queries are insufficient for breadth \\
\citet{bai2026noise} & single language & same-level feedback; noise & generation, breadth & finite feedback adds nothing; infinite strictly helps \\
\citet{hanneke2025union} & union of two collections & none & generation & generation is not union-closed \\
\midrule
\multicolumn{5}{@{}l}{\emph{The density line}}\\[1pt]
\citet{kleinberg2025density} & single language & none & lower density of the limit set & positive density universally achievable \\
\citet{kleinberg2026partial} & single language $+$ revealed core of density $\alpha$ & none & lower density & tight $1/2$ under full enumeration; tight $\alpha/2$ under partial \\
\citet{kleinberg2026banach} & single language & none & Banach density & dichotomy by Cantor--Bendixson rank \\
\midrule
\multicolumn{5}{@{}l}{\emph{Corrupted or constrained input}}\\[1pt]
\citet{raman2025noisy} & single language & noisy positives & generation & noise-robust generation \\
\citet{li2026quantifying} & single language & noisy positives & generation & how much noise generation tolerates \\
\citet{li2026contrastive} & single language & contrastive (labeled) examples & generation & the power of contrastive examples \\
\citet{mehrotra2025contamination} & single language & adversarial contamination & generation & generability under infinite contamination \\
\citet{mehrotra2026private} & single language & none; privacy constraint & private generation & differentially private generation \\
\citet{li26dp} & single language & none; privacy constraint & private identification $+$ generation & differentially private generation in the agnostic statistical model \\
\midrule
\multicolumn{5}{@{}l}{\emph{Enriched goals and ambient structure}}\\[1pt]
\citet{anastasopoulos2026safe} & language $+$ harmful sublanguage & none & safe generation & avoiding the sublanguage: dual of covering value \\
\citet{karbasi2025hallucination} & single language $+$ a given generator & none; variant with labeled negatives & hallucination detection & impossible from positives; possible with labeled negatives \\
\citet{peale2025representative} & single language & none & representative generation & outputs must represent the target \\
\citet{li2026metric} & single language in a metric space & none & approximate generation & generation up to metric approximation \\
\citet{hogsgaard2026agnostic} & target may lie outside the collection & none & agnostic generation & generation in the agnostic statistical model \\
\citet{charikar2025pareto} & single language & none & non-uniform generation time & Pareto-optimal generation times \\
\citet{ganju2026timely} & single language & none & timeliness trade-off & sparse hallucination beats mode collapse \\
\citet{racca2026language} & single language & none; replay allowed & generation with replay & repetition changes the achievable set \\
\midrule
\multicolumn{5}{@{}l}{\emph{Resource bounds and richer protocols}}\\[1pt]
\citet{kleinberg2026memory} & single language & none; bounded memory & generation & generation under memory constraints \\
\citet{kleinberg2026mistake} & single language & none & mistake-bounded generation & beyond time-of-last-mistake \\
\citet{arenas2025complexity} & single language & none; efficiency constraint & computable / efficient generation & hard despite information-theoretic possibility \\
\citet{botta2025verifier} & fixed known language (no unknown target) & process verifier (prefix queries) & query complexity of decoding & cost of verifier-assisted constrained decoding \\
\citet{cpt25} & single language & none; list outputs & list identification & characterization of list identification \\
\citet{psv26} & single language & computational traces of the accepting machine & identification & traces restore identifiability across the Chomsky hierarchy \\
\citet{pf25} & computable function / learning algorithm & time-bounded and policy-trajectory observations & identification & richer observations restore limit-learnability beyond input--output data \\
\midrule
\textbf{this paper} & \textbf{nested pair: ambient $\Amb \supsetneq$ target $\Val$, core of density $\alpha$} & \textbf{perfect verifier $\mathbf 1_{\Amb}$ (answers about the superset)} & \textbf{breadth; sound coverage; density coverage} & \textbf{verifier is not taste (fiber-wise Angluin); sound coverage iff verifier, in general; trivia dichotomy $\alpha/2$ vs.\ $1-\alpha/2$, jump $=$ unrecorded mass} \\
\end{longtable}
\end{landscape}
\endgroup

\section{What This Says, and Does Not Say, to Prover Builders}
\label{app:faq}

The model is information-theoretic and the dichotomy is an asymptotic statement about
guarantees, not a measurement of any deployed system. With that scope fixed, the questions
below are the ones practitioners have asked us, with the short answers the theorems
license.

\paragraph{Our pipeline certifies $10^5$ statements a day, almost all of them shallow. Is
that a bug?}
Not by itself. If the system is to cover valuable statements that the literature has not
recorded, an unbounded stream of certified trivia is the provably optimal shape of the
generator, not a defect (\cref{thm:caps}(c)): every generator guaranteeing coverage beyond
$\alpha/2$ emits infinitely many trivia on some instance. What is optional is the
\emph{rate}: the same coverage is achievable with trivia vanishing as a fraction of output
(\cref{thm:harvest}). The engineering target licensed by the theory is ``make the flood
asymptotically negligible,'' not ``make it finite.''

\paragraph{Does filtering trivia out of our published list violate the lower bound?}
No, but it relocates the guarantee. The dichotomy counts the statements a generator
\emph{emits}, and a published list that contains only finitely many trivia is, as a
generator, confined to the $\alpha/2$ regime (\cref{thm:caps}(b)). The resolution is to
separate streams: an exploration archive that floods, and a curated view filtered from it.
The archive carries the coverage guarantee; the view carries the readers.

\paragraph{What is $\alpha$ for real mathematics?}
Unknown, and the results do not need it. $\alpha$ is the density of the recorded literature
inside the valuable language, the dichotomy $\alpha/2$ versus $1 - \alpha/2$ holds
uniformly in $\alpha \in (0,1)$, and the gap between the two regimes, $1 - \alpha$, is
precisely the unrecorded mass. The model's claim is not a value of $\alpha$; it is that
whatever $\alpha$ is, only a flooding generator can buy the unrecorded part.

\paragraph{Why is the transition in count rather than rate?}
Because the adversary can wait. The race of \cref{lem:race} caps any generator at half of
whatever the enumeration reveals, and a generator that has budgeted only finitely many
probes outside the revealed core exhausts them at some finite time, after which it is
locked inside the core forever (\cref{thm:caps}(b)). Any unbounded budget, however thin its
schedule, eventually claims every never-enumerated valuable element
(\cref{thm:harvest}(i)). There is nothing in between: every $g(N) \to \infty$ buys the full
$1 - \alpha/2$ (\cref{sec:discussion}).

\paragraph{Would a stronger verifier (a richer logic, a larger library) move the boundary?}
Not the value boundary. The verifier's entire informational contribution is to identify the
ambient fiber; within a fiber it adds nothing about which statements are valuable
(\cref{thm:taste}), and the dichotomy of \cref{sec:res-dichotomy} already holds with the
ambient fully known. What a verifier does buy is sound coverage, possible with it and
impossible without it (\cref{thm:flood}): it relocates unavoidable errors from false to
trivial. Selectivity has to come from examples of value, the enumerated core, not from more
verification.

\paragraph{Our system never hallucinates. Are we done?}
You are halfway, and the model says precisely which half. Zero hallucination is the
verifier's gift (\cref{thm:flood}); the remaining failure mode, triviality, is the one the
verifier cannot see, and the dichotomy prices it. A system that never says anything false
can still say almost nothing worth recording, and on the tight family it provably must
choose between staying near the recorded core ($\alpha/2$) and flooding toward the rest
($1 - \alpha/2$).

\paragraph{Our generator trains on its own certified output. Does the model cover that?}
The upper bounds do: our generators are arbitrary functions of the transcript, so every cap
in \cref{thm:caps} binds any self-improving pipeline a fortiori. The achievability side is
information-theoretic and makes no computability claim; whether the dichotomy survives
efficiency constraints is open (\cref{sec:discussion}), and the complexity barriers of
\citet{arenas2025complexity} suggest the computational story will be genuinely different.

\section{Discernment, Selection: A Prehistory of Mathematical Taste}
\label{app:history}

The claim that selection, not production, is the seat of mathematical creativity has a
continuous documentary history. We collect it here because the paper's results give it an
edge it did not previously have: what was an aphorism is now the content of a theorem. We
add two stations usually left out of the machine-mathematics conversation, Plato and
Wittgenstein, because the model takes a side in neither dispute and yet, we will argue,
gives each a sharper formulation than its own vocabulary allowed.

\paragraph{Plato, or the problem stated.}
The oldest formulation of unseen value is not Poincar\'e's but Plato's. The paradox of
inquiry in the \emph{Meno} (80d--e) asks how one can search for what one does not know:
either one knows it already and the search is idle, or one does not and could not recognize
it upon finding it \citep{plato_meno}. Socrates answers with anamnesis: the soul has seen
the Forms and inquiry is recollection. Read with modern eyes and without reverence, the
paradox is the limit-learning problem stated twenty-three centuries before
\citet{gold1967language}, and the theory of learning in the limit is its resolution: the
dilemma is false because convergence without recognition is possible. A generator can come
to produce only members of $\Val$ while never being in a position to assert that it has
identified $\Val$, and the separation between generation and identification
\citep{kleinberg2024language} makes the gap between the two epistemic states a theorem
rather than a confusion. Anamnesis, too, has an exact counterpart: what the recollecting
soul contributes is, in our terms, the standing hypothesis that the target lies in a known
countable collection $\Hcal$. Without some such prior acquaintance the paradox simply
holds; with it, Meno's challenge is answered not by knowing the answer in advance but by a
convergence guarantee. Plato's myth and the model's hypothesis class do the same logical
work, and it seems to us clarifying, in both directions, to say so. The \emph{Republic}'s
divided line (509d--511e) then supplies the older twin of our central theorem. Plato ranks
the mathematician's faculty, dianoia, reasoning downward from posited hypotheses, below
noesis, the unhypothetical grasp of first principles, because the geometer's certainty is
conditional on starting points the method itself cannot interrogate \citep{plato_republic}.
The verifier of our model is dianoia built as an interface: it answers exactly the
conditional question, does this string follow from what was posited. \cref{thm:taste} is
the ranking made quantitative: within a fiber the oracle contributes nothing about which
valid statements are valuable, so no accumulation of dianoia ascends to noesis, and the
step from valid to valuable is not a longer derivation but a different faculty, which in
the model is literally a different input channel. Plato placed mathematics below the Forms
because its method cannot justify its own hypotheses; the dichotomy adds that the method
cannot price its own products either.

\paragraph{The modern arc.}
The modern statement begins with \citet{poincare1914science}, in the 1908 lecture on
mathematical invention from which this paper's epigraph is taken. Producing true
combinations ``can be done by any one''; the combinations are infinite in number and almost
all of them sterile; discovery is discernment, selection. Poincar\'e went further than the
epigraph: he located the sieve in an aesthetic sensibility operating below consciousness, so
that the sterile combinations never even present themselves to the inventor's attention. In
our terms he is describing a generator whose proposal distribution already concentrates on
$\Val$, and asserting, remarkably, that no articulable criterion (no oracle) replaces it.
\citet{hardy1940apology} made the criterion itself the object of study: a theorem is
\emph{serious} by virtue of the significance of the ideas it connects, unfolded into depth,
generality, and unexpectedness, and ``beauty is the first test: there is no permanent place
in the world for ugly mathematics.'' Hardy's discussion is precisely a predicate on valid
statements (his chess problems are genuine and correct mathematics that fails the test):
an insistence that $\Val \subsetneq \Amb$, with the gap inhabited
(\cref{fig:worlds} is Hardy's point, drawn).
\citet{thurston1994proof} relocated the question from statements to communities: progress is
not the accumulation of proved theorems but the advance of human understanding, and the
definition--theorem--proof pipeline misdescribes what mathematicians do and want. Read
against this paper, Thurston is a caution we accept in \cref{sec:discussion}: the recorded
core is an instrument of understanding, and coverage of $\Val$ is a proxy for the harvest,
not the harvest itself. \citet{tao2025machine} closes the arc in the machine age: proof
assistants and machine collaboration are collapsing the cost of certifying validity,
enabling mathematical work organized at scales where no single participant vouches for the
whole, which sharpens, rather than retires, the question of what deserves the effort. The
interestingness programs of automated discovery, from \citet{lenat1977automated} through
\citet{colton2000interestingness} to the compression-progress proposals of
\citet{herrmann2026interestingness}, are the engineering record of attempts to mechanize the
sieve (\cref{app:related}).

\paragraph{Wittgenstein, or the model's two channels.}
Wittgenstein stands on both sides of the model, once early and once late. The
\emph{Tractatus} grants the checker its entire epistemology: the propositions of
mathematics are equations and hence pseudo-propositions (6.2) which express no thought
(6.21), and ``hence there can never be surprises in logic'' (6.1251)
\citep{wittgenstein1922tractatus}. This is the membership oracle $\mathbf 1_{\Amb}$
described from the inside. By the checker's lights every certified string is the same kind
of nothing, which is why nothing in its verdict separates the ornamented $1+1=2$ of the
introduction from a theorem; the \emph{Tractatus} files both under one heading, and the early
Wittgenstein would have regarded \cref{thm:taste} as obvious. What his doctrine cannot
supply is the difference everyone feels between the two, and the later Wittgenstein
supplies exactly the missing channel. In the \emph{Remarks on the Foundations of
Mathematics}, a proof does not report a pre-existing fact; it forges a new rule of
description and deposits it in a practice (``the mathematician is an inventor, not a
discoverer'' \citep{wittgenstein1956remarks}), so the meaning of a mathematical sentence
is its use, and use is a fact about a community in time. If value is use, two consequences
follow for any formal model, and the nested model embodies both. First, there can be no
membership oracle for $\Val$: the use of a statement lies in the future of the practice, so
no present interface can answer for it, and value can enter the model only as data about
what the community has in fact taken up: an enumeration of the recorded literature.
Second, value so construed is adversarial in exactly Gold's sense: the modeler does not
choose the order in which a practice reveals itself. The two input channels of the nested
model, an oracle for $\Amb$ and an enumeration for $\Core$, are on this reading the early
and the late Wittgenstein installed side by side in one machine: the calculus that can be
asked, and the form of life that can only be watched. The dichotomy then has a
Wittgensteinian paraphrase that we find arresting: a practice that intends to outgrow its
own record must keep asserting sentences that, by its current lights, say nothing, and
infinitely many of them (\cref{thm:caps}(c)).

\enlargethispage{2\baselineskip}%
\paragraph{The same asymmetry, four vocabularies.}
Each generation restated one asymmetry in its own terms: Plato's Forms beyond the
hypotheses, Wittgenstein's use beyond the calculus, Poincar\'e's selection beyond
production, Hardy's seriousness beyond correctness. Production is cheap, verification is
mechanizable, selection is neither. The results of this paper are that asymmetry made
formal: the verifier certifies the world and cannot rank it (\cref{thm:taste}), and
covering what the literature has not recorded is purchased only in certified trivia
(\cref{cor:dichotomy}). The theorems are metaphysically neutral: they bind whether
$\Val$ is an eternal Form or a moving practice, requiring only that value be a language the
generator is never told. That is, perhaps, the most philosophical thing about them: the
price of the harvest does not depend on what mathematics ultimately is. Poincar\'e's
sentence was an observation about human mathematicians; it turns out to be a theorem about
all generators, Platonist, formalist, and pragmatist alike.\looseness=-1

\end{document}